\documentclass[mnsc]{informs_nojournal}

\RequirePackage{tgtermes}
\RequirePackage{newtxtext}
\RequirePackage{newtxmath}
 
\RequirePackage{endnotes}

\DoubleSpacedXI

\usepackage{algorithm}

\usepackage{algorithmic}

\usepackage{tikz}
\usetikzlibrary{arrows.meta, positioning}

\usepackage{natbib}
 \bibpunct[, ]{(}{)}{,}{a}{}{,} 
 \def\bibfont{\small}

\EquationsNumberedThrough

\TheoremsNumberedThrough      
 
\ECRepeatTheorems

\usepackage{graphicx} 
\usepackage{multirow} 
\usepackage{amsfonts} 
 
\usepackage[title]{appendix} 
\usepackage{xcolor} 
\usepackage{textcomp} 
\usepackage{manyfoot} 
\usepackage{booktabs} 
 
\usepackage{listings} 
 
\usepackage{subcaption}

\usepackage{mathtools}

\usepackage{hyperref}        
\usepackage{url}

\usepackage[capitalise,noabbrev]{cleveref}

\usepackage{enumitem}
\newcounter{subeqn}[section] \renewcommand{\thesubeqn}{\theequation\alph{subeqn}} 
\newcommand{\subeqn}{ 
  \refstepcounter{subeqn} 
  \tag{\thesubeqn} 
}

\newcommand{\AppendixStretch}{ 
  \renewcommand{\baselinestretch}{0.9}\selectfont
}

\newcommand{\otilde}{\mathcal{\tilde{O}}}

\newcommand{\bbR}{\mathbb{R}}

\newcommand{\bbE}{\mathbb{E}}

\newcommand{\lspo}{\ell_{\textnormal{SPO}}}

\newcommand{\rspo}{R_{\textnormal{SPO}}}

\newcommand{\X}{\mathcal{X}}

\newtheorem{condition}{Condition} 
 
\newtheorem{fact}{Fact}

\RUNAUTHOR{Wan, Liu, Grigas, and Shen}

\RUNTITLE{Decision-Focused Sequential Experimental Design}

\TITLE{Decision-Focused Sequential Experimental Design: A Directional Uncertainty-Guided Approach}

\ARTICLEAUTHORS{ 
\AUTHOR{Beichen Wan}
\AFF{Department of Statistics and Operations Research, University of North Carolina, Chapel Hill, 27599,  \EMAIL{bcwan@ad.unc.edu}}
\AUTHOR{Mo Liu}
\AFF{Department of Statistics and Operations Research, University of North Carolina, Chapel Hill, 27599,  \EMAIL{mo\_liu@unc.edu}}  
\AUTHOR{Paul Grigas}
\AFF{Department of Industrial Engineering and Operations Research, University of California, Berkeley, Berkeley, CA, 94720,  \EMAIL{pgrigas@berkeley.edu}}
\AUTHOR{Zuo-Jun Max Shen}
\AFF{Department of Industrial Engineering and Operations Research, University of California, Berkeley, Berkeley, CA, 94720,  \EMAIL{maxshen@berkeley.edu}} 
}

\ABSTRACT{ 
We consider the sequential experimental design problem in the predict-then-optimize paradigm. In this paradigm, the outputs of the prediction model are used as coefficient vectors in a downstream linear optimization problem. Traditional sequential experimental design aims to control the input variables (features) so that the improvement in prediction accuracy from each experimental outcome (label) is maximized. However, in the predict-then-optimize setting, performance is ultimately evaluated based on the decision loss induced by the downstream optimization, rather than by prediction error. This mismatch between prediction accuracy and decision loss renders traditional decision-blind designs inefficient.
\textit{\textbf{Methodology/results:}} To address this issue, we propose a directional-based metric to quantify predictive uncertainty. This metric does not require solving an optimization oracle and is therefore computationally tractable. We show that the resulting sequential design criterion enjoys strong consistency and convergence guarantees.
\textit{\textbf{Managerial implications:}} Under a broad class of distributions, we demonstrate that our directional uncertainty–based design attains an earlier stopping time than decision-blind designs. This advantage is further supported by real-world experiments on an LLM job allocation problem.
}

\begin{document}

\maketitle

\section{Introduction}
Experimental design studies how to strategically collect data by choosing experiments to run (i.e., which features to collect labels for) in order to maximize the information gained from experimental outcomes. It has long served as a fundamental tool for data collection. By carefully designing informative experiments, one can collect labeled data that supports the construction of predictive models. Here, the labeled data refers to pairs of (group, response) or, equivalently, (feature, label) or (covariate, outcome). Under a limited labeling budget, the goal of experimental design is to identify data acquisition strategies that lead to an accurate and reliable prediction model.

This data collection problem becomes particularly important when labels are expensive or time-consuming to obtain, especially in sequential settings where experiments are conducted over time and each newly observed label informs which experiment to run next. For example, in AI training, human annotators may be required to evaluate prompts for large language models (LLMs) \citep{ouyang2022training}, whereas in clinical trials, recruiting participants and measuring outcomes involve ethical, temporal, and financial constraints \citep{anderer2022adaptive}. Classical experimental design provides a rich framework for controlling inputs (e.g., covariates or feature configurations) to maximize the information gained from corresponding outputs (e.g., treatment responses or outcomes).
Because the information gained from labeling depends on the labeled data already collected, sequential experimental design chooses the next experiment adaptively based on past observations to maximize expected information gain (or uncertainty reduction) about the target of interest  \citep{castro2005faster, sugiyama2009pool, cai2016batch, beygelzimer2009importance}.

In sequential experimental design, the ``value'' of a candidate label depends on how uncertainty and performance are measured. Much of the regression-oriented literature optimizes surrogate objectives tied to (squared) prediction error, implicitly assuming that uniformly reducing predictive uncertainty translates into improved task performance. However, in many scenarios, the output of a predictive model is used directly as input to a downstream optimization problem to determine a decision policy (e.g., setting prices or choosing order quantities). In this setting, often referred to as the predict-then-optimize and/or contextual optimization setting, the performance metric of primary interest is the \emph{decision loss} induced by decisions made using the model’s predictions, rather than a standalone measure of predictive error. Consequently, smaller predictive uncertainty does not always imply smaller decision loss, as illustrated in Section~\ref{sec:intro_example}. This observation motivates the study of sequential experimental design with a decision-focused approach.

In this paper, we focus on the predict-then-optimize setting where the downstream decision problem is a linear optimization problem. Designing optimal sequential policies in this setting is challenging because uncertainty propagates through multiple stages: predictive uncertainty perturbs the parameters of the downstream optimization problem, and these perturbations translate discontinuously into changes in decisions (e.g., order quantities in inventory management or allocation choices in transportation systems). These decision changes further propagate into changes in decision loss. Consequently, relying on simple prediction-uncertainty proxies is misleading and inefficient for evaluating uncertainty in final decisions. On the other hand, directly evaluating uncertainty in decision loss is often both computationally and statistically intractable.

To design a sequential data collection metric that is both aligned with decision loss and computationally tractable, we propose a \emph{directional uncertainty}-guided approach. Building on the theoretical observation that the scale of the prediction does not affect decision-making in linear optimization (Fact~\ref{lem:spo_scale_invariant}), our design method concentrates on the \emph{directional} information of predictions. The central idea is to normalize each predicted cost vector onto the unit ball before evaluating its associated uncertainty. This normalization emphasizes how prediction errors project onto decision-relevant directions that influence the resulting downstream actions. 
Through this lens, our approach yields data acquisition rules that prioritize designs that are informative for downstream decisions, while de-emphasizing uncertainties that are irrelevant to decision making.

Since this directional-guided design method does not require solving the downstream optimization problem, it is straightforward to implement and intuitively appealing. However, an important question remains:

\textit{\bfseries Can the simple normalization of predictions yield a data collection metric that is better aligned with decision making and provides stronger theoretical guarantees?}

In this paper, we answer this question by showing that the directional uncertainty–based design approach achieves a strictly earlier stopping time than decision-blind designs that rely solely on prediction uncertainty. Our key contributions are summarized as follows.

\begin{itemize}
\item For contextual stochastic linear optimization, we introduce a decision-focused sequential experimental design method that prioritizes labels according to \emph{directional uncertainty}. This design method is motivated by the fact that the downstream linear optimization is independent of the prediction scale.

\item  We establish non-asymptotic guarantees for our directional uncertainty-guided design method in terms of the excess decision-risk bounds. By comparing the stopping time, we show that in general, our directional uncertainty-based method is no worse than traditional prediction uncertainty-based design methods.

\item Under a mild \emph{suboptimality gap} assumption, which states that there exists a positive gap in expected decision loss between optimal prediction models and suboptimal ones (Assumption~\ref{assu:gap}), we prove that our method achieves an earlier stopping time than decision-blind benchmark algorithms. This result highlights the benefit of normalization in experimental design in the predict-then-optimize setting. We further provide natural examples of distributions that satisfy the suboptimality gap assumption in Section~\ref{sec:example}.

\item Through numerical experiments based on a dataset for cardiovascular disease diagnosis as well as a \emph{real-world} study on prompt engineering for LLMs, we show that our policy attains the same decision risk with substantially less data than supervised or decision-blind sequential alternatives. These results confirm the practical advantages of the directional uncertainty-based approach.
\end{itemize}

The organization of the main body is as follows. In Section~\ref{sec:setting}, we introduce the formulation and setting of sequential experimental design in the predict-then-optimize setting. To illustrate the practical importance of this problem, we present a motivating application in Section~~\ref{sec:intro_example}. Section~\ref{sec:literature} reviews related work on decision-focused learning and experimental design. Our algorithm and its theoretical guarantees are provided in Section~\ref{sec:alg}. To further highlight the advantage of the directional-uncertainty approach, Section~\ref{sec:refinedlow} considers a class of distributions under which the proposed method guarantees an earlier stopping time than the decision-blind benchmark. Justifications of the assumptions are given in Section~\ref{sec:theoretical_justification}. Section~\ref{sec:numerical} examines the numerical performance of our sequential design method in two contexts—cardiovascular diagnosis and LLM job assignment. Both studies demonstrate that the directional–uncertainty–based design can achieve an earlier stopping time than traditional benchmarks. Finally, the conclusion is provided in Section \ref{sec:conclude}.

\subsection{Sequential Design for Predict-then-Optimize}\label{sec:setting}

In this section, we introduce the sequential experimental design problem in the \textit{predict-then-optimize} paradigm, which is a popular class of methods for addressing contextual stochastic linear optimization. In this setting, a decision maker (also a data collector) uses the collected data to fit a predictive model that predicts the unknown parameters in the downstream decision-making problem. Specifically, the decision-making problem is a linear optimization problem with unknown coefficients in the objective function. (For example, a logistics planner needs to select shipping routes for future deliveries, but must first predict fuel costs or travel times based on observable contextual features such as weather or traffic conditions.) The sequential design focuses on the data collection stage, where, by actively designing experiments, the (noisy) outcome of the designed experiments can be observed. The ultimate goal is to use a small number of experiments to obtain a prediction model that minimizes the expected decision loss. 

Let us first specify some notation used throughout the paper. We use $ \Vert \cdot \Vert $ to denote the $\ell_2$ norm. Let the sample space ${\mathcal{X}} = \{\mathrm{X}_{(i)}\}_{i=1}^m$ be the set of possible designs (treatments), where each $\mathrm{X}_{(i)}\in \bbR^\kappa$ stands for a possible design vector. Let $\mathcal{C}\subseteq\bbR^d$ denote the space of possible outcomes (i.e., the label space). We use $\mathcal{D}$ to denote the unknown joint distribution of the design $x$ and the outcome $c$. Since we focus on predicting the relationship between the feature $x\in \mathcal{X}$ and the outcome $c \in \mathcal{C}$, for simplicity, we assume for any $x= \mathrm{X}_{(i)}$, the marginal density of $x$, $\mu(x)$ is known. We denote the radius of the label space by $\rho(\mathcal{C}) := \max_{c \in \mathcal{C}} \Vert c \Vert $. Given a finite hypothesis class $\mathcal{H}$, our goal is to identify a predictive model $h\in\mathcal{H}:\mathcal{X}\rightarrow\mathcal{C}$ that minimizes the expected decision loss. Throughout the paper, we assume that $\mathcal{H}$ is finite with cardinality at most $|\mathcal{H}|$, a standard assumption in experimental design; see, e.g., \cite{zhao2024experimental}. When $\mathcal{H}$ is infinite, our analysis extends by replacing $|\mathcal{H}|$ with a complexity measure such as a covering number; see, e.g., \cite{wainwright2019high}.

In contextual stochastic linear optimization, there is a  downstream decision-making problem whereby, given each feature $x$, we aim to find a decision $w$ that minimizes the conditional expected cost: 
\begin{align}\label{equ:downstream}
    \min_{w\in S}\bbE[c^\top w|x] = \min_{w\in S}\bbE[c^\top |x]w,
\end{align} where $S$ denotes the polyhedral feasible region. 
We assume that the feasible region $\mathcal{S}$ is bounded and define its diameter as $\Delta(\mathcal{S}) := \sup_{w_1, w_2 \in \mathcal{S}} \Vert w_1 - w_2 \Vert$. Let $w^*(\hat{c}) := \arg\min_{w\in S}\hat{c}^\top w $ denote an oracle for solving Problem \eqref{equ:downstream} under the predicted cost vector $\hat{c}$. For simplicity, when Problem \eqref{equ:downstream} has multiple optimal solutions, we use a consistent tie-breaking rule (e.g., ordering extreme points of the polyhedron), thus the oracle $w^*$ yields a unique solution for every input cost vector $c$.  To quantify the decision loss in the downstream linear optimization, we adopt the Smart Predict-then-Optimize (SPO) loss introduced in \cite{elmachtoub2022smart},  defined by $$\lspo(\hat{c},c) = c^\top w^*(\hat{c}) - c^\top w^*(c).$$
Intuitively, the SPO loss quantifies the difference between the cost induced by $w^*(\hat{c})$ and the optimal decision $w^*(c)$, given that the actual cost vector is $c$.
The SPO risk is defined as the expectation of the SPO loss, i.e., $\rspo(h) := \bbE_{(x,c)\sim\mathcal{D}} [\lspo(h(x),c)]$. Then, for a given hypothesis class $\mathcal{H}$, our goal is to identify a predictive model that minimizes the SPO risk:  $\min_{h \in \mathcal{H}} \rspo(h)$. 
Let $\mathcal{H}^*$ denote the set of predictors that achieve the minimum SPO risk, and let $h^* \in \mathcal{H}^*$ represent one such optimal predictor. 

For convenience, we further introduce the following notations. For the given hypothesis class $\mathcal{H}$, we use $\hat{\mathcal{C}}$ to denote the range of the predicted labels, namely, $\hat{\mathcal{C}}:= \{c \in \bbR^d: c = h(x), \text{ for some } x \in \mathcal{X} \text{ and } h \in \mathcal{H}\}$. We define the maximum SPO loss by $\omega_{\lspo}(\mathcal{\hat C}, \mathcal{C}):= \sup_{\hat c \in \hat C, c \in \mathcal{C}} \lspo(\hat c, c)$.  We further define $ \gamma := \sum_{i=1}^m \max_{h_1, h_2 \in \mathcal{H}} \left\Vert \frac{h_1(x_i)}{\Vert h_1(x_i) \Vert} - \frac{h_2(x_i)}{\Vert h_2(x_i) \Vert} \right\Vert$,  as the cumulative maximum directional difference across the hypothesis class for each design.

In this paper, the sequential design process is as follows.
At the beginning of iteration $t$, we select a candidate design $x_t \in \mathcal{X}$ and conduct an experiment to observe its outcome $c_t$. Based on this new information from $c_t$, we then shrink the hypothesis class $H_t$ to remove some suboptimal prediction models $h$ from $H_t$. At the end of iteration $t$, the best model from the remaining hypothesis class $H_t$ is selected as the output.

	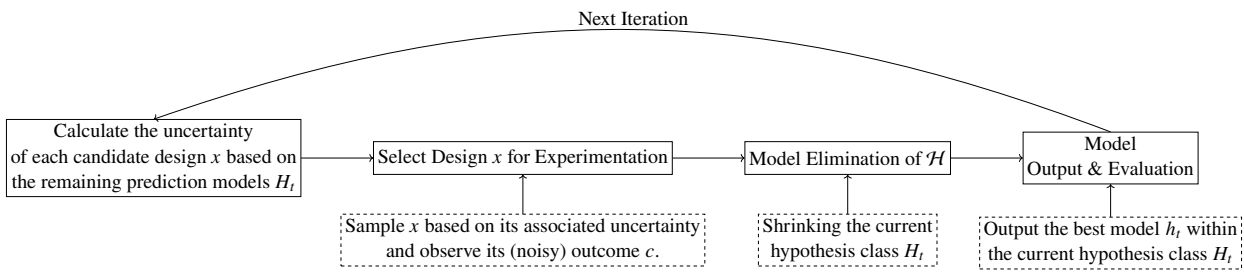
\begin{figure}[H]
		\centering
		\resizebox{\linewidth}{!}{ 
	       \begin{tikzpicture}[
                font=\Large,
				node distance=1.8cm and 2.2cm,
				box/.style={rectangle, draw, align=center, minimum width=3.5cm, minimum height=1cm},
				comment/.style={align=center, font=\Large	},
				methodbox/.style={rectangle, draw, dashed, align=center, font=\Large, minimum width=3.5cm},
				arrow/.style={->, thick}
				]

				\node[box] (uncertainty) {Calculate the uncertainty \\of each candidate design $x$ based on \\ the remaining prediction models $H_t$};
				\node[box,right=  of uncertainty] (expdesign) {Select Design $x$ for Experimentation};
				\node[box, right= of expdesign] (elim) {Model Elimination of $\mathcal{H}$};
 
				\node[box, right= of elim] (eval) {Model\\Output \& Evaluation};

				\draw[arrow] (uncertainty) -- (expdesign);
				\draw[arrow] (expdesign) -- (elim);
				\draw[arrow] (elim) -- (eval);

				\draw[arrow, bend right=20] (eval.north) to node[above] {Next Iteration} (uncertainty.north);

                \node[methodbox, below=1.2cm of expdesign] (exp_meth) {
					Sample $x$ based on its associated uncertainty \\ and observe its (noisy) outcome $c$.
				};
				
				\node[methodbox, below=1.2cm of elim] (elim_meth) {
					Shrinking the current \\ hypothesis class $H_t$ 
				};

				\node[methodbox, below=1cm of eval] (eval_comment) {Output the best model $h_t$ within\\ the current hypothesis class $H_t$};

				\draw[arrow] (exp_meth.north) -- (expdesign.south);
				\draw[arrow] (elim_meth.north) -- (elim.south);
				 
				\draw[arrow] (eval_comment.north) -- (eval.south);
				
			\end{tikzpicture} 
		}

			\caption{\normalsize Illustration of Sequential Experimental Design.}
			\label{fig:flowchart}
	\end{figure}
    
In Figure \ref{fig:flowchart}, at each iteration, when selecting design $x$, the high-level idea is to find a good metric to evaluate the uncertainty of each design $x$. Intuitively, among the remaining set of potential optimal prediction models $h\in H_t$, if the uncertainty (disagreement) of the prediction, $h(x)$, is small, then this design $x$ has less contribution to the model elimination, and thus is less likely to be selected.

An important question is which metric should be used to quantify uncertainty. For a given design $x$ and a class of candidate prediction models $H$, we compare the following three uncertainty metrics at $x$:
\begin{enumerate}[
    label=\textbf{Metric \arabic*.},
    leftmargin=3em,
    align=left
]
    \item $\max_{h_1,h_2\in H}\left\{\|h_1(x)-h_2(x)\|\right\}$ \hfill (decision-blind)
    \item $\max_{h_1,h_2\in H}\left|\lspo\!\big(h_1(x),h_2(x)\big)\right|$ \hfill (computationally intractable)
    \item $\max_{h_1,h_2\in H}\left\{
        \left\|\frac{h_1(x)}{\|h_1(x)\|}-\frac{h_2(x)}{\|h_2(x)\|}\right\|
    \right\}$ \hfill (directional)
\end{enumerate}

Although \textit{Metric 1} (based on the $\ell_2$ norm) is widely used in sequential data collection (e.g., \cite{sugiyama2009pool,beygelzimer2009importance,cai2016batch}), it may be poorly aligned with the SPO loss (see Example \ref{example:motivation}). In contrast, \textit{Metric 2} measures uncertainty directly in terms of the SPO loss, but it is typically computationally prohibitive: evaluating $\lspo$ requires solving the downstream optimization problem for each predicted cost vector, and quantifying uncertainty over a large set of candidate predictions can require a very large number of optimization-oracle calls. While prior work (e.g., \cite{liu2023active}) proposes margin-based surrogates to approximate such uncertainty via upper bounds, these approaches still rely on repeatedly solving the optimization problem and can remain computationally expensive. 

Motivated by these considerations, we propose \textit{Metric 3} and show that normalizing predictions yields an uncertainty measure that is both computationally efficient and more closely aligned with the SPO loss. Intuitively, this directional uncertainty discards the scale information in the predictions, which is appropriate because the SPO loss is scale-invariant (see Fact~\ref{lem:spo_scale_invariant}).

Before presenting our sequential data collection policy, we illustrate one application of the proposed sequential design in job assignment in Section \ref{sec:intro_example}.

\subsection{Application Example: Contextual Job Assignment}
\label{sec:intro_example}

We consider a job assignment problem in which $m$ jobs must be assigned to $n$ agents to minimize total cost. Let $w_{i,j}\in\{0,1\}$ indicate whether job $i$ is assigned to agent $j$, and let $c_{i,j}$ denote the random cost of assigning job $i$ to agent $j$, which depends on contextual information $x$. The problem can be formulated as
\begin{align}
    \min_{w \in \mathbb{R}^{m\times n}}& ~~ \bbE\!\left[\sum_{i=1}^m\sum_{j=1}^n c_{i,j} w_{i,j} \mid x\right] \label{obj:example} \\
    \text{s.t.}\quad
    & \sum_{i=1}^m w_{i,j} \le A_j, \quad \forall j=1,\dots,n, \label{con:1supply} \\
    & \sum_{j=1}^n w_{i,j} = B_i, \quad \forall i=1,\dots,m, \label{con:2demand} \\
    & w_{i,j} \in \{0,1\}, \quad \forall i=1,\dots,m,\  j=1,\dots,n. \label{con:3binary}
\end{align}
Constraint~\eqref{con:1supply} captures agent capacity limits, while \eqref{con:2demand} ensures that each job’s demand is satisfied.

A motivating specific problem is LLM task allocation. Different LLMs exhibit heterogeneous performance across tasks and incur different computational costs. When computational resources are limited, it is desirable to assign more complex tasks to high-capacity models and simpler tasks to lightweight models. In this setting, $m$ denotes the number of tasks, $n$ the number of available LLMs, $x$ the task or document type, and $c_{i,j}$ the (negative) performance score of model $j$ on task $i$. The assignment decision $w$ is obtained by solving Problem~\eqref{obj:example}, where $A_j$ represents the processing capacity of model $j$ and $B_i=1$ for all $i$.

In the above example, building the prediction model requires a training set with labeled samples, and obtaining this training set can be costly (e.g., human evaluation of the performance of LLMs on processed tasks). Hence, it is important to apply an efficient data collection algorithm.

\subsection{Related Work}\label{sec:literature}
To position our paper, we summarize some related work in decision-focused learning and sequential experimental design.

\textit{Decision-focused learning and contextual optimization}. Contextual optimization studies how to use contextual features to estimate uncertain parameters in optimization. In the predict-then-optimize regime, where point predictions of the uncertain parameters are used directly, higher predictive accuracy does not necessarily yield a smaller decision loss. Decision-focused learning (DFL), therefore, incorporates the downstream decision-making problem into the training of prediction models. Given the large and growing literature, we cite a representative subset of works most relevant to our linear optimization setting and refer readers to surveys such as \citet{mandi2024decision,sadana2025survey} for additional references and broader perspectives. 
    \cite{bertsimas2020predictive} uses local prediction models to estimate the downstream decision loss. \cite{zhu2022joint} considers the joint estimation and optimization from a robust perspective. \cite{er2025decision} further addresses the necessary and sufficient conditions for the existence of the decision-unbiased predictions for linear optimization.
    A stream of research by \cite{donti2017task} \cite{amos2017optnet} \cite{mandi2020smart} applied implicit differentiation to perform gradient updates on the prediction model parameters. Another stream of research is to use an analytical surrogate loss function that can be differentiated through. \citet{elmachtoub2022smart} and \citet{huang2024decision} propose different surrogate loss functions and establish their consistency under various conditions. Other related DFL methods include \cite{chung2022decision}, \cite{feng2022developing,feng2025contextual}, and \cite{cristian2025efficient}. In the context of data collection for predict-then-optimize, the sample complexity and generalization error bounds developed in \citet{ho2020risk,el2022generalization, hu2022fast} can be used to derive stopping-time guarantees for data collection. Although the motivation for DFL is clear, theoretically establishing its statistical benefit is nontrivial; see, for example, \cite{hu2022fast}. A stream of papers \cite{elmachtoub2023estimate,elmachtoub2025dissecting,lan2025bias} studies this benefit relative to the two-stage approach under model misspecification.

The above works on decision-focused learning either assume that the data are given a priori or that they are collected in an i.i.d. manner. In contrast, our paper is the first to study the SPO loss in a sequential experimental design setting with noisy observations, where the data collected are inherently non-i.i.d. Moreover, our paper is the first to demonstrate a statistical benefit of DFL over decision-blind design under specific structural assumptions on the noise distributions.

\textit{Sequential Experimental Design}. Experimental design has a long history in statistics and operations research. Classical work develops optimality criteria (e.g., A-, D-, and E-optimality) and studies structured designs such as Box--Behnken; see, e.g., \cite{kiefer1959optimum, box1960some}. These ideas have broad applications in physical and engineering experiments. More recently, sequential and adaptive procedures select the next experiment based on the data collected so far, resulting in stagewise designs that improve sample efficiency; see, for example, \cite{anderer2022adaptive,che2024optimization, zhao2025pigeonhole, li2025strong}. In regression settings, a closely related line of work—often discussed under the name \emph{pool-based selective sampling}—studies policies that sequentially select which inputs to label to improve a predictor efficiently \citep{castro2005faster,sugiyama2009pool,cai2016batch,beygelzimer2009importance,zhu2025power}. These methods can be viewed as sequential design rules in which the \emph{information value} of a candidate point is scored, and labels are acquired accordingly. Within this stream, \citet{beygelzimer2009importance} propose an importance-weighted algorithm (IWAL) that generalizes classification algorithms to regressions. IWAL measures, for a candidate feature vector, the maximum difference in loss over all possible labels and uses that quantity to set a sampling probability. While elegant, computing these scores becomes challenging when the loss departs from squared error, limiting direct applicability in decision-oriented contexts.

In the decision-focused setting, \cite{bennouna2025data} consider the offline data collection problem in a noiseless setting. To incorporate the SPO loss into the sequential data collection process, \citet{liu2023active} develop a margin-based sequential design using the notion of \emph{distance to degeneracy}. They show theoretical and empirical gains by utilizing the distances to the margins. However, their approach assumes access to margin information, which is computationally demanding when the feasible region is complex, and it treats near-degenerate points uniformly without distinguishing their relative information values.

In contrast, our sequential experimental design is computationally tractable by targeting \emph{directional uncertainty}. Conceptually, our work bridges classical sequential experimental design and decision-focused learning by tailoring information gathering to task structure.

\section{Algorithm and Performance Guarantees}\label{sec:alg}

In this section, we motivate the proposed uncertainty metric and present the corresponding algorithm. We then establish theoretical performance guarantees under mild assumptions.

\subsection{Algorithm Motivation}
\label{sec:alg_motivation}

In sequential experimental design problems, a common approach is to use prediction uncertainty to determine the next design sequentially. At each step, we aim to observe the outcome of an experiment whose predictive distribution has high uncertainty, rather than an experiment that the model has already learned well. A typical method for quantifying prediction uncertainty is to use the $\ell_2$ distance among the predictions produced by all hypotheses. However, this $\ell_2$ based criterion is decision-blind, and hence, when considering the SPO loss, it may lead to inefficiency in the predict-then-optimize setting.
To illustrate why a standard distance metric may fail to align with the SPO loss, let us consider the following example in Figure \ref{fig:predicted_cost_vector_with corresponding_optimizer}.

\begin{figure}[htbp]
  \centering
  \begin{minipage}[t]{0.40\textwidth}
    \centering
    \includegraphics[width=\linewidth]{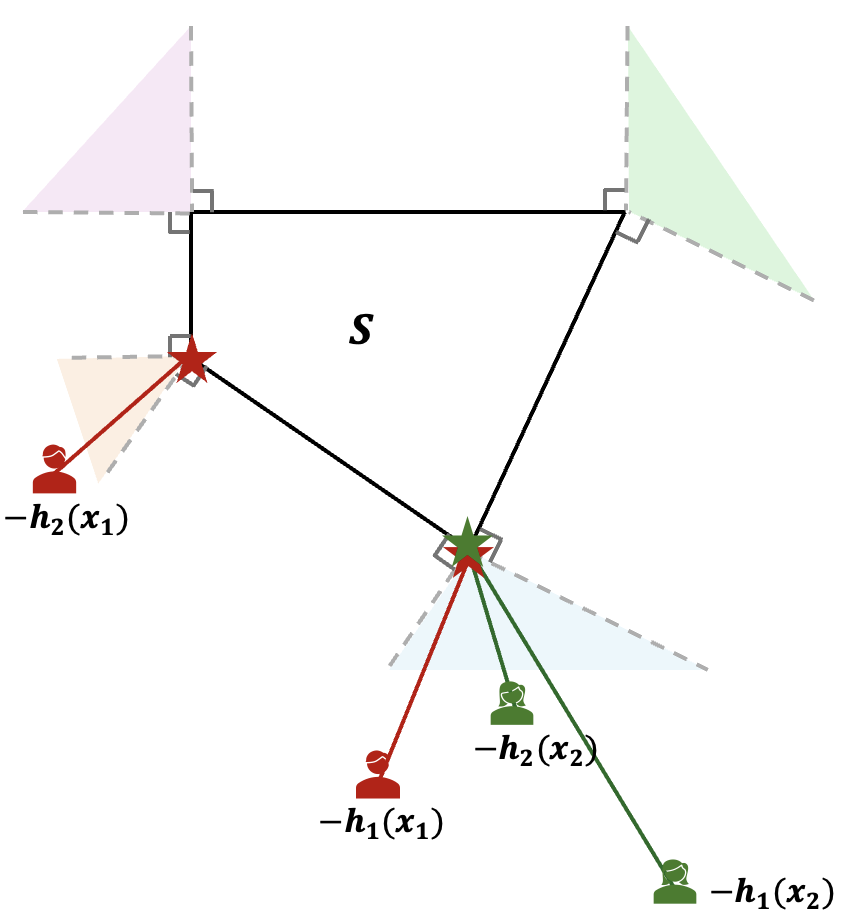}
    \caption{Cost Vector Prediction with Corresponding Optimizer}
    \label{fig:predicted_cost_vector_with corresponding_optimizer}
  \end{minipage}
  \hfill                 
  \begin{minipage}[t]{0.48\textwidth}
    \centering 
    \includegraphics[width=\linewidth]{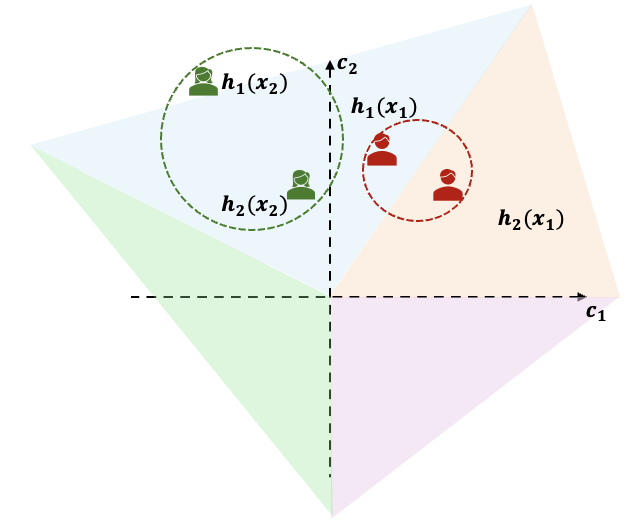}
    \caption{Cost Vector Prediction with $\ell_2$ Distance} 
    \label{fig:predicted_cost_vector_with_ell-2_distance}
  \end{minipage}
\end{figure}

\begin{example}\label{example:motivation}
Consider a two-dimensional setting where the current hypothesis class contains two predictors $\mathcal{H} = \{h_1,h_2\}$. Suppose there are two designs $\mathcal{X} = \{x_1,x_2\}$, denoted by red ($x_1$) and green ($x_2$) dots respectively in Figure \ref{fig:predicted_cost_vector_with corresponding_optimizer}.  Let $h_1(x_1),h_1(x_2),h_2(x_1),h_2(x_2)$ denote the corresponding predictions. Figure \ref{fig:predicted_cost_vector_with corresponding_optimizer} shows the feasible region $\mathcal{S}$, predicted cost vectors, and their corresponding optimizers. For ease of illustration, we represent each cost vector by its negative direction. Because of the linearity of the objective function, the optimal solution corresponds to the vertex of the cone containing the predicted cost vector. To evaluate the uncertainty of the two possible designs, Figure \ref{fig:predicted_cost_vector_with_ell-2_distance} shows the location of prediction values and the area of predictions that corresponds to the optimality of each vertex. In Figure \ref{fig:predicted_cost_vector_with_ell-2_distance}, although the size of the green circle is larger, indicating that $\mathcal{H}$ exhibits greater $\ell_2$ uncertainty on $x_2$ than on $x_1$, the predictions at $x_2$ yield the same optimizer. Thus, the uncertainty of SPO loss on $x_2$ is zero. On the other hand, the predictions of $h_1(x_1)$ result in a different optimizer from $h_2(x_1)$. It implies that the uncertainty of SPO loss on $x_1$ is nonzero. Hence, from the perspective of the SPO loss, $x_1$ has a larger uncertainty, even though $x_2$ exhibits a larger $\ell_2$ uncertainty. \hfill \Halmos
\end{example}

Example \ref{example:motivation} shows that there may exist a misalignment between decision-focused uncertainty and $\ell_2$ distance. This misalignment is due to a scale-invariant property of SPO loss, as pointed out by \citet{elmachtoub2022smart} and stated in the following fact.

\begin{fact}[Fact 1: The SPO loss is scale invariant.]
    \label{lem:spo_scale_invariant}
    For all $\hat{c}$ and $\alpha>0$, it holds that $w^*(\hat{c}) = w^*(\alpha\hat{c})$. Consequently, we have $\lspo(\alpha \hat{c},c) = \lspo(\hat{c},c)$.
\end{fact}

Fact \ref{lem:spo_scale_invariant} shows that for a fixed true cost vector $c$, $\lspo(\cdot,c)$ is scale invariant in the prediction $\hat{c}$.  
This motivates us to utilize directional information to quantify predictive uncertainty. Accordingly, we propose a directional uncertainty metric that better aligns with the SPO loss while remaining tractable; see Figure \ref{fig:angle_uncertainty} (vs. the standard $\ell_2$ metric in Figure \ref{fig:l2_uncertainty}).

\begin{figure}[htbp]
  \centering
   
  \begin{minipage}[t]{0.4\textwidth}
    \centering
    \includegraphics[width=\linewidth]{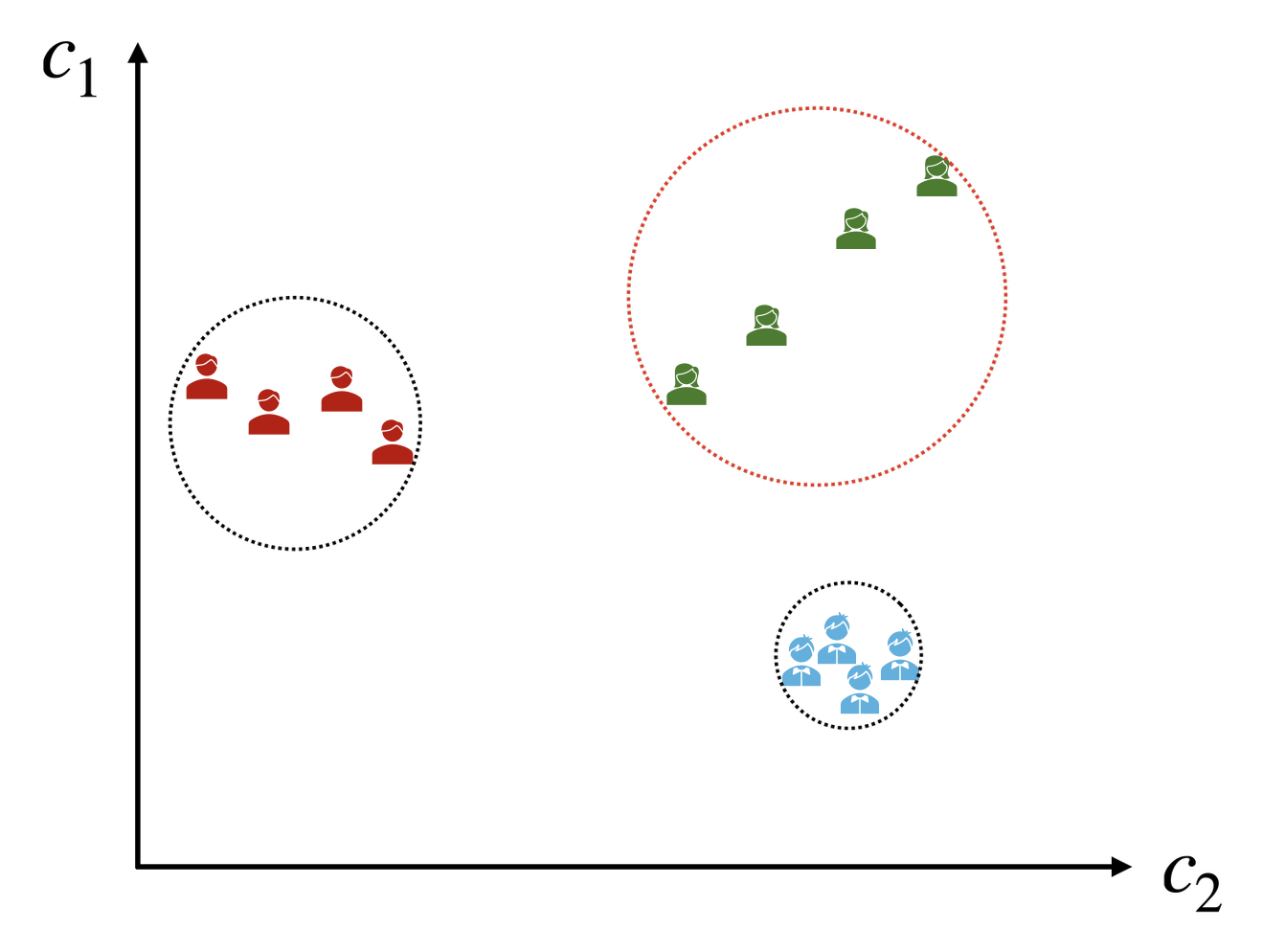}
    \caption{$\ell_2$ Based Uncertainty}
    \label{fig:l2_uncertainty}
  \end{minipage}
  \hfill                 
  \begin{minipage}[t]{0.4\textwidth}
    \centering 
    \includegraphics[width=\linewidth]{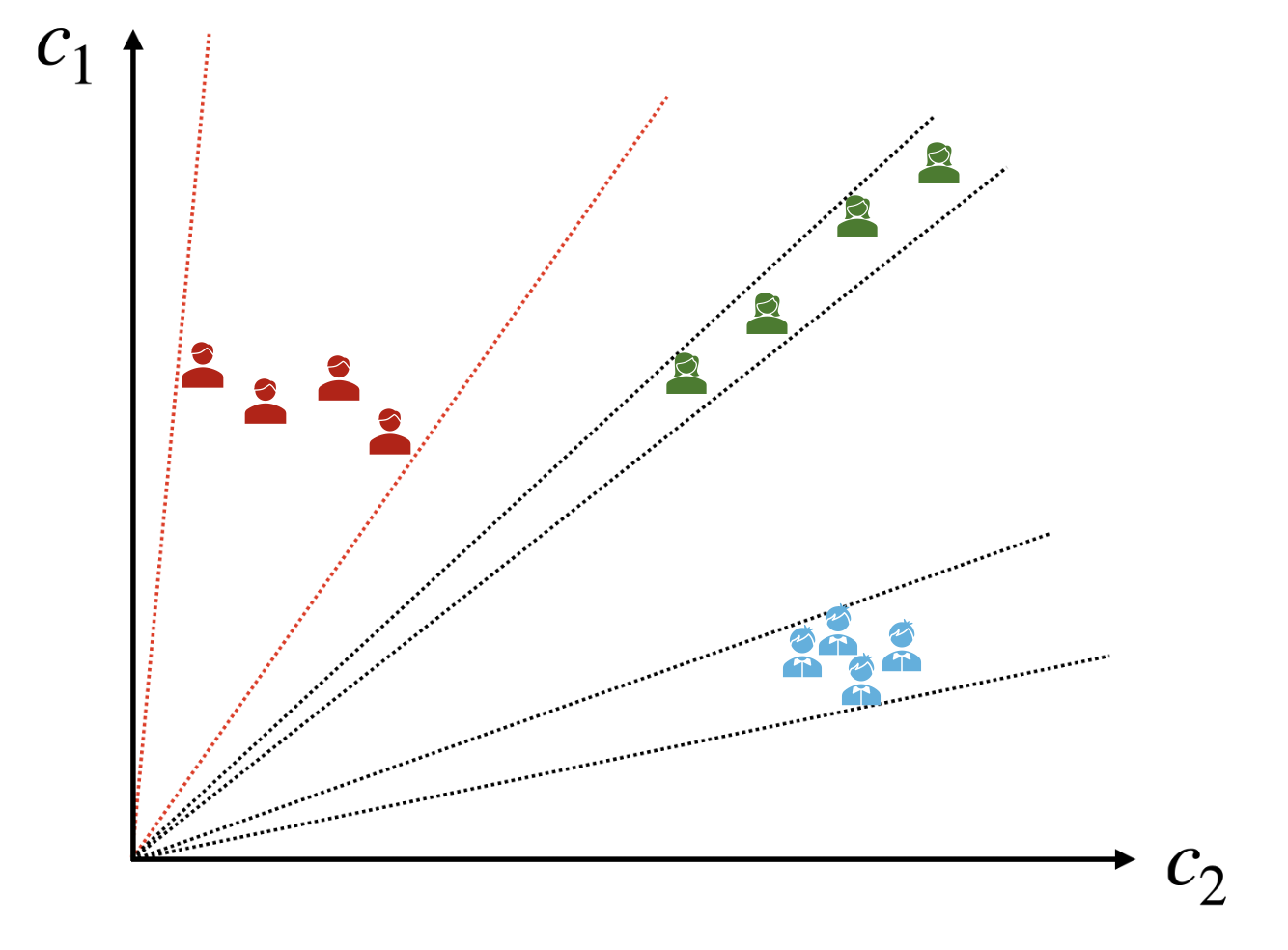}
    \caption{Direction-Based Uncertainty} 
    \label{fig:angle_uncertainty}
  \end{minipage}
\end{figure}

\vskip -1cm
In Figures \ref{fig:l2_uncertainty} and \ref{fig:angle_uncertainty}, there are three possible designs (represented by red, green, and blue colors of dots) and four candidate prediction models, $h_1,...,h_4 \in \mathcal{H}$. We assume the target cost vector $c \in \mathbb{R}^2$ has two dimensions, and each dot indicates a model's predicted cost vector for a given design. The circle stands for the maximum $\ell_2$ distance among the predictions for these three designs. Based on the sizes of the circles in Figure \ref{fig:l2_uncertainty}, the green design has the largest uncertainty and thus should be most likely to be selected by a standard approach. However, based on our new directional uncertainty in Figure \ref{fig:angle_uncertainty}, the red design has the largest angle difference, and thus is most likely to be selected, correctly reflecting the decision-focused nature of the problem.

\subsection{Importance Weighted Sequential Experimental Design}\label{sec:alg_weighted}

\begin{algorithm}[tp]
	\caption{Importance-Weighted Sequential Design Based on Directional Uncertainty (IWSD-DU)}\label{alg:general-iwal}
	\begin{algorithmic}[1]
		\STATE \textbf{Input:}   
		Initial slackness $r_0$

		\STATE Set $W_0 \gets \emptyset$, $n_0 \gets 0$, $H_0 \gets \mathcal{H}$

		\FOR{$t =$  1,2,$\cdots$, $T$}

		\STATE For each potential design $\mathrm{X}_{(i)} \in \mathcal{X}$, calculate $p_{t,i} =  \max_{h_1, h_2 \in H_t}\left\{\left\|\frac{h_1(\mathrm{X}_{(i)})}{\|h_1(\mathrm{X}_{(i)})\|} - \frac{h_2(\mathrm{X}_{(i)})}{\|h_2(\mathrm{X}_{(i)})\|}\right\|\right\}$, $\pi_{t,i} \gets \frac{p_{t,i}}{\sum_{j=1}^m p_{t,j}}$ 
        
        \IF {$\sum_{j=1}^m p_{t,j}>0$}
        \STATE Sample $x_t$ according to the probability $\pi_t = (\pi_{t,1},\dots,\pi_{t,m})$. Let $Q_{t,i}\in\{0,1\}$ denotes the indicator for the realization of $x_t$
 
        \ELSE 
        \STATE stop the algorithm and return $h_T$
        \ENDIF

		\STATE Conduct experiment and obtain a label $c_t$ associated with $x_t$ 
		\STATE Update $W_t \gets W_{t-1} \cup \{(x_t, c_t, \pi_t,Q_t)\}$, $n_t \gets n_{t-1} + 1$
		
		\STATE Let empirical reweighted loss $\hat{\ell}_{}^t(h) \gets \frac{1}{t} \sum_{(x_t, c_t, \pi_t,Q_t) \in W_t} \sum_{i=1}^m \frac{\mu(x_i)}{\pi_{t,i}}Q_{t,i} \lspo(h(x_i), c_i)$
		
		\STATE Update $h_t \gets \arg\min_{h \in H_{t-1}} \hat{\ell}_{}^t (h)$ and $\hat{\ell}_{}^{t, *} \gets \min_{h \in H_{t-1}} \hat{\ell}_{}^t (h)$.

		\STATE Update the confidence set of the predictor $H_t$ by $H_t \gets \{h\in H_{t-1}: \hat{\ell}_{}^{t}(h) \le \hat{\ell}_{}^{t,*} + r_{t} \}$

        \STATE Update $r_{t+1} \gets r_t \sqrt{\frac{t}{t+1} \frac{\log(2(t+1))}{\log (2t)}}$

		\ENDFOR
		 
		\STATE \textbf{Return} $h_T$
	\end{algorithmic}
\end{algorithm}

Our importance-weighted sequential design algorithm based on directional uncertainty (IWSD-DU) is stated in Algorithm \ref{alg:general-iwal}. The algorithm operates in a sequential design environment where, at the beginning of each iteration $t$, we can actively select a sample $x_t$ from the pool ${\mathcal{X}}$.
The algorithm maintains a confidence set of predictors $H_t \subseteq \mathcal{H}$. 
Initially, the confidence set of the predictor is set as the entire hypothesis space $\mathcal{H}$. To obtain the sample weights in each period, we calculate the maximum $\ell_2$ norm of prediction difference for the predictors within the current confidence set $\mathcal{H}_t$. In this step, calculating the labeling probability is computationally easy, especially when the confidence set $\mathcal{H}_t$ contains a finite number of predictors.  
To determine which design point to query, we sample $x_t = \mathrm{X}_{(i)}$ with probability $\pi_{t,i}$, where $\pi_{t,i} = \frac{p_{t,i}}{\sum_{j=1}^m p_{t,j}}$ and 
$p_{t,i} = \max_{h_1, h_2 \in H_t} \left\{ \left\| \frac{h_1(x_i)}{\|h_1(x_i)\|} - \frac{h_2(x_i)}{\|h_2(x_i)\|} \right\| \right\}$. 
Intuitively, $p_{t,i}$ measures the maximum directional disagreement in the predicted cost vector at $x_i$ among all hypotheses in the current hypothesis class $H_t$. A larger value of $p_{t,i}$ corresponds to greater uncertainty at $x_i$, resulting in a higher sampling probability~$\pi_{t,i}$.

For ease of notation, let $Q_{t,i}\in\{0,1\}$ denote the indicator random variable that equals one if the outcome of $x_t$ corresponds to feature $i$, for all $i\in[m]$, that is,  $Q_{t,i}\gets\mathbb{I}\{x_t = \mathrm{X}_{(i)}\}$. After designing an experiment $x_t$, we observe an outcome $c_t$ drawn from the conditional distribution of $c$ given $x_t$ and add the sample $(x_t, c_t)$ and its corresponding weight $\pi_{t,i}$ into the existing training set $W_{t-1}$.
After each iteration, we update the prediction model by minimizing the empirical re-weighted loss, which is defined as

\begin{align*}
	\ell^{\mathtt{rew}}(h; (x_t, c_t, \pi_t)) :=  \frac{\mu(x_i)}{\pi_{t,i}}\lspo(h(x_i), c_i), ~&\text{if design $x_i$ is chosen.}
\end{align*}

In Algorithm \ref{alg:general-iwal}, the random variables at iteration $t$ are $(x_t,c_t,\pi_t,Q_t)$.
  
With slight abuse of notation, in Algorithm \ref{alg:general-iwal}, the re-weighted loss function at iteration $t$ can be rewritten as

\begin{align*}
	\ell^{\mathtt{rew}}(h; (x_t, c_t, \pi_t)) := \sum_{i=1}^{m} \frac{\mu(x_i)}{\pi_{t,i}} Q_{t,i}\lspo(h(x_i), c_i).
\end{align*}

In Algorithm \ref{alg:general-iwal}, the confidence set for the optimal prediction models is constructed by allowing a slackness $r_t$ of the empirical reweighted loss, compared to the current minimum empirical loss. Specifically, $H_t \gets \{h\in H_{t-1}: \hat{\ell}_{}^{t}(h) \le \hat{\ell}_{}^{t,*} + r_{t} \}$. Because this confidence set is nested, the directional uncertainty $p_{t,i}$ associated with each design $x_i$ is nonincreasing in $t$. As the number of observations $t$ grows sufficiently large, the directional uncertainties of all designs approach zero, which induces a stopping time for the sequential experimental design.

\subsection{Preliminaries for Theoretical Guarantees}\label{sec:Lipsschitz}

In this subsection, we discuss the preliminaries for our algorithm's theoretical guarantees.  For simplicity, we use random variable $z_t\in \mathcal{Z}:= \mathcal{X}\times \mathcal{C}\times (0,1]^{m}\times \{0,1\}^{m}$ to denote the tuple of random variables $z_t := (x_t,c_t,\pi_t, Q_t)$. Thus, $z_t$ depends on $z_1, ..., z_{t-1}$ and thus is non i.i.d.. We define $\mathcal{F}_{t-1}$ as the $\sigma$-field of all random variables until the end of iteration $t-1$, i.e., $z_1, ..., z_{t-1}$. 
Given the sampling probability $\pi_t = (\pi_{t,1},\dots,\pi_{t,m})$, $Q_t=(Q_{t,1},\dots, Q_{t,m})$ is a  random vector that is independent of other randomness and $\bbE[Q_{t,i}] = \pi_{t,i}$, the expectation of the re-weighted loss is 
\begin{align*}
	\bbE[\ell^{\mathtt{rew}}(h;z_t)] &= \sum_{i=1}^{m} \bbE[ \frac{\mu(x_i)}{\pi_{t,i}} Q_{t,i} \lspo(h(x_i), c_i)] 
    \\ &= \sum_{i=1}^{m} \frac{\mu(x_i)}{\pi_{t,i}} \bbE[  Q_{t,i} \lspo(h(x_i), c_i)]    \\ &= \sum_{i=1}^{m} \frac{\mu(x_i)}{\pi_{t,i}} \bbE[  Q_{t,i}]\bbE[\lspo(h(x_i), c_i)] \\
    &=  \bbE[\lspo(h(x), c)] = \rspo(h).
\end{align*}

This implies that the expectation of the re-weighted loss is an unbiased estimator for the risk of $h$. Thus, under certain assumptions (Assumption \ref{assumption:tau_lower_bound}), the empirical re-weighted loss is expected to converge to the risk, and its minimizer $h_t$ over the confidence set $\mathcal{H}_{t-1}$ is expected to converge to the true SPO risk minimizer $h^*$.

To establish convergence, as shown in the proof of Theorem~\ref{thm:prediction}, the reweighted loss 
$\ell^{\mathtt{rew}}(h; (x_t, c_t, \pi_t))$ must admit a uniform upper bound at each iteration. Intuitively, this requirement ensures that when the sampling probability $\pi_t$ is small, the resulting excess SPO loss is proportionally small, which induces a ``Lipschitz-like'' behavior. Since the SPO loss is not Lipschitz continuous in general, we impose an additional assumption to guarantee this property.
A sufficient condition for this ``Lipschitz-like'' property is a lower-bound condition on the angle between the predicted cost vector and the degenerate cost vectors. The formal statement of this assumption is provided below, together with the definition of the distance to degeneracy.

\begin{definition}[Distance to Degeneracy, \cite{el2022generalization}]
    \label{def:distance_to_deg}
    For a polyhedron $\mathcal{S}$, the set of degenerate cost vectors is defined as $\mathcal{C}^o:=\{\hat{c}\in \mathbb{R}^d: \min_{w\in \mathcal{S}}\hat{c}^\top w \text{ has multiple solutions} \}$. The distance to degeneracy of a prediction $\hat{c}$ is defined as $\nu_\mathcal{S}(\hat{c}):=\inf_{c\in \mathcal{C}^o}\{\|c-\hat{c}\|\}$. 
\end{definition}

The distance to degeneracy quantifies the largest perturbation of the prediction $\hat c$ that still preserves the optimality of the current decision $w^*(\hat c)$. If the distance to degeneracy $\nu_\mathcal{S}(c)$ for a cost vector $c$ is large, then adding a small noise or slightly perturbing the prediction will not change the optimizer of the downstream problem. On the other hand,  if $\nu_\mathcal{S}(c)$ is relatively small, then perturbing the vector might lead to different solutions in the downstream problem. Hence, for a prediction $\hat{c}$, if $\nu_\mathcal{S}(\hat{c})$ is large, then we are more confident about the decision loss incurred by $\hat{c}$.
This distance to degeneracy serves as a key concept in showing the ``Lipschitz'' property of the SPO loss, as shown in Assumption \ref{assumption:tau_lower_bound}.

\begin{assumption}[Directional Margin Condition]
    \label{assumption:tau_lower_bound}
    There exists a constant $\eta >0$ such that for all $h \in\mathcal{H}, x\in \mathcal{X}$, it holds that $ \frac{\nu_\mathcal{S}(h(x))}{\|h(x)\|} \geq \eta$.
\end{assumption}

Assumption \ref{assumption:tau_lower_bound} is a very mild assumption. Actually, given the set of possible designs $\mathcal{X}$, all finite hypothesis classes $\mathcal{H}$ can satisfy Assumption \ref{assumption:tau_lower_bound} easily with a parameter $\eta>0$ by a simple equivalent transformation, as detailed in Section \ref{sec:justification_angle} later.

\begin{remark}[Geometric Interpretation of Directional Margin Condition.]
    Intuitively, Assumption \ref{assumption:tau_lower_bound} ensures a lower bound for $\frac{\nu_\mathcal{S}(h(x))}{\|h(x)\|}$, which is equivalent to an angular condition, as shown in Figure \ref{fig:angle_margin}. To illustrate this angle condition, we first observe that the label space $\bbR^d$ can be partitioned into multiple normal cones. Specifically, for a bounded polyhedron $\mathcal{S}$, there exists a convex hull representation $\mathcal{S} = \text{conv}\{v_1,\dots,v_k\}$, where $v_1,\dots,v_k \in \mathbb{R}^d$ denote the vertices of $\mathcal{S}$. For each vertex $v_j$, define a cone $\mathcal{K}_j = \{\hat{c}\in \mathbb{R}^d: v_j = \arg \min_{w\in \mathcal{S}} \hat{c}^\top w\}$ to be the corresponding normal cone of the cost vector $\hat c$ whose optimizer is $v_j$. (We can assume for each $c$ there exists only one negative normal cone that contains it by applying the tie-breaking rule we previously assumed). Using this representation, we have $\bbR^d = \cup_{j=1}^k \mathcal{K}_j  $. 
  
Consequently, for all  $h \in\mathcal{H}, x\in \mathcal{X}$, each prediction $h(x)$ lies in a corresponding negative normal cone $\mathcal{K}_j$. The boundary of the cone is the degenerate cost vector defined in Definition \ref{def:distance_to_deg}. As shown in Figure \ref{fig:angle_margin}, by some simple trigonometric functions, $ \frac{\nu_\mathcal{S}(h(x))}{\|h(x)\|} \geq \eta$ is equivalent to 
$$\left\|\frac{h(x)}{\|h(x)\|}-\frac{c^o}{\|c^o\|}\right\|\ge 2 \sin\left(\frac{\theta}{2}\right) =  2 \sin\left(\frac{\arcsin(\eta)}{2}\right), \quad \text
{ for all degenerate cost vector } c^o \in \mathcal{K}_j.$$ 

\begin{figure}[ht]
    \centering
\includegraphics[width=0.6\textwidth]{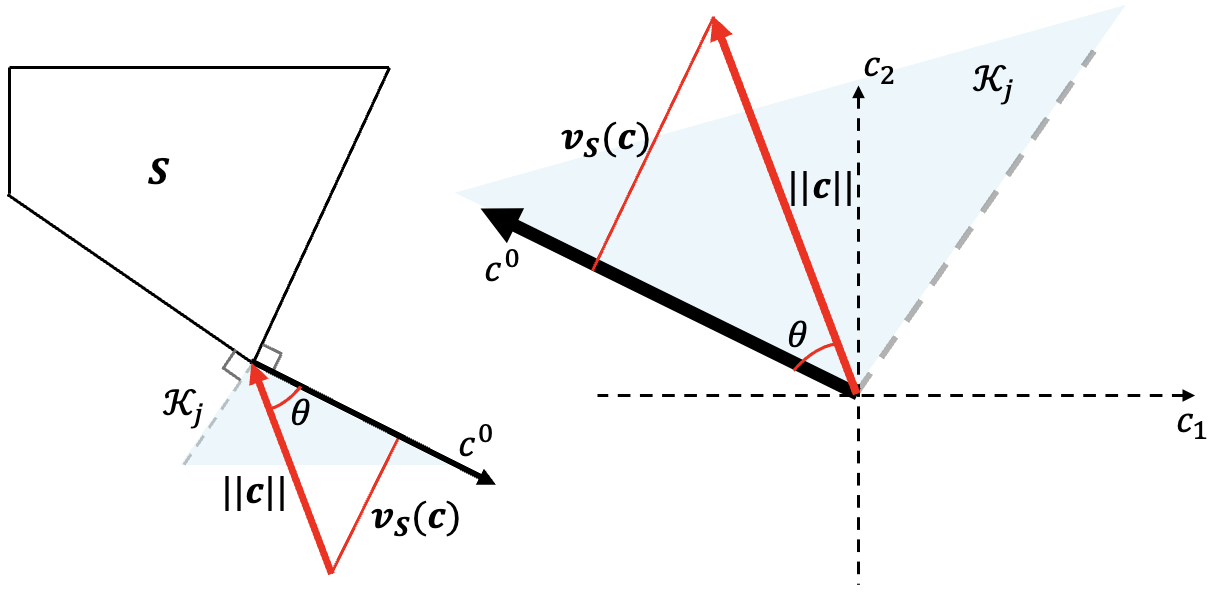}
\caption{Directional Margin Condition}
\label{fig:angle_margin}
\end{figure}
\vskip -0.5cm
This implies that the lower bound of $ \frac{\nu_\mathcal{S}(h(x))}{\|h(x)\|} $ is equivalent to the lower bound of the angle between $h(x)$ and its closest boundary of the cone that $h(x)$ lies in.\hfill \Halmos
\end{remark}

With Assumption \ref{assumption:tau_lower_bound}, Lemma \ref{lem:SPO_Lipschitz} shows that the SPO loss has a``Lipchitz" like property.
\begin{lemma}[``Lipschitz-like" property of $\lspo$ under directional margin condition]
    \label{lem:SPO_Lipschitz}
    Under Assumption \ref{assumption:tau_lower_bound}, for any $h_1, h_2 \in \mathcal{H}$ and any $(x,c)\in \mathcal{X}\times \mathcal{
    C
    }$, there exists a constant $L = \frac{\Delta(\mathcal{S})\rho(\mathcal{C})}{ 2\eta}$ such that $\lspo(.,c)$ is L-Lipchitz with respect to the directional difference, namely,
    $$
    |\lspo(h_1(x),c) - \lspo(h_2(x),c)| \leq L \left\|\frac{h_1(x)}{\|h_1(x)\|}-\frac{h_2(x)}{\|h_2(x)\|}\right\|.
    $$
\end{lemma}

The proof of Lemma \ref{lem:SPO_Lipschitz}, along with the remaining proofs, is provided in the appendix.

\subsection{Risk Bound and Rates}\label{sec:bounds}

With the “Lipschitz-like"  property established in Lemma \ref{lem:SPO_Lipschitz}, we are ready to show the convergence of IWSD-DU (Algorithm \ref{alg:general-iwal}) in Theorem \ref{thm:prediction}. The high-level idea of the proof is to construct a martingale difference sequence and apply Azuma's inequality to prove the risk bound under a finite sample guarantee.

\begin{theorem}[Risk Bound] \label{thm:prediction} Suppose that Assumption \ref{assumption:tau_lower_bound} holds. Let $\delta \in (0, 1]$ be a given parameter, and set $r_0 \gets \max\{2 \omega_\ell(\hat{\mathcal{C}},\mathcal{C}), 2\gamma L \sqrt{\log(2|\mathcal{H}|/\delta)}\}$ where $L$ is the Lipschitz constant in Lemma \ref{lem:SPO_Lipschitz}. Then, with probability at least $1 - \delta$, for all $T \geq 1$, the SPO risk satisfies 
\begin{equation*}
    \rspo(h_T)- \rspo(h^*)\leq 4\gamma L \sqrt{\frac{\log(2T|\mathcal{H}|/\delta)}{T}}.
\end{equation*}
\end{theorem}

In practice, the stopping time for the experimental design is determined by a tolerance for the decision loss, which is denoted by $\epsilon$. Based on Theorem \ref{thm:prediction}, we can derive an upper bound for the stopping time for the IWSD-DU under the tolerance $\epsilon$ in Theorem \ref{thm:stoppingtime1}.

\begin{theorem}[Sample complexity]\label{thm:stoppingtime1}Under the assumption of Theorem \ref{thm:prediction}, given any $\epsilon >0$ and $\delta \in (0,1)$, define 
$ \tau := \frac{3(4\gamma L)^2}{\epsilon^2}\log \frac{2|\mathcal{H}|}{\epsilon^2 \delta}$. We have that $\forall T\geq \tau$, with probability $1-\delta$, we have $\rspo(h_T)-\rspo(h^*)\leq \epsilon$.
\end{theorem}

The stopping time $\tau$ characterized in Theorem~\ref{thm:stoppingtime1} provides explicit control over the number of design points queried. By combining this stopping-time with the risk bound from Theorem~\ref{thm:prediction}, we obtain Proposition \ref{sublinear_iwal} that elucidates the asymptotic dependence on $T$.

\begin{proposition} \label{sublinear_iwal}
Under the same setting as Theorem~\ref{thm:prediction}, with probability at least $1-\delta$ for all $T \ge 1$, the excess SPO risk satisfies $\rspo(h_T) - \rspo^* \le \otilde(T^{-1/2}\sqrt{\log(1/\delta)})$. Equivalently, with probability at least $1-\delta$, there exists a stopping time $\tau \sim \otilde(\epsilon^{-2}\log(1/\delta))$ such that $\forall T \ge \tau$ we have $\rspo(h_T) - \rspo(h^*) \le \epsilon$.
\end{proposition}

Proposition~\ref{sublinear_iwal} implies a sample complexity on the order of $\otilde(\epsilon^{-2})$, which matches the typical sample complexity of the supervised learning approach in the predict-then-optimize regime (e.g., \cite{el2022generalization}). This result shows that the experimental design algorithm based on directional-uncertainty enjoys at least the same theoretical guarantees as these decision-blind approaches (e.g., see Theorem 2 for general reweighted regression problem in \cite{beygelzimer2009importance}). In the next section, we further compare our directional-uncertainty–based algorithm with the traditional $\ell_2$-norm–based method.

\section{Comparison with Decision-blind Sequential Design}\label{sec:refinedlow}

Traditional experimental design methods are typically decision-blind, focusing solely on improving predictive accuracy. However, for downstream linear optimization, the directional accuracy of parameter estimates is often more critical for decision quality than their overall magnitude (e.g., as measured by the $\ell_2$ norm). In this section, we theoretically illustrate the advantages of our proposed design method by contrasting our directional-uncertainty approach with standard benchmarks, as summarized in Figure~\ref{fig:decision_flowchart}.

\begin{figure}[ht]
\centering

\resizebox{\linewidth}{!}{ 
\begin{tikzpicture}[
    node distance=1.8cm and 2.2cm,
    box/.style={rectangle, draw, align=center, minimum width=3.5cm, minimum height=1cm},
    comment/.style={align=center, font=\normalsize	},
    methodbox/.style={rectangle, draw, dashed, align=center, font=\normalsize, minimum width=3.5cm},
    arrow/.style={->, thick}
]

\node[box] (expdesign) {Experimental Design};
\node[box, right=2.8cm of expdesign] (elim) {Model\\Elimination};
\node[box, right=of elim] (train) {Model\\Training};
\node[box, right=0.8cm of train] (eval) {Model\\Evaluation};

\draw[arrow] (expdesign) -- (elim);
\draw[arrow] (elim) -- (train);
\draw[arrow] (train) -- (eval);

\draw[arrow, bend right=20] (eval.north) to node[above] {Downstream Decision-making} (expdesign.north);

\node[methodbox, below=1.2cm of expdesign] (exp_meth) {
Collect Informative Data\\[2pt]
\textbf{Our method:} Directional Uncertainty\\
\textbf{Benchmark 1:} Prediction Uncertainty\\
\textbf{Benchmark 2:} Directional Uncertainty
};

\node[methodbox, below=1.2cm of elim] (elim_meth) {
Shrinking hypothesis class $H_t$\\[2pt]
\textbf{Our method:} SPO Loss\\
\textbf{Benchmark 1:} Squared Loss\\
\textbf{Benchmark 2:} Squared Loss
};

\node[methodbox, below=1.2cm of train] (train_meth) {
Select best prediction model\\
from the hypothesis class $H_t$
\\[2pt]
\textbf{Our method:} SPO Loss\\
\textbf{Benchmark 1:} Squared Loss\\
\textbf{Benchmark 2:} Squared Loss
};

\node[comment, below=1.2cm of eval] (eval_comment) {SPO Loss};

\draw[arrow] (exp_meth.north) -- (expdesign.south);
\draw[arrow] (elim_meth.north) -- (elim.south);
\draw[arrow] (train_meth.north) -- (train.south);
\draw[arrow] (eval_comment.north) -- (eval.south);

\end{tikzpicture} 
}

\caption{\normalsize Illustration of Directional Uncertainty-Based Method and Benchmark Methods for Experimental Design.}
\label{fig:decision_flowchart}
\end{figure}
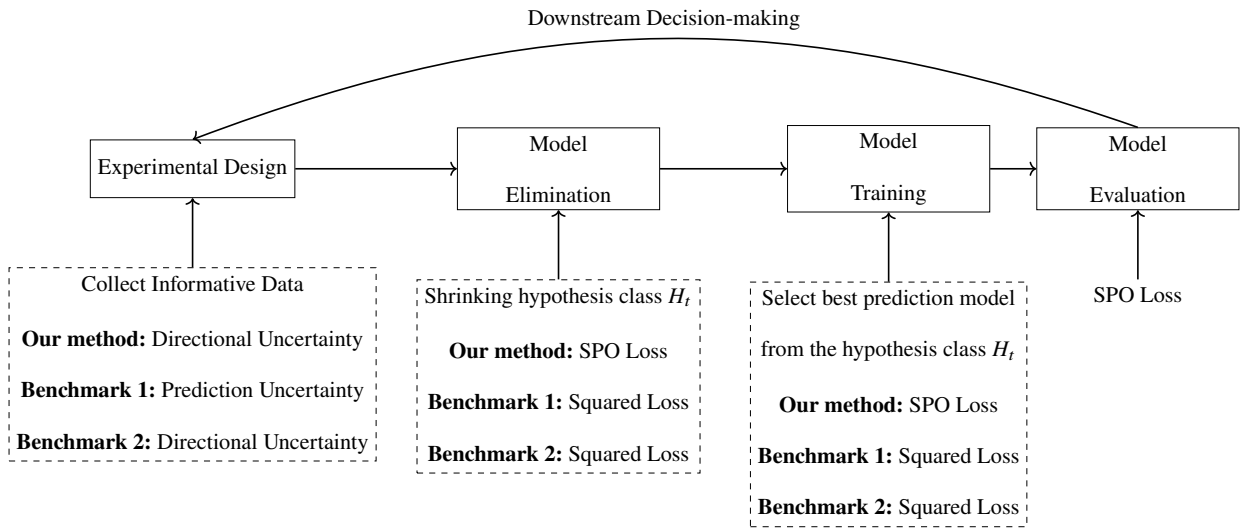
\vskip -1cm
We compare our method against two decision-blind benchmarks. The first benchmark is purely prediction-focused: its experimental design criterion is based on standard prediction uncertainty (measured by the $ \ell_2 $ norm), and its prediction model $ h_t  $ is trained and evaluated using the squared $ \ell_2 $ norm (i.e., standard least squares). This algorithm is formally stated in Algorithm \ref{alg:blind} in Appendix \ref{append:benchmarkalg}, and the only differences between this benchmark and Algorithm \ref{alg:general-iwal} are Line 4 and Line 12. The second benchmark represents a hybrid approach: it uses a design criterion based on directional uncertainty (normalized $ \ell_2  $ norm), but, like the first benchmark, it still relies on the squared $ \ell_2  $ norm for the training and elimination of the model $ h_t $.

\begin{remark}[Consistency of metrics for training and elimination]
It is well known in machine learning that the loss used for model training should match the loss used for model elimination. A mismatch between these losses leads to biased model selection and invalidates convergence guarantees; see \cite{beygelzimer2009importance, liu2023active}. For this reason, the two benchmarks above are the only reasonable baselines that preserve theoretical guarantees.
\end{remark}

\subsection{Comparison Under Optimality Gap Assumption}
\label{sec: compare}

To first compare our method with Benchmark Algorithm 1, we introduce the following mild assumption. This assumption ensures that a suboptimal hypothesis does not yield an SPO risk that is arbitrarily close to the optimal SPO risk.

\begin{assumption}[Optimality gap for decision risk]\label{assu:gap}
    There exsits a constant $\alpha>0$, such that $\rspo(h) - \rspo^* \ge \alpha$ for any $h \in \mathcal{H}\setminus \mathcal{H}^*$.
\end{assumption}

Assumption \ref{assu:gap} is considered moderate not only due to the finite cardinality of the hypothesis set $ \mathcal{H}  $, but also because the SPO risk $ \mathcal{R}_{\text{SPO}}(h)  $ is typically discontinuous with respect to the predicted cost vector, which often results in a non-zero optimality gap. We provide a more detailed discussion and an illustrative example of this optimality gap in Section \ref{sec:example}.
Given a fixed constant $ \alpha > 0  $, we denote the class of distributions that satisfy Assumption \ref{assu:gap} by $ \mathcal{P}_\alpha $, i.e.,
$
\mathcal{P}_\alpha := \left\{\mathcal{D}: \mathcal{R}_{\text{SPO}}(h) - \mathcal{R}_{\text{SPO}}^* \ge \alpha \text{ for any } h \in \mathcal{H}\setminus \mathcal{H}^* \right\}
.$
Within this class of distributions $\mathcal{P}_\alpha $, Theorem \ref{thm:better} demonstrates that our proposed Algorithm \ref{alg:general-iwal} has a bounded and finite stopping time to identify the best prediction model $ h^*  $ for any distribution $ \mathcal{D} \in \mathcal{P}_\alpha $. In contrast, we show that the decision-blind algorithms may not converge to the true optimal hypothesis.

\begin{theorem}[Earlier stopping time than decision-blind design]
\label{thm:better} Suppose that Assumptions \ref{assumption:tau_lower_bound} and  \ref{assu:gap} hold with $\alpha > 0$. We have the following results:
\begin{enumerate}
    \item For Algorithm \ref{alg:general-iwal}, for any $\delta \in (0,1)$, there exists a constant $T_0$ (depending on $\alpha$ and $\delta$), such that for any distribution $\mathcal{D}\in \mathcal{P}_\alpha$ and for any $T \ge T_0$, with probability at least $1 - \delta$, the candidate prediction set $H_T$ has excluded all suboptimal prediction models, i.e., $H_T \subseteq \mathcal{H}^*$.
    \item For the decision-blind design (Algorithm \ref{alg:blind}), there exists a distribution of $\mathcal{D} \in \mathcal{P}_\alpha$ and a constant $\delta_0$, such that for any $T$, the probability of having some suboptimal prediction models within $H_T$ is larger than $\delta_0$, i.e., with probability at least $\delta_0$, there exists $h \in H_T$, such that $\rspo(h)> \rspo^*$.
\end{enumerate}
\end{theorem}

Theorem \ref{thm:better} demonstrates a fundamental limitation of decision-blind algorithms. We achieved this by constructing a specific distribution where two distinct prediction models have a prediction error that is symmetric around the true conditional mean, yet these models lead to different optimal decisions. This scenario clearly shows that decision-blind algorithms, which minimize prediction error (e.g., squared $\ell_2$ loss), may fail to converge to the true optimal hypothesis (i.e., the one minimizing SPO risk). This theoretical result reinforces our earlier observation regarding the misalignment between prediction and decision objectives.

Theorem \ref{thm:better} establishes an earlier stopping time for our proposed Algorithm \ref{alg:general-iwal}, this performance gain stems from two key differences relative to Benchmark 1: the integration of the SPO loss criterion and the normalization of the prediction vector in the experimental design. 

Notably, normalizing prediction vectors can independently accelerate convergence even without using the SPO loss criterion. In particular, Example~\ref{example:normalized_better} illustrates that, under certain distributional settings, normalization leads to an earlier stopping time compared with the traditional decision-blind design. This improvement occurs even when the squared $\ell_2$ loss is still used for model training and elimination, demonstrating that the directional-uncertainty (normalized $\ell_2$ difference) criterion can provide benefits on its own.

\begin{example}[Directional uncertainty yields an earlier stopping time under squared loss]\label{example:normalized_better}
In this example, we show that Benchmark 2 is better than Benchmark 1, as depicted in Figure \ref{fig:decision_flowchart}. Specifically, Benchmark 2 in Figure \ref{fig:decision_flowchart} still strategically leverages directional uncertainty (normalized $\ell_2$ difference) to guide the experimental design criterion, but, unlike our proposed method, it uses the standard squared $\ell_2$ loss for both model training and hypothesis elimination.

Suppose that there are two candidate designs $x_1$ and $x_2$, and that we have two candidate prediction models, $h_1$ and $h_2$. The prediction values under two prediction models with designs $x_1$ and $x_2$ are shown in Figure \ref{fig:normalized_better_example}. The distributions of $\bbE[c|x_1]$ and $\bbE[c|x_2]$ are also shown in Figure \ref{fig:normalized_better_example}. Particularly, for design $x_2$, all three vectors $\bbE[c|x_2]$, $h_1(x_2)$ and $h_2(x_2)$ have the same direction. For design $x_1$, the prediction $h_1(x_1)$ is in the same cone as $\bbE[c|x_1]$, while $h_2(x_1)$ is in a different cone. Thus, in terms of the SPO loss, $h_1$ is optimal. In terms of the squared loss, there is flexibility to control whether $h_1$ or $h_2$ minimizes the expected squared loss over two designs $x_1$ and $x_2$. As an illustration, suppose that $x_1$ and $x_2$ occur with equal probability, and that $\|h_1(x_1) - \bbE[c|x_1]\| = \|h_2(x_2) - \bbE[c|x_2]\|$, as well as 
$\|h_2(x_1) - \bbE[c|x_1]\| = \|h_1(x_2) - \bbE[c|x_2]\|$.
Consequently, in this example, $h_1$ and $h_2$ have the same squared loss. As an illustration, we further assume that $\| h_1(x_2) - h_2(x_2)\| \ge \| h_1(x_1) - h_2(x_1)\| $. Thus, in Benchmark 1, when quantifying the uncertainty using the $\ell_2$ norm, design $x_2$ has a higher probability of being sampled.

\begin{figure}[ht]
    \centering
    \includegraphics[width=0.6\linewidth]{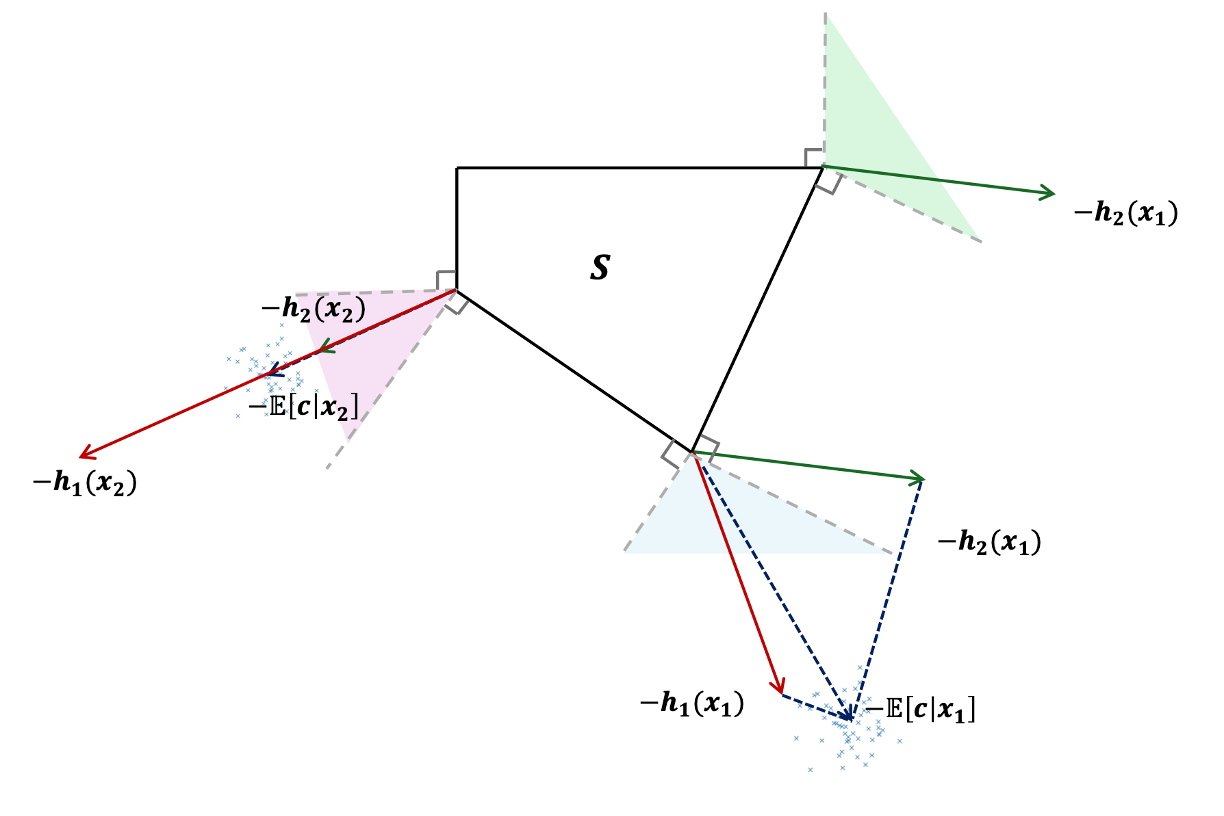}
    \caption{An Example Showing the Benefit of Normalization When Using the Squared Loss for Decision-Focused Learning}
    \label{fig:normalized_better_example}
\end{figure}

\vskip -1cm
In Benchmark 1, since $h_1$ and $h_2$ have the same expected squared loss, for any sufficiently large iteration $T$, by the central limit theorem, the empirical re-weighted loss of $h_1$ and $h_2$ will be approximately the same. Consequently, there exists a constant $\delta_0$ (e.g., $\delta_0 = 1/4$) such that, with probability at least $\delta_0$, $h_2$ is included in the confidence set $H_T$. In other words, after $T$ iterations, with probability at least $\delta_0$, the optimal predictor $h_1$ still exhibits a higher empirical reweighted squared loss than $h_2$, leading to the incorrect selection of $h_2$.

On the other hand, when we use the directional uncertainty, the uncertainty associated with design $x_2$ is zero, since $h_1(x_2)$ and $h_2(x_2)$ are in the same direction. In this case, the algorithm based on the directional uncertainty will select design $x_1$ almost surely. Since $h_1(x_1)$ is closer to $\bbE[c|x_1]$ than $h_2(x_1)$ to $\bbE[c|x_1]$, by the central limit theorem, for any sufficiently large $T$, the empirical squared loss of $h_1$ is smaller than $h_2$, and thus $h_2$ will be excluded from the confidence set $H_T$. In other words, $\forall \delta>0$ and  any sufficiently large $T$, with probability at least $1-\delta$, Benchmark~2 satisfies $H_T = \mathcal{H}^* = \{h_1\}$ with probability at least $1 - \delta$. \hfill \Halmos
\end{example}

Example \ref{example:normalized_better} demonstrates that utilizing the directional uncertainty criterion for experimental design yields superior performance, even when the squared $\ell_2$ loss is retained as both the training and hypothesis elimination criterion. This result is attributable to a simple yet important property that the SPO loss is inherently scale-invariant with respect to the cost vector prediction.

\section{Justifications and Examples of Assumptions}\label{sec:theoretical_justification}

Recall that the convergence guarantee of Algorithm \ref{alg:general-iwal} requires Assumption \ref{assumption:tau_lower_bound}, while its theoretical performance advantage over benchmark methods requires Assumption \ref{assu:gap}. In this section, we demonstrate that these assumptions are mild by providing concrete examples and simple transformations that satisfy them.

\subsection{Justification of the Directional Margin Condition}\label{sec:justification_angle}

In this section, we show that for any instances of the finite hypothesis class $\mathcal{H}$, Assumption \ref{assumption:tau_lower_bound} can be achieved by taking a simple pre-processing step. Let us first introduce the concept of the equivalent hypothesis class.

\begin{definition}[Equivalent Hypothesis Class]
\label{def:equiv_hypothesis}
Given two hypothesis classes $\mathcal{H}_1$ and $\mathcal{H}_2$ and a sample space $\mathcal{X}$, we say that $\mathcal{H}_1$ and $\mathcal{H}_2$ are equivalent with respect to $\mathcal{X}$ if the following conditions hold:
\begin{enumerate}
    \item For every $x \in \mathcal{X}$ and every $h \in \mathcal{H}_1$, there exists $\tilde{h} \in \mathcal{H}_2$ such that $R_{SPO}(h(x)) = R_{SPO}(\tilde{h}(x))$.
    \item For every $x \in \mathcal{X}$ and every $h \in \mathcal{H}_2$, there exists $\tilde{h} \in \mathcal{H}_1$ such that $R_{SPO}(h(x)) = R_{SPO}(\tilde{h}(x))$.
\end{enumerate}
\end{definition}

The definition of an equivalent hypothesis class essentially states that for every hypothesis in the original class, there exists a corresponding hypothesis in the transformed class with the same SPO risk. Consequently, replacing a given hypothesis class with an equivalent one does not alter the fundamental optimization problem of identifying the hypothesis that minimizes the SPO risk.

Now we state the procedure of pre-processing and how it works: First, for any $\mathcal{S}= \text{conv}\{v_1,\dots,v_k\}$, we can calculate non-degenerate cost vectors $c_1,\dots c_k$ with respect to each cone, where $c_i$ is any vector in the interior of $\mathcal{K}_i$, let $\eta = \min_{i=1,...,k} \frac{\nu_{\mathcal{S}}(c_i)}{\|c_i\|}$. We then pass the original hypothesis class into Algorithm~\ref{alg:aug-hypothesis} to obtain the augmented class $\tilde{\mathcal{H}}$. Proposition \ref{prop:equiv_hypothsis} shows that $\tilde{\mathcal{H}}$ is an equivalent hypothesis class with the desired non-degeneracy properties.

\begin{algorithm}
        \caption{Pre-processing: Augmented Hypothesis Class}\label{alg:aug-hypothesis}
    \begin{algorithmic}
		\STATE \textbf{Input:} Hypothesis class $\mathcal{H}$, non-degenerate cost vectors $c_1,...,c_k$
    \STATE $\tilde{\mathcal{H}} \gets \mathcal{H}$
        \FOR{$i = 1,...,n$}
            \FOR{$j = 1,...,|\mathcal{H}|$}
             
            \IF{$\frac{\nu_{\mathcal{S}}(h_j(x_i))}{\|h_j(x_i)\|}<\eta$}
            \STATE Let $\tilde{h}_j(x) \gets \begin{cases}
			c, &\text{if $x = x_i$}\\
			h_j(x), & \text{otherwise},
		\end{cases}$
        
        where $c\in\{c_1,...,c_k\}$ is the non-degenerate cost vector corresponding to the cone $\mathcal{K}$ which contains $ h_j(x_i)$
        \ENDIF
        
            \ENDFOR
        \ENDFOR

        \STATE \textbf{Return} $\tilde{\mathcal{H}}$
    \end{algorithmic}
\end{algorithm}

\begin{proposition}
\label{prop:equiv_hypothsis} Given the set of possible designs $\mathcal{X}$, for any hypothesis prediction model class $\mathcal{H}$, let $\tilde{\mathcal{H}}$ be the pre-processed class output by Algorithm \ref{alg:aug-hypothesis} after computing non-degenerate cost vectors $c_1,...,c_k$ and $\eta = \min_{i=1,...,k} \frac{\nu_{\mathcal{S}}(c_i)}{\|c_i\|}$, the following two arguments hold:
    \begin{enumerate}
        \item Assumption \ref{assumption:tau_lower_bound} is satisfied with parameter $\eta$ under $\tilde{\mathcal{H}}$.
        \item $\mathcal{H}$ and $\tilde{\mathcal{H}}$ are equivalent hypothesis classes.
    \end{enumerate}
\end{proposition}

Proposition~\ref{prop:equiv_hypothsis} implies that Assumption~\ref{assumption:tau_lower_bound} can be easily satisfied for any hypothesis class via the preprocessing step in Algorithm~\ref{alg:aug-hypothesis}.

\subsection{Examples of Optimality Gap}\label{sec:example}

In Section~\ref{sec: compare}, we show that the optimality gap (Assumption~\ref{assu:gap}) can be used to formally establish the superior performance of our decision-focused algorithm relative to the decision-blind benchmark. We also note that, empirically, our algorithm continues to outperform the decision-blind alternative even in the absence of such a gap, as demonstrated in Appendix~\ref{append:distribution_of_gap}.

In this section, we present examples illustrating the existence of the optimality gap in Assumption~\ref{assu:gap} without relying on the standard well-specification assumption, namely that the optimal predictor $h^*(x)$ equals the conditional expectation $\mathbb{E}[c \mid x]$. Our construction is based on three mild regularity conditions: (i) a relaxation of well-specification, (ii) a sensitivity condition for a particular experimental design, and (iii) a regularity condition on the prediction set and the polyhedral structure of the feasible decision space.

\begin{condition}\label{assu:risk_minimizer}
The minimum risk equals $\rspo(\bbE[c|\cdot])$, i.e., 
     $\min_{h\in \mathcal{H}}\rspo(h) = \bbE[\lspo(\bbE[c|x],c)]$.
\end{condition}
The above condition is a relaxed, well-specified condition on the hypothesis class. It is less restrictive than assuming $h^*(x) = \bbE[c \mid x]$, since $\ell_{\mathrm{SPO}}(\alpha\, \bbE[c \mid x], c) = \ell_{\mathrm{SPO}}(\bbE[c \mid x], c)$ for any $\alpha > 0$. The intuition is that any predictor proportional to the conditional expectation is also optimal under the SPO loss.

Under Condition \ref{assu:risk_minimizer}, for any design $\X_{(j)} \in \X$ we have 
$$
\begin{aligned}
\bbE[\lspo(h(x),c) - \lspo(h^*(x),c)|x = \mathrm{X}_{(j)}] 
&= \bbE[\lspo(h(x),c) - \lspo(\bbE[c|x],c)|x = \mathrm{X}_{(j)}] 
\\&= \bbE[c^\top \left(w^*(h(x)) - w^*(\bbE[c|x]\right)|x = \mathrm{X}_{(j)} ]
\\ &= \bbE[c^\top |x = \mathrm{X}_{(j)} ]\left(w^*(h(\mathrm{X}_{(j)})) - w^*(\bbE[c|\mathrm{X}_{(j)}])\right)\geq 0,
\end{aligned}
$$
where the last inequality used the fact that $w^*(\bbE[c|\mathrm{X}_{(j)}])$ is the minimizer of $\min_{w \in \mathcal{S}}(\bbE[c|\mathrm{X}_{(j)}])^\top w $. Hence, under condition \ref{assu:risk_minimizer}, the difference in the risk is always non-negative conditional on any $x \in \mathcal{X}$. This enables us to lower bound the difference of $\rspo$ and $\rspo^*$ by conditioning on any specific design $\mathrm{X}_{(i)} \in \X$. 

In order to further lower bound the difference of $\rspo$ and $\rspo^*$ condition on $\mathrm{X}_{(i)} \in \X$, we need conditions on a necessary condition of sub-optimal decisions. We first give the following definition.

Define $\Delta$ as the supremum over the norm of the set of vector $v$ such that $ \forall x \in \X$, $h^* \in \mathcal{H}^*$, its downstream optimizer $w^*(h^*(x)) = w^*\left(\frac{h^*(x)}{\|h^*(x)\|}\right)$ stays the same after perturbation $\frac{h^*(x)}{\|h^*(x)\|} \gets \frac{h^*(x)}{\|h^*(x)\|} + v$, i.e.,

$$    \Delta := \left\{ \sup_{v \in \bbR^d:\|v\|\le 2}: \|v\|~\Bigg|~ w^*\left(\frac{h^*(x)}{\|h^*(x)\|} + v\right) = w^*(h^*(x)),~~ \forall x\in\mathcal{X},\forall h^*\in H^* \textnormal{} \right\},$$

Intuitively, $\Delta$ serves as the uniform upper bound on the magnitude for which it is safe to change the prediction without changing the output decision.  
Now we can state a necessary condition for making a suboptimal decision. 

\begin{condition}
    \label{assu:x_delta}
    There exists $\mathrm{X}_{(i)} \in \mathcal{X}$ with $ \underline{\mu}:= \mu(\mathrm{X}_{(i)})>0$ such that 
     
    $$\left\|\frac{h^*(\mathrm{X}_{(i)})}{\|h^*(\mathrm{X}_{(i)})\|} - \frac{h'(\mathrm{X}_{(i)})}{\|h'(\mathrm{X}_{(i)})\|} \right\| > \Delta, \quad \text{ for all }h' \in \mathcal{H}\setminus \mathcal{H}^*, h^* \in \mathcal{H}^*. $$
\end{condition}

Condition~\ref{assu:x_delta} states that there exists a design $\mathrm{X}_{(j)} \in \mathcal{X}$ with strictly positive probability mass $\underline{\mu}$ such that a suboptimal prediction leads to a suboptimal downstream decision. Combined with Condition~\ref{assu:risk_minimizer}, this allows us to lower bound $\rspo - \rspo^*$ by the objective difference induced by a suboptimal decision. To translate this decision-level distinction into a gap in SPO risk, we impose the third condition.

\begin{condition}
    \label{assu:risk_gap}
    Suppose that the norm of the conditional expectation $\bbE[c \mid x]$ is bounded below by a positive constant $\underline{c}$, that is, $\|\bbE[c \mid x]\| \ge \underline{c}$ for all $x$. Recall the the convex hull representation of $\mathcal{S} $ is $ \mathcal{S} = \text{conv}\{v_1,\dots,v_k\}$, define the minimum distance between vertices of $\mathcal{S}$ as
    $\underline{\gamma} := \min_{i,j \in \{1,\dots,k\},i \neq j} \|v_i - v_j\|$ and suppose $ \underline{\gamma}>0$.
\end{condition}

Condition~\ref{assu:risk_gap} requires that neither the norm of the true conditional expectation nor the minimum distance between vertices of $\mathcal{S}$ is arbitrarily close to zero. Together, Conditions~\ref{assu:risk_minimizer}, \ref{assu:x_delta}, and \ref{assu:risk_gap} serve as mild regularity assumptions that rule out pathological or irregular distributions and are commonly satisfied in practice. We are now ready to present an example that shows how to construct an optimality gap under these conditions.

\begin{proposition}
\label{prop:gap_example}
Under Conditions~\ref{assu:risk_minimizer}, \ref{assu:x_delta}, and \ref{assu:risk_gap}, the optimality gap in Assumption~\ref{assu:gap} holds with $\alpha = \underline{\mu}\,\eta\,\underline{c}\,\underline{\gamma}$.
\end{proposition}

Proposition~\ref{prop:gap_example} provides a concrete example ensuring the existence of a non–arbitrarily small optimality gap. Together with the analysis in Section~\ref{sec: compare}, this optimality gap can be used to establish the superiority of our algorithm over the benchmark methods.

\section{Numerical Study}\label{sec:numerical}

To illustrate the effectiveness of our approach, we conducted two sets of numerical studies: one on the cost-sensitive classification problem and another on the LLM resource allocation problem. We benchmarked:
\begin{enumerate}
    \item \textbf{(Our Method)} Directional uncertainty with SPO loss-based model elimination/training
    \item $\ell_2$ uncertainty with SPO loss-based model elimination/training
    \item Directional uncertainty with $\ell_2$ loss-based model elimination/training
    \item $\ell_2$ uncertainty with $\ell_2$ loss-based model elimination/training
\end{enumerate}
Methods~1 and~2 use the SPO loss to shrink the hypothesis set, whereas Methods~3 and~4 rely on the squared $\ell_2$ loss. The contrast between Methods~1 and~3 versus Methods~2 and~4 isolates the impact of using directional uncertainty under a fixed loss criterion.

\subsection{Cost-Sensitive Classification}
\label{sec:cardio}

The cost-sensitive multiclassification problem can be formulated as a contextual stochastic linear optimization problem, as detailed in Appendix \ref{sec:cost_sensitive},. We now apply this framework to a modern clinical diagnosis setting, where a predictive model is trained to determine whether a patient has a disease based on a dataset containing objective, examination, and subjective features. If the model can accurately predict disease presence, the cost of performing precise diagnostic procedures can be avoided. At the same time, the cost-sensitive formulation explicitly accounts for the asymmetry between false negatives and false positives, thereby aligning the learning objective with the operational goal of minimizing overall diagnostic costs.

We consider an experimental design setting in clinical diagnosis, where controlled experiments can be conducted on different patients to obtain their true health conditions. However, performing precise diagnostic tests is often costly. Therefore, we aim to develop an efficient data collection pipeline that enables accurate model training while minimizing experimental costs.

Specifically, we used the cardiovascular disease dataset from Kaggle\footnote{\href{https://www.kaggle.com/datasets/sulianova/cardiovascular-disease-dataset}{https://www.kaggle.com/datasets/sulianova/cardiovascular-disease-dataset}. This dataset has been extensively studied in medical research; see, for example, \cite{sadr2024cardiovascular}.}. In cardiovascular diagnosis, the conventional approach involves conducting an MRI test, which is precise but costly. Alternatively, a predictive model can serve as a clinical decision-support tool by estimating the likelihood of disease presence based solely on readily available patient features.

The dataset contains more than 60,000 patient records, each indicating whether the individual has cardiovascular disease. To construct the candidate hypothesis set $\mathcal{H}$, we randomly sampled $N \in \{10, 25, 50, 100, 10000\}$ data points from the dataset and minimized the SPO$+$ loss to obtain five linear predictors (linear mappings from features to labels). Each minimization was performed using stochastic gradient descent with a step size of $5 \times 10^{-3}$, a batch size of 256, and 12 epochs. In this manner, we obtained five linear predictor hypotheses, where the first four are regarded as suboptimal and the fifth as the best-performing hypothesis. 

To align with our experimental design framework, we employed the K-means clustering algorithm to identify clusters that represent distinct patient types based on empirical density. The number of clusters was set to 15. In each iteration of our algorithm, we consider the experimental design to be selecting a sample from one of the cluster centers and drew its corresponding true label from the same cluster.

Finally, to evaluate the effectiveness of our algorithm relative to the benchmark methods, we conducted 100 trials for each method and computed the mean SPO loss on the testing set. An example illustrating the testing set loss trajectories with 95\% confidence intervals of our algorithm and the benchmarks is shown in Figure \ref{fig:cardio_1}, where the \textit{cost ratio} refers to the ratio between unit false negative cost and false positive cost, and is set to 3 in the basic setting.

Another important consideration is how model performance changes when the objective value becomes more asymmetric in the cost of misclassification. In the cost-sensitive classification problem, this sensitivity is directly linked to the cost ratio. When the ratio is low, making accurate predictions aligns closely with making optimal decisions; when the ratio is high, this alignment weakens, and the objective becomes more sensitive to decision quality. To examine this effect, we varied the cost ratio and plotted the testing set loss ratio of our methods relative to decision-blind benchmarks after 100 iterations. Figure \ref{fig:loss_to_costratio} presents the test-set loss ratio with 95\% confidence intervals across cost ratios ranging from 1 to 15.

\begin{figure}[htbp]
  \centering
   
  \begin{minipage}[t]{0.46\textwidth}
    \centering
    \includegraphics[width=\linewidth]{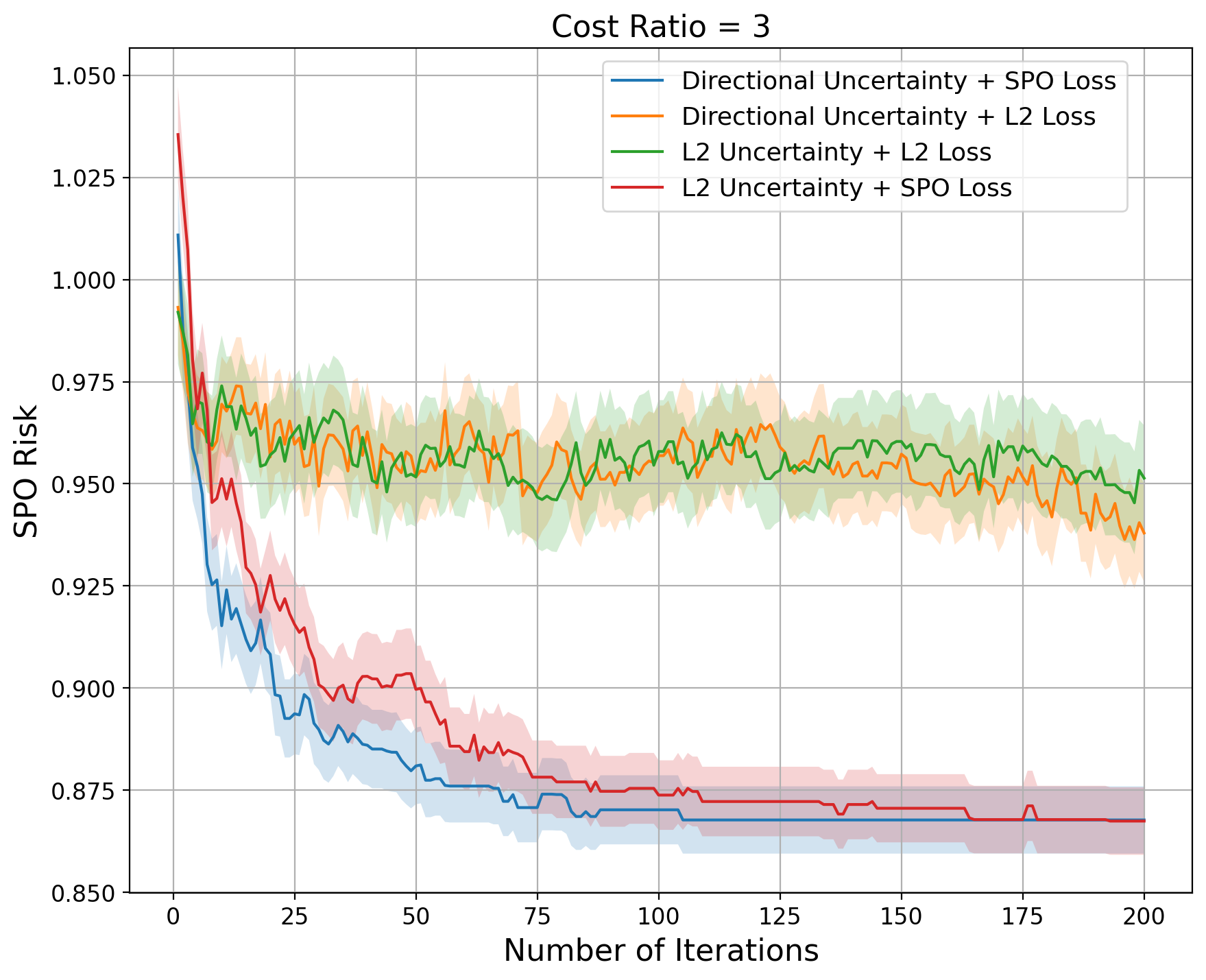}
     
    \caption{Cost-Sensitive Classification}
    \label{fig:cardio_1}
  \end{minipage}
  \hfill                 
  \begin{minipage}[t]{0.46\textwidth}
    \centering 
    \includegraphics[width=1\linewidth]{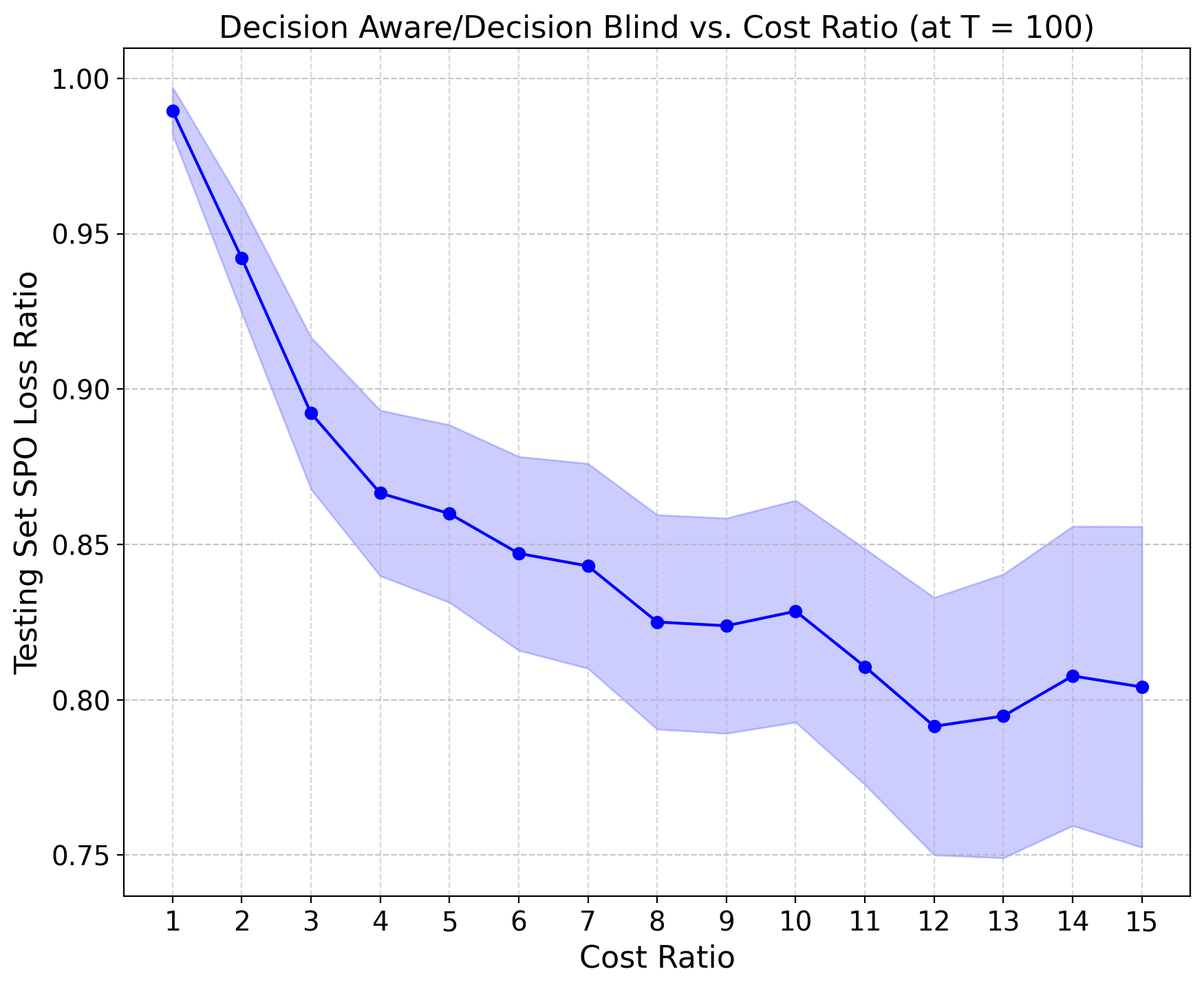}
    \caption{SPO Loss Ratio vs. Cost Ratio} 
    \label{fig:loss_to_costratio}
  \end{minipage}
\end{figure}

Figure \ref{fig:cardio_1} reveals the following key insights:
\begin{itemize}
    \item Using the SPO loss as the elimination criterion leads to a substantial reduction in loss. This can be observed by comparing the directional-uncertainty-based and $\ell_2$-uncertainty-based algorithms trained with $\ell_2$ loss against their counterparts trained with $\lspo$ loss.
    \item Incorporating directional uncertainty yields a moderate improvement, as shown by comparing the directional-uncertainty-based algorithms with the $\ell_2$-uncertainty-based algorithms under both $\ell_2$ and $\lspo$ loss criteria.
     
\end{itemize}

On the other hand, Figure \ref{fig:loss_to_costratio} showed that, as the cost ratio increases, the performance advantage of our decision-focused method over the decision-blind benchmarks grows correspondingly.

These results demonstrate that incorporating directional uncertainty provides consistent improvements, while minimizing $\lspo$ better aligns the learning objective with decision-making performance. Further experimental details and results are provided in Appendix \ref{appendix:cost_sensitive}.

\subsection{LLM Assignment Problem}

In Section \ref{sec:intro_example}, we demonstrated that the sequential design can also be applied to an LLM job assignment problem. Following this formulation, we conducted an experiment using collected web text data and incorporated human evaluation.

We consider a setting where text generation tasks for an input from a specific category are assigned to different LLMs. We assumed there is a consistent quality score measure for these generated texts, and our objective is to maximize this score for each generated text (see evaluation rubrics in Appendix \ref{appendix:llm}.). We also assumed the text generation tasks are in a parallel setting where the tasks are not sequentially dependent, which allows us to adopt the formulation in Section \ref{sec:intro_example}. We assume the hypotheses are mappings from text categories to task scores. In this way, by predicting task performance for any given text category, we can solve the assignment problem and determine an optimal allocation of the models.

In our experiment setting, we considered five tasks: text summarization, text extension, text translation, bullet point generation, and sentiment analysis. These tasks cover two main categories: condensing tasks, which reduce or analyze texts, and generative tasks, which create new text beyond the given texts. Condensing tasks are more common tasks while generative tasks requires more model knowledge and capability.

Regarding the LLMs, we used ChatGPT5, Gemma3, Qwen1.5, and Deepseek-llm. Among these models, ChatGPT5 represents the more advanced but expensive models, while the other three represent older, less powerful, and more accessible models. Reflecting this categorization, we imposed capacity constraints: ChatGPT5 was limited to a single use, while the other three models could each be used up to two times.

To collect and label data, we gathered 50 short texts from five diverse categories: Wikipedia mathematics, Wikipedia animals, medical journal abstracts, UNC alumni news, and movie reviews. These categories were chosen to represent a wide range of text types, with different topics and tones. The generated output for each text was then assessed by human evaluators using a 1-to-5 quality scale (1=poor, 5=perfect) based on a consistent criterion. The summary of the task scores is provided in Table \ref{table:average_scores} in Appendix \ref{appendix:llm}.

In the experiment, each prediction model was constructed by drawing bootstrap samples of size $N$ from each category (with replacement) and taking their averages. The candidate hypothesis class $\mathcal{H}$ was then formed by varying $N \in \{1,3,5,7,10\}$. This resampling technique was used to overcome the problem of limited data, with scores calculated as the categorical sample means. For each method, we conducted 100 trials and computed the average SPO loss and $\ell_2$ loss on the testing set, which comprises 20\% of the dataset. The results are shown in Figure \ref{fig:llm_spo} and  Figure \ref{fig:llm_l2}.

\begin{figure}[htbp]
 
  \centering
   
  \begin{minipage}[t]{0.46\textwidth}
    \centering
    \includegraphics[width=\linewidth]{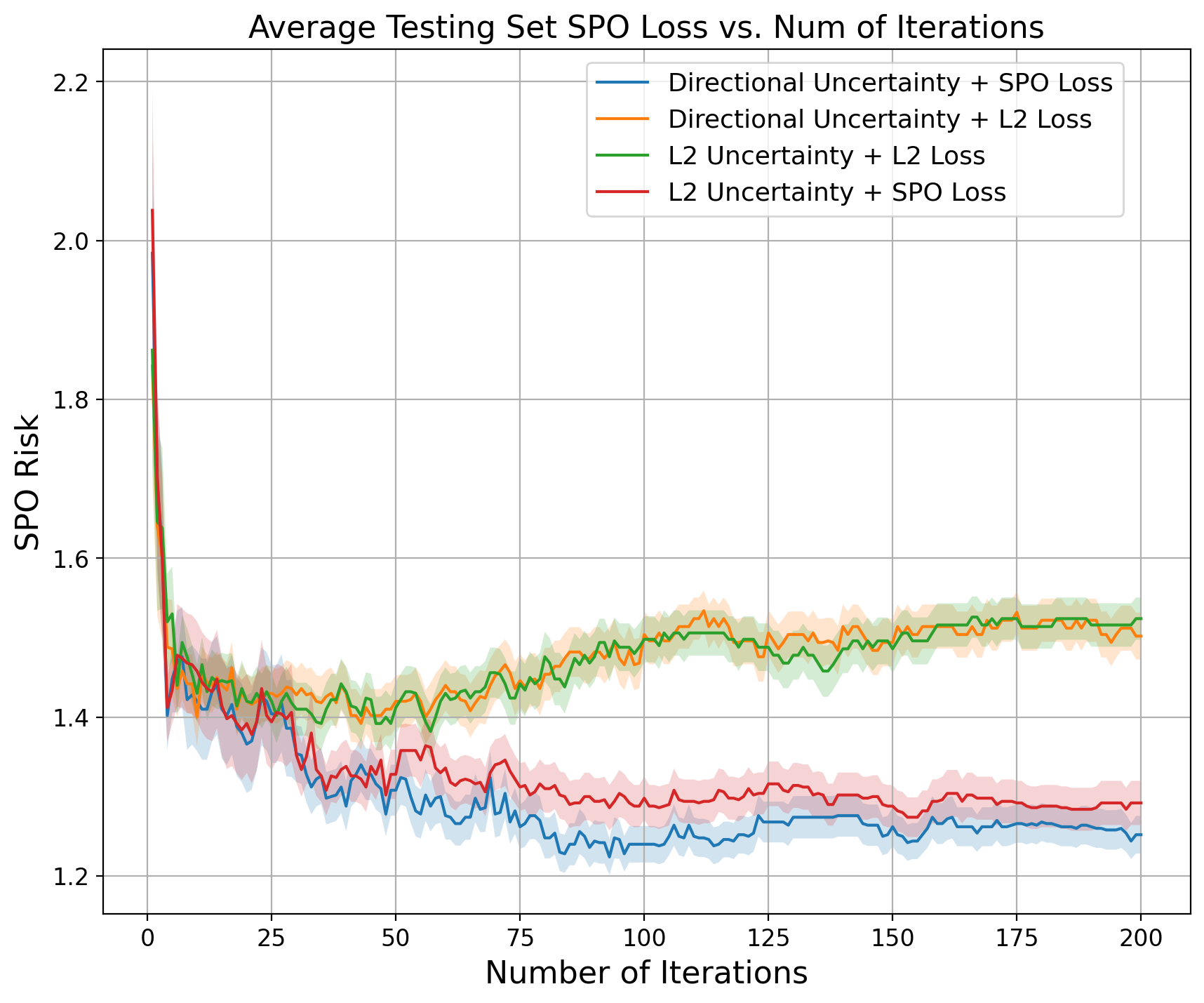}
    \caption{LLM Allocation Problem SPO Loss}
    \label{fig:llm_spo}
  \end{minipage}
  \hfill                 
  \begin{minipage}[t]{0.46\textwidth}
    \centering 
    \includegraphics[width=1\linewidth]{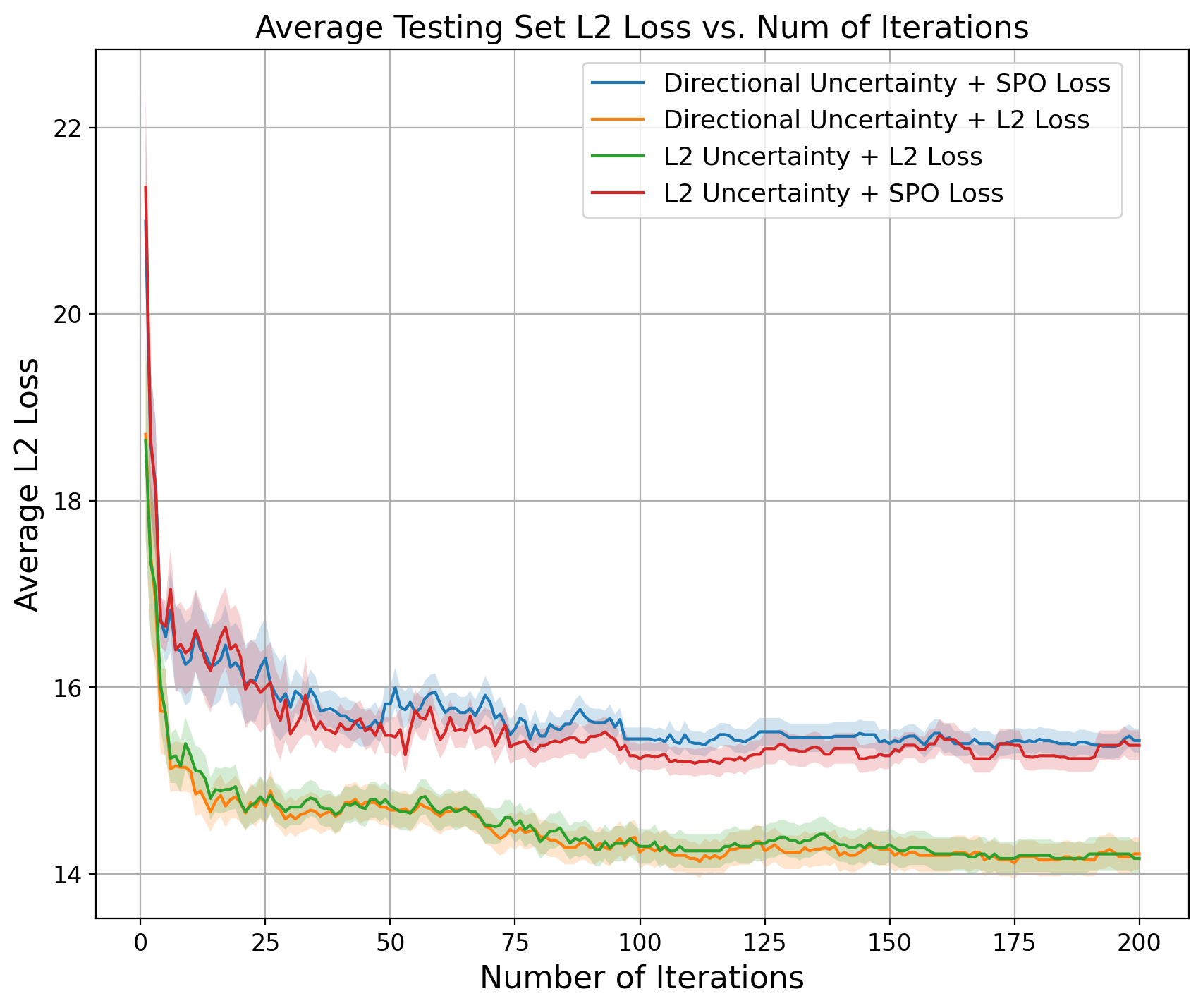}
    \caption{LLM Allocation Problem L2 Loss} 
    \label{fig:llm_l2}
  \end{minipage}
\end{figure}

\vskip -1cm 
For decision quality, the results in Figure \ref{fig:llm_spo} are consistent with the cost-sensitive classification experiment. Our decision-focused algorithm with directional uncertainties outperforms other benchmarks. Moreover, Figure \ref{fig:llm_l2} further highlights the fundamental mismatch between prediction and decision objectives. Although benchmark decision blind methods perform well on the $\ell_2$ prediction metric, they fail to converge to a predictor that induces the optimal decision. 

These results from a real-world experiment demonstrate the strength of our algorithm and provide a concrete example of the mismatch induced by decision-blind algorithms. More details on the experimental setup are provided in Appendix \ref{appendix:llm}.

\section{Conclusion}\label{sec:conclude}

In this paper, to address the misalignment between decision loss and traditional experimental design criteria based solely on prediction uncertainty, we propose a new uncertainty measure that captures the \emph{directional} uncertainty of predictions. This criterion leads to a theoretically earlier stopping time compared with decision-blind design methods. Real-world numerical experiments further demonstrate the advantages of our approach over benchmark algorithms. This directional uncertainty-based criterion provides new insights into quantifying uncertainty in the predict-then-optimize paradigm for future studies.

\bibliographystyle{informs2014}
\bibliography{reference}

@inproceedings{mandi2020smart,
  title={Smart predict-and-optimize for hard combinatorial optimization problems},
  author={Mandi, Jayanta and Stuckey, Peter J and Guns, Tias and others},
  booktitle={Proceedings of the AAAI conference on artificial intelligence},
  volume={34},
  number={02},
  pages={1603--1610},
  year={2020}
}

@article{el2022generalization,
  title={Generalization bounds in the predict-then-optimize framework},
  author={El Balghiti, Othman and Elmachtoub, Adam N and Grigas, Paul and Tewari, Ambuj},
  journal={Mathematics of Operations Research},
  year={2022}
}

@inproceedings{castro2005faster,
  title={Faster rates in regression via active learning},
  author={Castro, Rui and Willett, Rebecca and Nowak, Robert},
  booktitle={NIPS},
  volume={18},
  pages={179--186},
  year={2005}
}

@article{sugiyama2009pool,
  title={Pool-based active learning in approximate linear regression},
  author={Sugiyama, Masashi and Nakajima, Shinichi},
  journal={Machine Learning},
  volume={75},
  number={3},
  pages={249--274},
  year={2009},
  publisher={Springer}
}

@article{cai2016batch,
  title={Batch mode active learning for regression with expected model change},
  author={Cai, Wenbin and Zhang, Muhan and Zhang, Ya},
  journal={IEEE transactions on neural networks and learning systems},
  volume={28},
  number={7},
  pages={1668--1681},
  year={2016},
  publisher={IEEE}
}

@inproceedings{beygelzimer2009importance,
  title={Importance weighted active learning},
  author={Beygelzimer, Alina and Dasgupta, Sanjoy and Langford, John},
  booktitle={Proceedings of the 26th annual international conference on machine learning},
  pages={49--56},
  year={2009}
}

@article{ho2020risk,
  title={Risk guarantees for end-to-end prediction and optimization processes},
  author={Ho-Nguyen, Nam and K{\i}l{\i}n{\c{c}}-Karzan, Fatma},
  journal={Management Science},
  year={2022},
  publisher={INFORMS}
}

@article{elmachtoub2022smart,
  title={Smart “predict, then optimize”},
  author={Elmachtoub, Adam N and Grigas, Paul},
  journal={Management Science},
  volume={68},
  number={1},
  pages={9--26},
  year={2022},
  publisher={INFORMS}
}

@article{bertsimas2020predictive,
  title={From predictive to prescriptive analytics},
  author={Bertsimas, Dimitris and Kallus, Nathan},
  journal={Management Science},
  volume={66},
  number={3},
  pages={1025--1044},
  year={2020},
  publisher={INFORMS}
}

@article{zhu2022joint,
  title={Joint estimation and robustness optimization},
  author={Zhu, Taozeng and Xie, Jingui and Sim, Melvyn},
  journal={Management Science},
  volume={68},
  number={3},
  pages={1659--1677},
  year={2022},
  publisher={INFORMS}
}

@article{donti2017task,
  title={Task-based end-to-end model learning in stochastic optimization},
  author={Donti, Priya and Amos, Brandon and Kolter, J Zico},
  journal={Advances in neural information processing systems},
  volume={30},
  year={2017}
}

@inproceedings{amos2017optnet,
  title={Optnet: Differentiable optimization as a layer in neural networks},
  author={Amos, Brandon and Kolter, J Zico},
  booktitle={International conference on machine learning},
  pages={136--145},
  year={2017},
  organization={PMLR}
}

@book{wainwright2019high,
  title={High-dimensional statistics: A non-asymptotic viewpoint},
  author={Wainwright, Martin J},
  volume={48},
  year={2019},
  publisher={Cambridge University Press}
}

@article{liu2023active,
  title={Active Learning in the Predict-then-Optimize Framework: A Margin-Based Approach},
  author={Liu, Mo and Grigas, Paul and Liu, Heyuan and Shen, Zuo-Jun Max},
  journal={arXiv preprint arXiv:2305.06584},
  year={2023}
}

@article{sadr2024cardiovascular,
  title={Cardiovascular disease diagnosis: a holistic approach using the integration of machine learning and deep learning models},
  author={Sadr, Hossein and Salari, Arsalan and Ashoobi, Mohammad Taghi and Nazari, Mojdeh},
  journal={European Journal of Medical Research},
  volume={29},
  number={1},
  pages={455},
  year={2024},
  publisher={Springer}
}

@article{zhao2025pigeonhole,
  title={Pigeonhole design: Balancing sequential experiments from an online matching perspective},
  author={Zhao, Jinglong and Zhou, Zijie},
  journal={Management Science},
  volume={71},
  number={3},
  pages={1889--1908},
  year={2025},
  publisher={INFORMS}
}

@article{li2025strong,
  title={Strong Formulations and Algorithms for Regularized A-optimal Design},
  author={Li, Yongchun},
  journal={arXiv preprint arXiv:2505.14957},
  year={2025}
}

@article{huang2024decision,
  title={Decision-focused learning with directional gradients},
  author={Huang, Michael and Gupta, Vishal},
  journal={Advances in Neural Information Processing Systems},
  volume={37},
  pages={79194--79220},
  year={2024}
}

@article{bennouna2025data,
  title={What Data Enables Optimal Decisions? An Exact Characterization for Linear Optimization},
  author={Bennouna, Omar and Bennouna, Amine and Amin, Saurabh and Ozdaglar, Asuman},
  journal={arXiv preprint arXiv:2505.21692},
  year={2025}
}

@article{er2025decision,
  title={Decision-Focused Bias Correction for Fluid Approximation},
  author={Er, Can and Liu, Mo},
  journal={arXiv preprint arXiv:2512.15726},
  year={2025}
}

@article{zhu2025power,
  title={The Power of One Change: Bayesian Sequential Experimentation under Limited Flexibility},
  author={Zhu, Yvonne Huijun and Feng, Yifan and Ni, Tu},
  journal={Available at SSRN},
  year={2025}
}

@article{cristian2025efficient,
  title={Efficient End-to-End Learning for Decision-Making: A Meta-Optimization Approach},
  author={Cristian, Rares and Harsha, Pavithra and Perakis, Georgia and Quanz, Brian},
  journal={arXiv preprint arXiv:2505.11360},
  year={2025}
}

@article{chung2022decision,
  title={Decision-aware learning for optimizing health supply chains},
  author={Chung, Tsai-Hsuan and Rostami, Vahid and Bastani, Hamsa and Bastani, Osbert},
  journal={arXiv preprint arXiv:2211.08507},
  year={2022}
}

@article{elmachtoub2023estimate,
  title={Estimate-then-optimize versus integrated-estimation-optimization: A stochastic dominance perspective},
  author={Elmachtoub, Adam N and Lam, Henry and Zhang, Haofeng and Zhao, Yunfan},
  journal={arXiv preprint arXiv:2304.06833},
  volume={107},
  year={2023}
}

@article{hu2022fast,
  title={Fast rates for contextual linear optimization},
  author={Hu, Yichun and Kallus, Nathan and Mao, Xiaojie},
  journal={Management Science},
  volume={68},
  number={6},
  pages={4236--4245},
  year={2022},
  publisher={INFORMS}
}

@article{anderer2022adaptive,
  title={Adaptive clinical trial designs with surrogates: When should we bother?},
  author={Anderer, Arielle and Bastani, Hamsa and Silberholz, John},
  journal={Management science},
  volume={68},
  number={3},
  pages={1982--2002},
  year={2022},
  publisher={INFORMS}
}

@article{ouyang2022training,
  title={Training language models to follow instructions with human feedback},
  author={Ouyang, Long and Wu, Jeffrey and Jiang, Xu and Almeida, Diogo and Wainwright, Carroll and Mishkin, Pamela and Zhang, Chong and Agarwal, Sandhini and Slama, Katarina and Ray, Alex and others},
  journal={Advances in neural information processing systems},
  volume={35},
  pages={27730--27744},
  year={2022}
}

@article{feng2025contextual,
  title={Contextual data-integrated newsvendor solution with operational data analytics (oda)},
  author={Feng, Qi and Shanthikumar, J George and Wu, Jian},
  journal={Management Science},
  year={2025},
  publisher={INFORMS}
}

@article{feng2022developing,
  title={Developing operations management data analytics},
  author={Feng, Qi and Shanthikumar, J George},
  journal={Production and Operations Management},
  volume={31},
  number={12},
  pages={4544--4557},
  year={2022},
  publisher={SAGE Publications Sage CA: Los Angeles, CA}
}

@article{mandi2024decision,
  title={Decision-focused learning: Foundations, state of the art, benchmark and future opportunities},
  author={Mandi, Jayanta and Kotary, James and Berden, Senne and Mulamba, Maxime and Bucarey, Victor and Guns, Tias and Fioretto, Ferdinando},
  journal={Journal of Artificial Intelligence Research},
  volume={80},
  pages={1623--1701},
  year={2024}
}

@article{sadana2025survey,
  title={A survey of contextual optimization methods for decision-making under uncertainty},
  author={Sadana, Utsav and Chenreddy, Abhilash and Delage, Erick and Forel, Alexandre and Frejinger, Emma and Vidal, Thibaut},
  journal={European Journal of Operational Research},
  volume={320},
  number={2},
  pages={271--289},
  year={2025},
  publisher={Elsevier}
}

@article{box1960some,
  title={Some new three level designs for the study of quantitative variables},
  author={Box, George EP and Behnken, Donald W},
  journal={Technometrics},
  volume={2},
  number={4},
  pages={455--475},
  year={1960},
  publisher={Taylor \& Francis}
}

@article{kiefer1959optimum,
  title={Optimum designs in regression problems},
  author={Kiefer, Jack and Wolfowitz, Jacob},
  journal={The annals of mathematical statistics},
  volume={30},
  number={2},
  pages={271--294},
  year={1959},
  publisher={Institute of Mathematical Statistics}
}

@article{che2024optimization,
  title={Optimization-driven adaptive experimentation},
  author={Che, Ethan and Jiang, Daniel R and Namkoong, Hongseok and Wang, Jimmy},
  journal={arXiv preprint arXiv:2408.04570},
  year={2024}
}

@incollection{zhao2024experimental,
  title={Experimental design for causal inference through an optimization lens},
  author={Zhao, Jinglong},
  booktitle={Tutorials in Operations Research: Smarter Decisions for a Better World},
  pages={146--188},
  year={2024},
  publisher={INFORMS}
}

@article{lan2025bias,
  title={The Bias-Variance Tradeoff in Data-Driven Optimization: A Local Misspecification Perspective},
  author={Lan, Haixiang and Liao, Luofeng and Elmachtoub, Adam N and Kroer, Christian and Lam, Henry and Zhang, Haofeng},
  journal={arXiv preprint arXiv:2510.18215},
  year={2025}
}

@article{elmachtoub2025dissecting,
  title={Dissecting the Impact of Model Misspecification in Data-driven Optimization},
  author={Elmachtoub, Adam N and Lam, Henry and Lan, Haixiang and Zhang, Haofeng},
  journal={arXiv preprint arXiv:2503.00626},
  year={2025}
}

\newpage
\begin{APPENDICES}

\AppendixStretch
\section{Benchmark Algorithms}\label{append:benchmarkalg}

\begin{algorithm}[ht]
	\caption{Benchmark: General Decision-blind Sequential Design}\label{alg:blind}
	\begin{algorithmic}[1]
    \STATE \textbf{Input:}   
		Initial slackness $r_0$

		\STATE Set $W_0 \gets \emptyset$, $n_0 \gets 0$, $H_0 \gets \mathcal{\tilde H}$.

		\FOR{$t =$  1,2,$\cdots$, $T$}

		\STATE For each potential design $\mathrm{X}_{(i)} \in \mathcal{X}$, calculate $p_{t,i} \gets  \max_{h_1, h_2 \in H_t}\left\{\|h_1(\mathrm{X}_{(i)}) - h_2(\mathrm{X}_{(i)})\|\right\}$, $\pi_{t,i} = \frac{p_{t,i}}{\sum_{j=1}^m p_{t,j}}$ 
        
        \IF {$\sum_{j=1}^m p_{t,j}>0$}
        \STATE Sample $x_t$ according to the probability $\pi_t = (\pi_{t,1},\dots,\pi_{t,m})$. Let $Q_{t,i}\in\{0,1\}$ denotes the indicator for the realization of $x_t$
        \ELSE 
        \STATE stop the algorithm and return $h_T$
        \ENDIF

		\STATE Conduct experiment and obtain a label $c_t$ associated with $x_t$ 
		\STATE Update $W_t \gets W_{t-1} \cup \{(x_t, c_t, \pi_t,Q_t)\}$, $n_t \gets n_{t-1} + 1$
		
		\STATE Let empirical reweighted loss $\hat{\ell}_{}^t(h) \gets \frac{1}{t} \sum_{(x_t, c_t, \pi_t,Q_t) \in W_t} \sum_{i=1}^m \frac{\mu(x_i)}{\pi_{t,i}}Q_{t,i} \|h(x_i)- c_i\|_2^2$
		
		\STATE Update $h_t \gets \arg\min_{h \in H_{t-1}} \hat{\ell}_{}^t (h)$ and $\hat{\ell}_{}^{t, *} \gets \min_{h \in H_{t-1}} \hat{\ell}_{}^t (h)$

		\STATE Update the confidence set of the predictor $H_t$ by $H_t \gets \{h\in H_{t-1}: \hat{\ell}_{}^{t}(h) \le \hat{\ell}_{}^{t,*} + r_{t} \}$

        \STATE Update $r_{t+1} \gets r_t \sqrt{\frac{t}{t+1} \frac{\log(2(t+1))}{\log (2t)}}$

		\ENDFOR
		 
		\STATE \textbf{Return} $h_T$
	\end{algorithmic}
\end{algorithm}

\section{Application Example: Cost-Sensitive Multiclass classification}\label{sec:cost_sensitive}

One application of decision-focused learning is in the cost-sensitive multiclass classification. Let $X$ denote the feature input and $Y$ denote the corresponding label. 
Suppose there are $m$ classes and $\mathcal{C}$ is a cost matrix with non-negative entries. Here $\mathcal{C}_{ij}$ denotes the misclassification cost for assigning a sample to a class $j$ when the true class is $i$.  Assuming that a correct classification incurs no cost, then the diagonal of $\mathcal{C}$ is zero. Under this setup, for a given input $x$, the classification goal is to minimize the expected misclassification cost, i.e., $\min_{j\in\{1,\cdots,m\}} \sum_{i=1}^m  \mathbb{P}(Y=i|X=x)\mathcal{C}_{i,j}$. Let $c=[c_1,...,c_m]^\top$ be the binary unit vector encoding the true class label, i.e., $c_i = 1$ if the true class of a sample is $i$ and $c_k = 0$ for all $k \neq i$. Then the expected cost of predicting the sample as class $j$ can be expressed as $\sum_{i=1}^m  \mathbb{P}(Y = i|X = x)\mathcal{C}_{i,j} =\sum_{i=1}^m \bbE [c_i|x] \mathcal{C}_{i,j}$. If we further use a probability vector $w\in \bbR^m$ to denote the decision vector, the resulting contextual stochastic linear programming formulation of the classification problem is:

\noindent  
\begin{minipage}{0.4\textwidth}
\centering
\begin{align*}
\min_{w \in \mathbb{R}^m} \quad & \bbE [(c^\top \mathcal{C}) w | x] \\
\text{s.t.} \quad & \mathbf{1}^\top w = 1, \\
& w \geq 0,
\end{align*}
\end{minipage} 
which is equivalent to 
\begin{minipage}{0.3\textwidth}
\begin{align*}
\min_{w,y \in \mathbb{R}^m} \quad & \bbE [c^\top y | x] \\
\text{s.t.} \quad & \mathbf{1}^\top w = 1, \\
&y = Cw, \\
& w \geq 0.
\end{align*}
\end{minipage}

As an illustrative example, consider a clinical testing scenario such as cardiovascular diagnosis, where there are $m=2$ classes: class 1 corresponds to negative and class 2 corresponds to positive.  Suppose the cost for a false positive is 1 and a false negative is 2, then the cost matrix is $\mathcal{C} = \begin{bmatrix}
0 & 1 \\
2 & 0
\end{bmatrix}$.

\section{Experiment Details}

\subsection{Cost-Sensitive Classification Experiment Details}
\label{appendix:cost_sensitive}

This section details the experimental setup and additional findings for cost-sensitive classification. All experiments were executed on a MacBook Pro (Apple M1 Pro, 16 GB RAM).

In the experiment, we first removed outliers from the dataset. Specifically, we excluded samples with systolic blood pressure greater than 250 or less than 70, and those with diastolic blood pressure greater than 200 or less than 40.

We conducted $100$ independent trials of the experiment. The hyperparameter $r_0$ for $\lspo$ and $\ell_2$ minimizer are set to be 5 and 100 respectively. In each trial, $1\%$ of the total dataset was randomly designated as the testing set. We then reported the average SPO risk on the training set, aggregated across all trials, along with its $95\%$ confidence interval.
A subset of our experimental results is presented in Figure \ref{fig:figures_of_cardio}. 

As illustrated in Figure \ref{fig:figures_of_cardio}, when the cost ratio is 1, the objective reduces to a pure prediction task; consequently, our method yields performance comparable to benchmark approaches. However, as the cost ratio increases to 3, decision-focused methods (SPO training loss) begin to outperform decision-blind algorithms (L2 training loss) due to their alignment with the decision-making objective. Furthermore, when both methods utilize SPO training loss, our directional uncertainty algorithm demonstrates a clear advantage over the L2 uncertainty algorithm. At higher cost ratios (5, 7, 9, and 11), these trends intensify: the superiority of decision-focused over decision-blind algorithms widens, and the benefits of incorporating directional uncertainty become increasingly pronounced.

However, as this cost ratio increases, the performance advantage of our method becomes significant, which highlights its alignment with the decision-focused objective. 

In addition to Figure \ref{fig:figures_of_cardio}, table \ref{tab:loss_ratio_summary} summarizes the experiment result by using the percentage SPO risk metric. The percentage SPO risk for a hypothesis $h\in \mathcal{H}$ is defined as $\frac{\sum_{i=1}^{N_{test}} \lspo(h(x_i),c_i) -\sum_{i=1}^{N_{test}} \lspo(h^*(x_i),c_i) }{\sum_{i=1}^{N_{test}} \lspo(h^*(x_i),c_i)}$, which is the average percentage increase in SPO risk relative to the best-performing hypothesis. A low percentage SPO risk indicates that the hypothesis is performing well within the current candidate hypothesis class, whereas a high percentage implies the hypothesis is significantly sub-optimal. The results in the table confirm that the relative effectiveness of our algorithm improves as the cost ratio increases, thereby strengthening our previous conclusion.

\begin{figure}[H]
    \centering

    \begin{subfigure}{0.45\textwidth}
        \centering
         
        \includegraphics[width=0.8\linewidth]{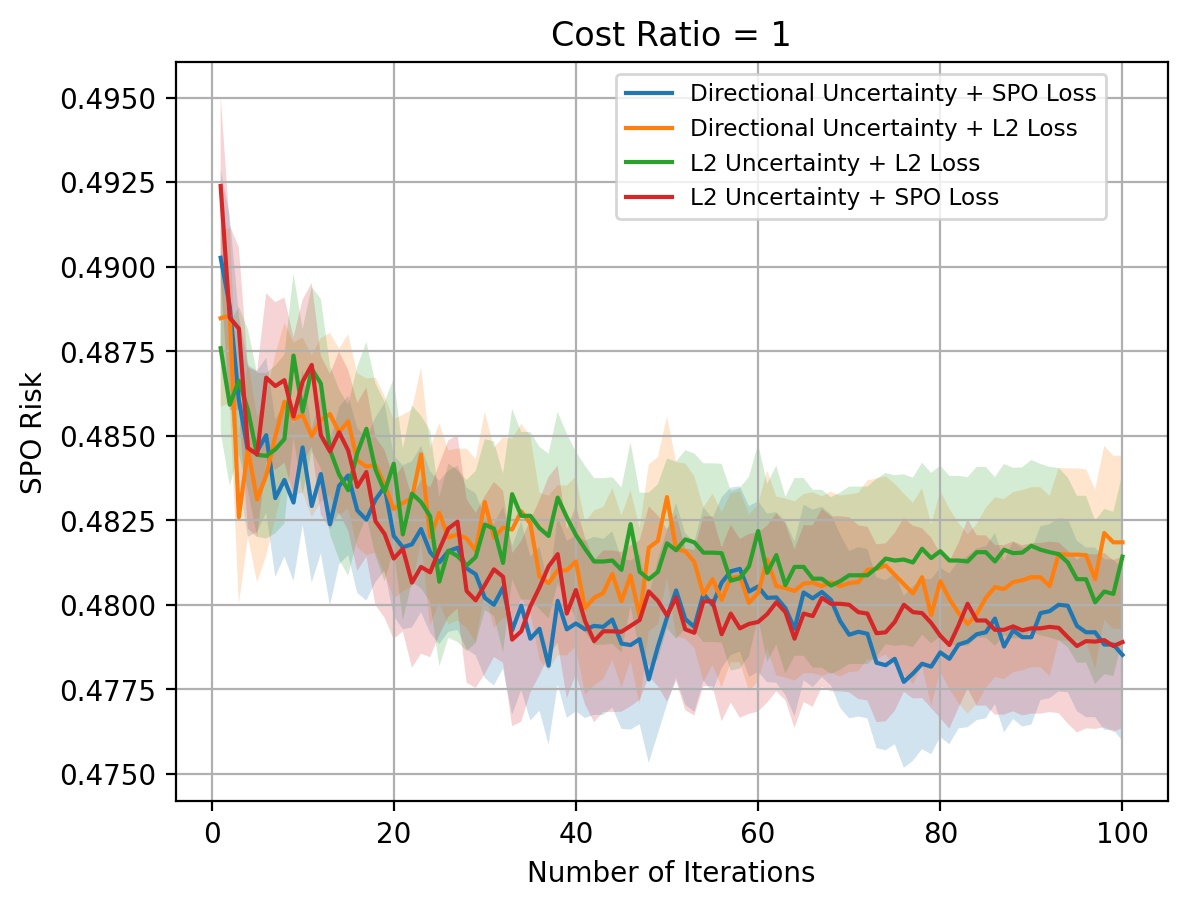}
        \label{fig:sub1}
    \end{subfigure}
     
    \hspace{0.01\textwidth} 
    \begin{subfigure}{0.45\textwidth}
        \centering
        \includegraphics[width=0.8\linewidth]{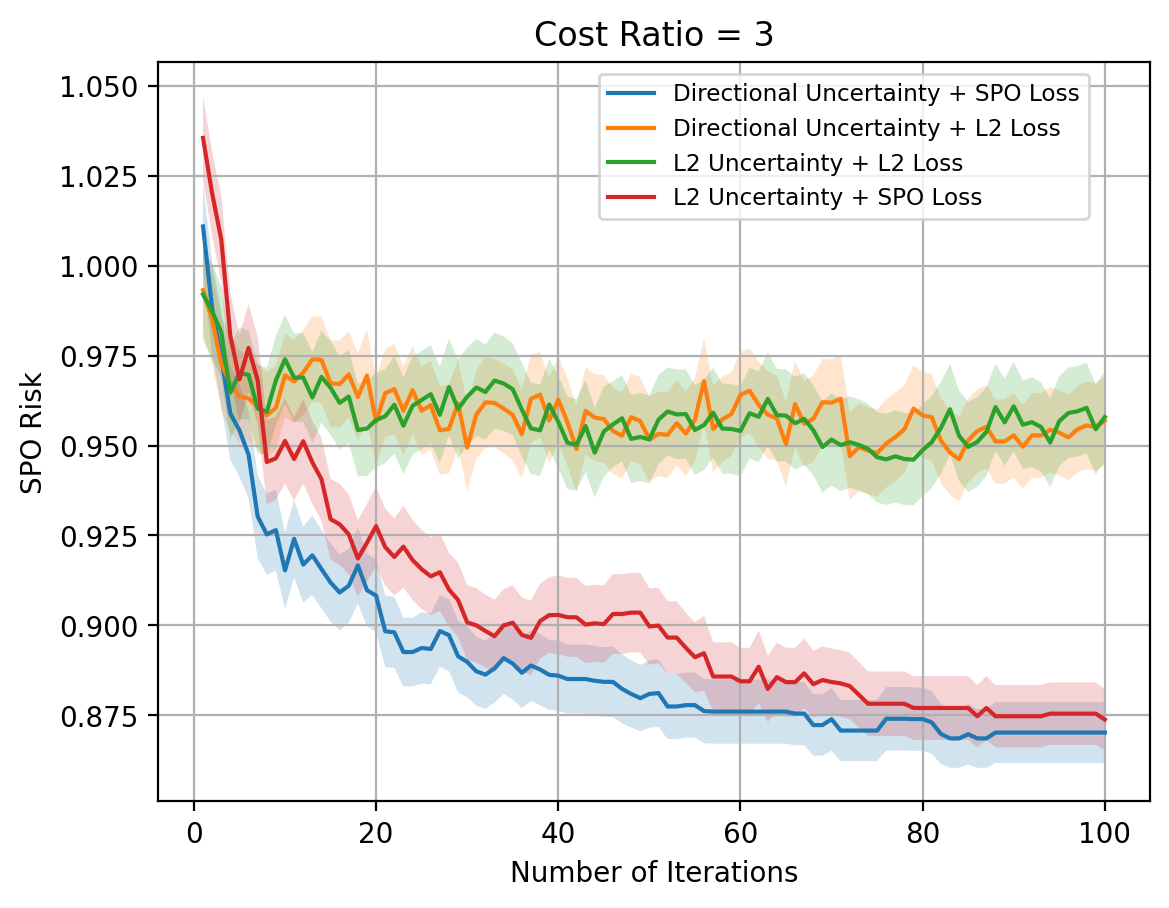}
        \label{fig:sub2}
    \end{subfigure}

    \begin{subfigure}{0.45\textwidth}
        \centering
        \includegraphics[width=0.8\linewidth]{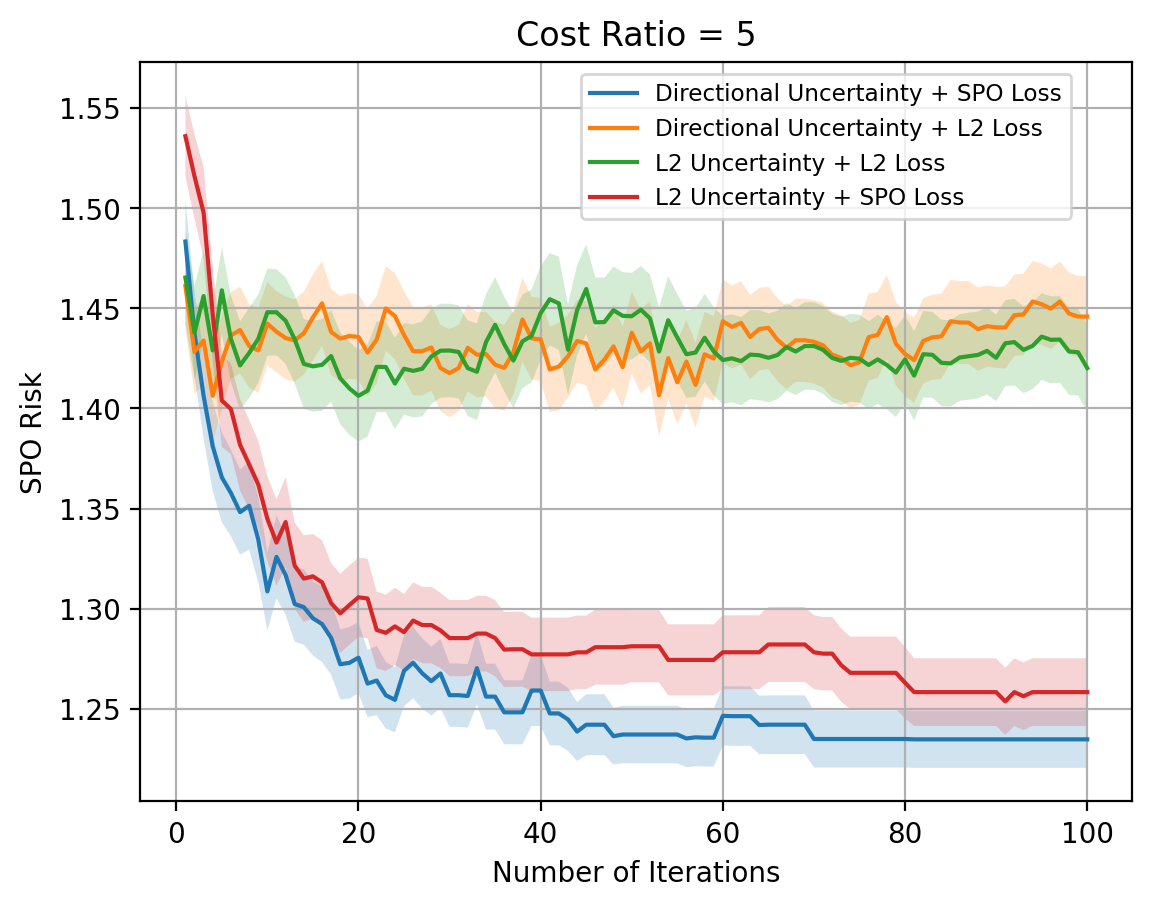}
        \label{fig:sub3}
    \end{subfigure}
    \hspace{0.01\textwidth}
    \begin{subfigure}{0.45\textwidth}
        \centering
        \includegraphics[width=0.8\linewidth]{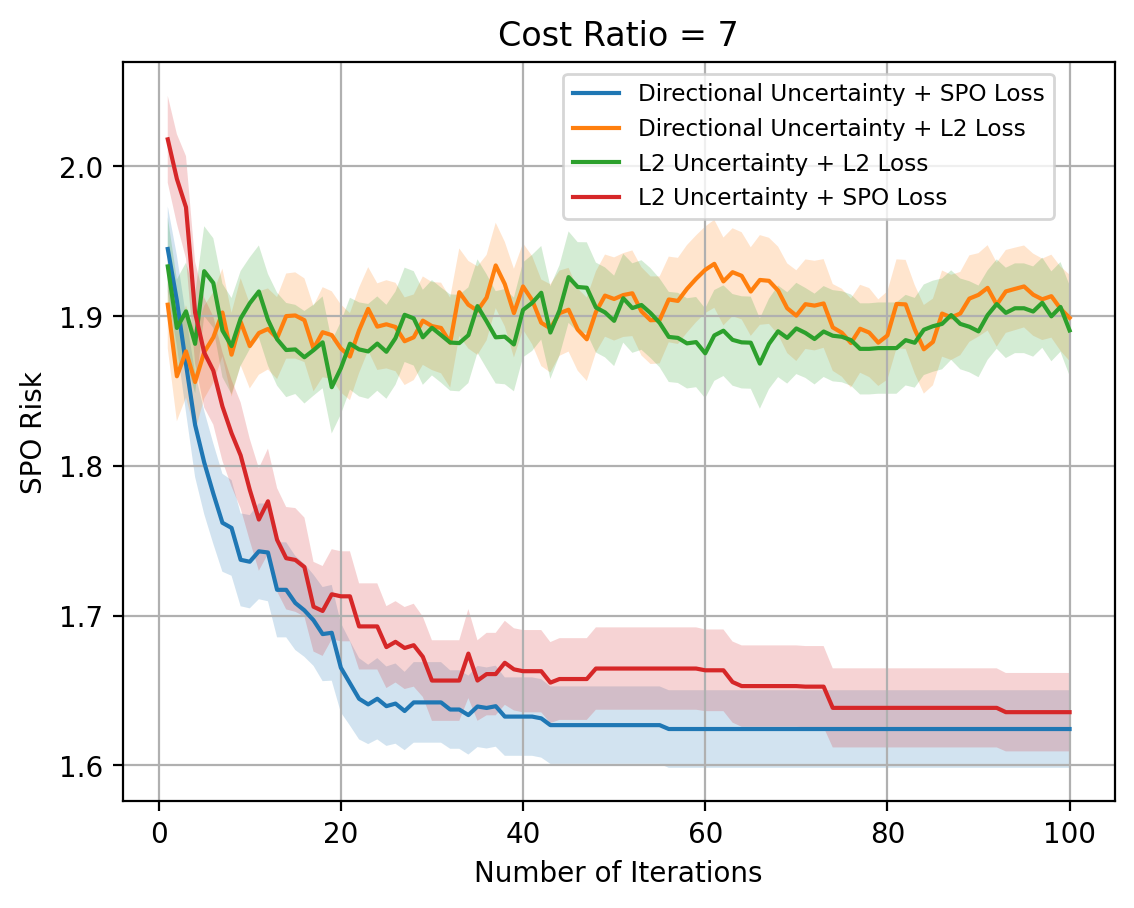}
        \label{fig:sub4}
    \end{subfigure}

    \begin{subfigure}{0.45\textwidth}
        \centering
        \includegraphics[width=0.8\linewidth]{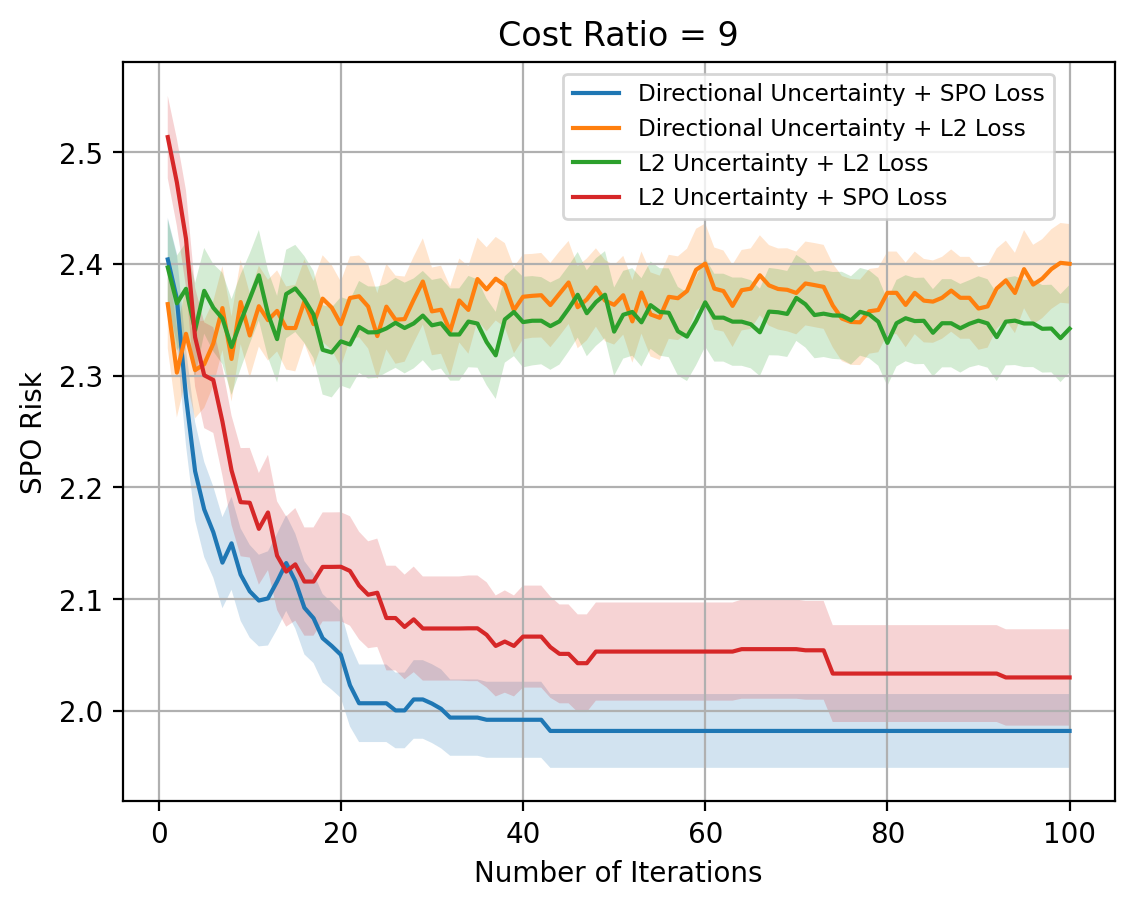}
        \label{fig:sub5}
    \end{subfigure}
    \hspace{0.01\textwidth}
    \begin{subfigure}{0.45\textwidth}
        \centering
        \includegraphics[width=0.8\linewidth]{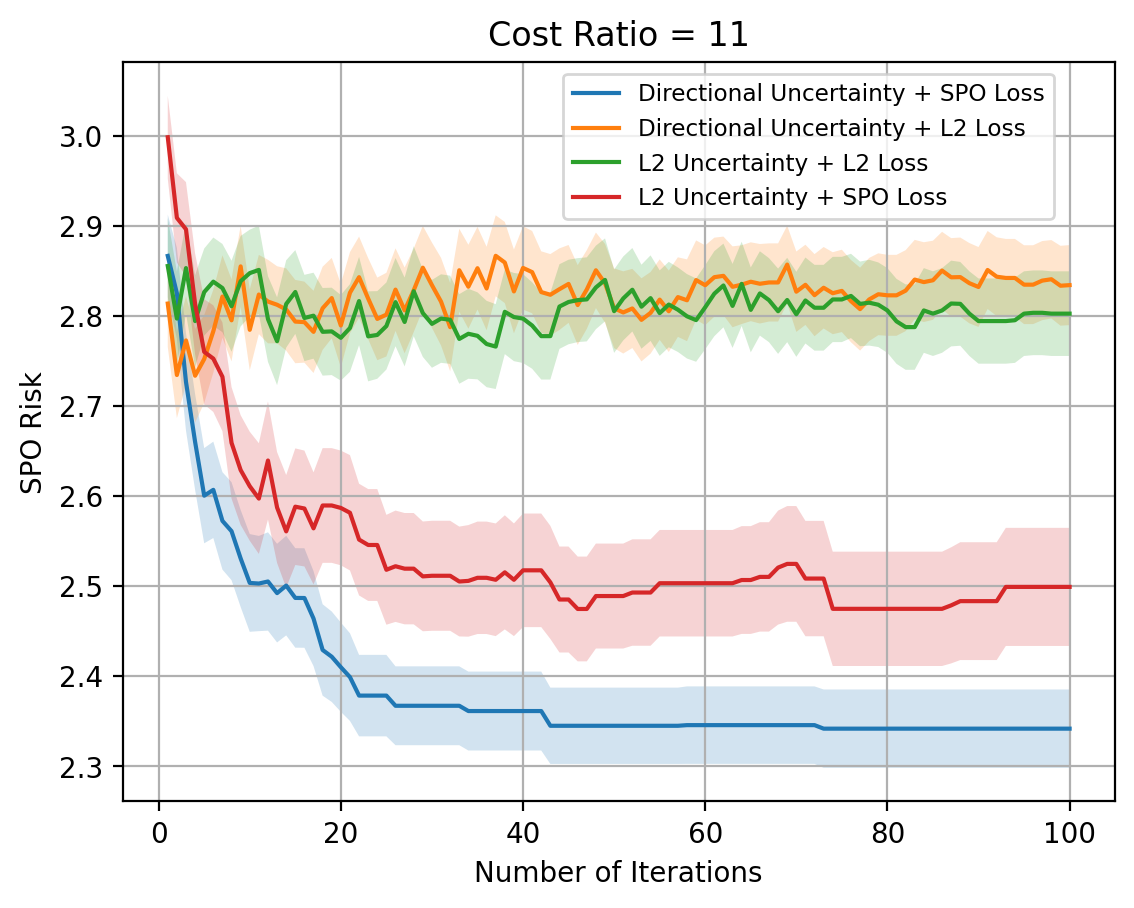}
        \label{fig:sub6}
    \end{subfigure}

    \vspace{-2mm}
    
    \caption{Average SPO Risk Across Different Cost Ratios}
    \label{fig:figures_of_cardio}
\end{figure}

\begin{table}[htbp]
\centering
\caption{Percentage SPO Risk (\%) by Different Method.}
\label{tab:loss_ratio_summary}
{\small  
\begin{tabular}{lcccc}
\toprule
{} & Directional Uncertainty & Directional Uncertainty  & $\ell_2$ Uncertainty  & $\ell_2$ Uncertainty \\
{} &     + $\lspo$ criterion           &     + $\ell_2$     criterion            &     + $\ell_2$    criterion                 &         + $\lspo$ criterion              \\
Cost Ratio &                &                      &                          &                     \\
\midrule
1          &          2.47\% &               3.20\% &                   3.13\% &              2.56\% \\
2          &          0.08\% &               4.98\% &                   5.04\% &              0.75\% \\
3          &          0.29\% &              10.27\% &                  10.45\% &              0.26\% \\
4          &          0.00\% &              15.08\% &                  14.13\% &              0.62\% \\
5          &          0.00\% &              18.15\% &                  16.07\% &              2.43\% \\
6          &          0.00\% &              18.62\% &                  17.84\% &              2.08\% \\
7          &          0.00\% &              18.56\% &                  18.18\% &              1.98\% \\
8          &          0.00\% &              20.26\% &                  20.29\% &              1.76\% \\
9          &          0.00\% &              22.65\% &                  19.65\% &              2.92\% \\
10         &          0.00\% &              22.51\% &                  21.74\% &              3.00\% \\
11         &          0.00\% &              21.47\% &                  20.21\% &              3.63\% \\
12         &          0.00\% &              22.83\% &                  21.37\% &              2.82\% \\
13         &          0.00\% &              22.86\% &                  21.49\% &              3.68\% \\
14         &          0.00\% &              21.81\% &                  20.25\% &              2.42\% \\
15         &          0.00\% &              22.02\% &                  22.49\% &              2.70\% \\
\bottomrule
\end{tabular}
}  
\end{table}
\subsection{LLM Allocation Problem Experiment Details}
\label{appendix:llm}

This section details the experimental setup and additional findings for LLM allocation. All experiments were executed on a MacBook Pro (Apple M1 Pro, 16 GB RAM).

We first present the precise optimization formulation for the experiment. Let $w_{i,j}$ denote the binary decision variable for assigning task $i \in \{1,\dots,5\}$ to model $j \in \{1,\dots,4\}$. Furthermore, let $c_{i,j} \in \{-5,\dots,-1\}$ represent the negative performance score of the output generated for task $i$ by model $j$; consequently, minimizing the objective function in Equation \ref{obj:llm} is equivalent to maximizing the total performance scores. Regarding constraints, Equation \ref{con:gpt} defines the capacity limit for GPT-5, while Equation \ref{con:other} governs the capacity of the remaining three models. Equation \ref{con:demand} ensures that each task is executed exactly once, and Equation \ref{con:binary} imposes the binary constraint on the decision variables.

\begin{align}
    \min_{w \in  \mathbb{R}^{5\times 4}}&: ~~ \bbE[\sum_{i=1}^5\sum_{j=1}^4 c_{i,j}w_{i,j}|x]\label{obj:llm}\\
\text{subject to} & ~~ \sum_{i=1}^5 w_{i,1} \leq 1, \subeqn \label{con:gpt}\\
& ~~ \sum_{i=1}^5 w_{i,j} \leq 2,\forall j=2,3,4, \subeqn\label{con:other}\\
&~~ \sum_{j=1}^4w_{i,j} = 1 ,\forall i=1...,5,\subeqn \label{con:demand}\\
& w_{i,j} \in \{0,1\} ,\quad \forall i =1,...,5, j =1,...,4.  \subeqn \label{con:binary}
\end{align}

\subsubsection{Experiment Setup and Data Collection}

In this experiment, we ran 100 independent trials, the hyperparameter $r_0$ for $\lspo$ and $\ell_2$ minimizer are set to be 5 and 20 respectively.
We used human evaluation for labeling the generated outputs. Regarding the labeling process, Table \ref{tab:score} summarizes the criteria used to evaluate the generated responses, and Table \ref{tab:avg_scores} shows the summarized statistics for the labeled output. In this collected data, we observed that the ChatGPT-5 is the most advanced model that performs well accross all tasks, while Deepseek-llm, Gemma3, Qwen1.5 have good performances only on some tasks. Among these tasks, summarization was the easiest, with a mean score of 4.55, while the other four tasks proved more challenging. The overall average score for the collected data is 3.66.

\begin{table}[htbp]
\begin{small}
\centering
\caption{Overall Quality Score Criterion (1-5 Scale)}
\label{tab:score}
 
\begin{tabular}{| p{2cm} | p{12cm} |} 
\hline
\textbf{Score} & \textbf{Criterion (Overall Quality)}\\
\hline \hline

\textbf{Score=5} (Excellent) &
The text is fluent, accurate, and fully addresses the prompt. It can not be greatly improved. \\
\hline

\textbf{Score=4} (Good) &
The text is good but has minor issues that do not impact its overall meaning or usefulness. These issues include one sentence that repeats content from the original text, small translation mistakes, and reasoning for the sentiment that is not fully accurate. \\
\hline

\textbf{Score=3} (Acceptable) &
The text addresses the prompt but contains noticeable errors, such as content repetition in a few sentences and noticeable translation mistakes.\\
\hline

\textbf{Score=2} (Poor) &
The text has significant problems, such as short garbled sections, an unfinished response, incorrect translations, or inaccurate sentiment analysis.\\
\hline

\textbf{Score=1} (Unusable) &
The text is completely inaccurate. It could include nonsense sentences that are not related to the text, copy the original text when summarizing or making bullet points, or contain large paragraphs of garbled text. \\
\hline
\end{tabular}
\end{small}
\end{table}

Finally, an example of text generation with its evaluated score is provided below. Table \ref{tab:llm_output_examples} shows the prompt that was used, what the generated text looks like, and what its score is.

\begin{table}[]
\caption{Average Scores for Tasks and Models}
\label{tab:avg_scores}
\begin{tabular}{|l|l|l|l|l|l|l|}
\hline
Model/Task   & Summary & Extension & Translation & Bullet Points & Sentiment Analysis & All Tasks \\ \hline
ChatGPT-5     & 5       & 4.96      & 5           & 4.93          & 4.66               & 4.91      \\ \hline
Deepseek-llm & 4.83    & 2.25      & 2.16        & 2.5           & 2.66               & 2.88      \\ \hline
Gemma3       & 3.94    & 2.73      & 2.2         & 2.26          & 2.1                & 2.64      \\ \hline
Qwen1.5      & 4.44    & 3.72      & 3.55        & 4.44          & 4.83               & 4.2       \\ \hline
All Models   & 4.55    & 3.42      & 3.23        & 3.53          & 3.56               & 3.66      \\ \hline
\end{tabular}
\label{table:average_scores}
\end{table}

\noindent\textbf{Source Text:} (from movie ratings category)
\begin{quote}
\textit{Christopher Nolan’s film *Oppenheimer* received widespread critical acclaim and exceptional audience interest. According to its Wikipedia reception summary, critics placed it near the top of year‑end lists and praised its direction, performances, and cinematography. The film became the highest‑grossing biographical feature ever and one of the top IMAX releases, reflecting sustained demand across weeks in theaters. Aggregators reported high scores, and notable filmmakers publicly lauded the work, calling it a defining achievement of the decade. The box‑office trajectory was unusual: even without immediately topping the domestic chart, it maintained momentum and ultimately crossed \$900 million worldwide, with IMAX formats delivering a significant share. Taken together, the critical consensus and commercial performance position *Oppenheimer* as both a prestige title and a mainstream event, bridging awards attention with durable audience appeal.}
\end{quote}

\begin{table}[htbp]
\begin{small}

\linespread{0.6}\selectfont

\renewcommand{\arraystretch}{1.2}

\setlength{\tabcolsep}{4pt}

\centering
\caption{Qualitative examples of generated text.}
\label{tab:llm_output_examples}

\begin{tabular}{| p{2.5cm} | p{11cm} |p
{2cm}|}
\hline
\textbf{Task (Prompt)} & \textbf{Generated Output}&
\textbf{Score}\\
\hline \hline

\textbf{Summarization} \newline
\textit{"Summarize the following text in one or two sentences..."} &
Christopher Nolan's film Oppenheimer received exceptional critical acclaim, high audience interest, and impressive commercial success, positioning it as both a prestige title and a mainstream event, with awards attention and durable audience appeal.&
5\\
\hline

\textbf{Extensions} \newline
\textit{"Continue writing the following text, adding one more paragraph..."} &
In conclusion, Christopher Nolan's film Oppenheimer has truly been a game-changer. It has received exceptional critical acclaim, audience interest, and commercial success, making it a defining achievement of the decade. The film's ability to bridge the gap between awards attention and durable audience appeal has solidified its place as both a prestige title and a mainstream event. Its unique box-office trajectory, with sustained demand across weeks in theaters and high-performing IMAX formats, has resulted in an impressive worldwide gross of over \$900 million. This unprecedented success has not only solidified Nolan's position as a visionary filmmaker but has also solidified Oppenheimer's place as a cinematic masterpiece that will be remembered for years to come.&
3
\\
\hline

\textbf{Bulletpoints} \newline
\textit{"Extract the key points from the following text as a bulleted list..."} &
* Oppenheimer received widespread critical acclaim and exceptional audience interest
* The film placed near the top of year-end lists and was praised for its direction, performances, and cinematography
* It became the highest-grossing biographical feature ever and one of the top IMAX releases
* Aggregators reported high scores and notable filmmakers lauded the work as a defining achievement of the decade
* The box-office trajectory was unusual, but the film ultimately grossed over \$900 million worldwide with a significant share from IMAX formats
* The critical consensus and commercial performance position Oppenheimer as a prestige title and a mainstream event, bridging awards attention with durable audience appeal. &
5\\
\hline

\textbf{Sentiment} \newline
\textit{"Analyze the sentiment of the following text... and explain your reasoning..."} &
The sentiment of the text is positive. The reasons for this positive sentiment are evident in the various positive phrases and indicators used throughout the text:

1. "Widespread critical acclaim" - This phrase indicates that the film received positive reviews and high praise from critics.
2. "Exceptional audience interest" - This phrase suggests that the film also attracted a significant and engaged audience.
3. "Highest-grossing biographical feature ever" - This statement highlights the commercial success of the film, which is another positive indicator.
4. "One of the top IMAX releases" - This phrase indicates that the film did well in a specific format, which is another positive aspect.
5. "Defining achievement of the decade" - This statement suggests that the film has had a lasting impact and is considered a significant work in its genre.
6. "Unusual box-office trajectory" - This phrase indicates that the film did not immediately top the domestic chart but still managed to maintain momentum and achieve significant global success.
7. "Prestige title and mainstream event" - This statement highlights the film's ability to appeal to both critics and audiences, which is a strong positive sentiment.
&
5\\
\hline

\end{tabular}
\end{small}
\end{table}

\subsubsection{Robustness Check Under Extra Constraint}

To demonstrate our method's robustness across different scenarios, we introduced an additional constraint to the baseline problem. Specifically, we restricted the joint usage of Qwen1.5 and Deepseek-llm to a maximum of two instances. This constraint was designed to enforce a more even distribution of model utilization. Maintaining all other settings constant, we introduced this additional constraint and compared our method’s performance against the other benchmark algorithms. 

During this experiment, the hyperparameter $r_0$ for $\lspo$ and $\ell_2$ minimizer are set to be 2.5 and 20, respectively. The results are presented in Figure \ref{fig:llm_extra}.

\begin{figure}[H]
    \centering
    \includegraphics[width=0.6\linewidth]{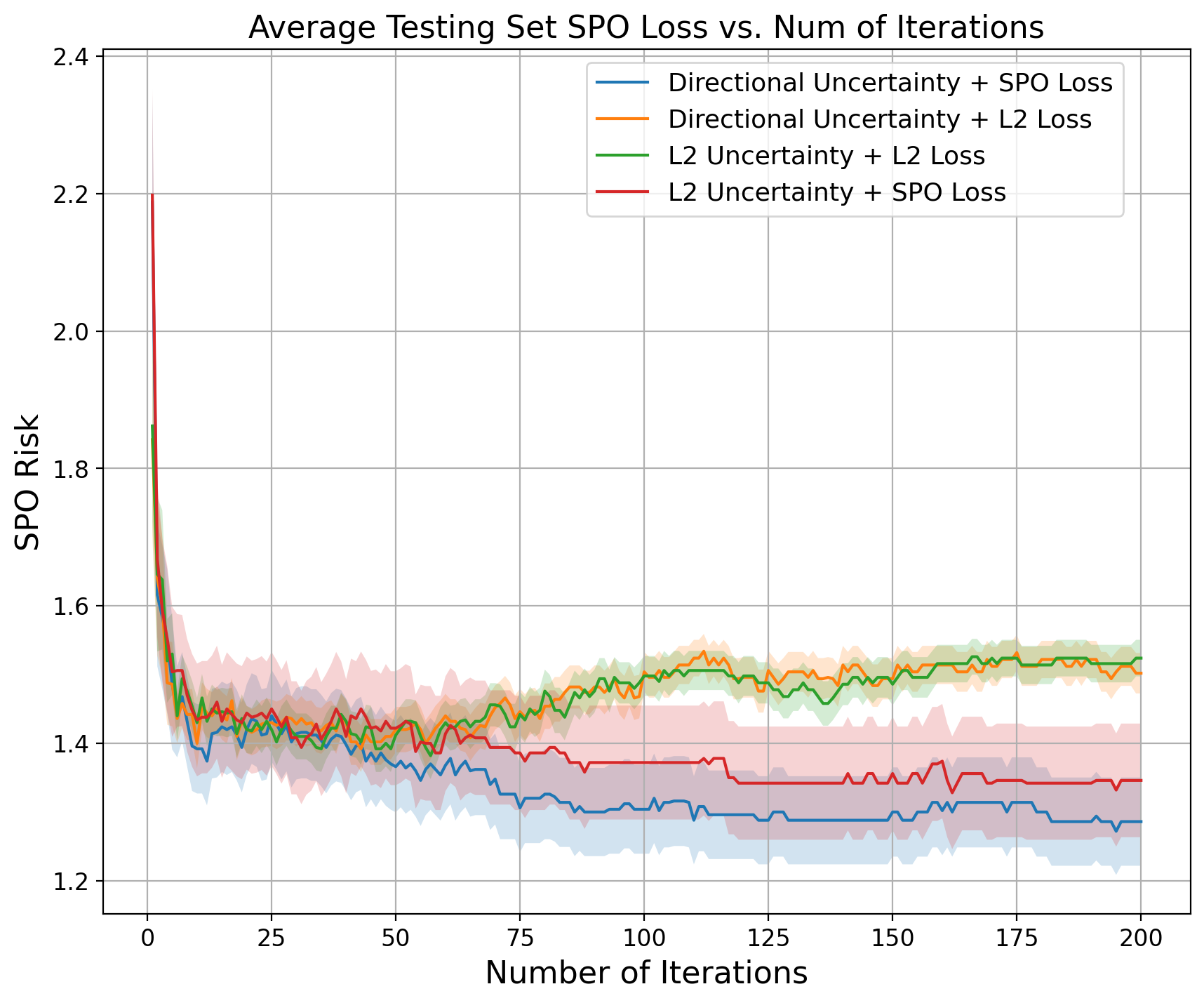}
    \caption{SPO Risk Under Extra Constraint}
    \label{fig:llm_extra}
\end{figure}

Figure \ref{fig:llm_extra} demonstrates that, under the additional constraint, the testing set SPO risk curves are similar with the original problem. Our method continues to outperform other benchmark algorithms. These results indicate that our method remains consistent under additional constraints, demonstrating its robustness to changes in the downstream decision-making problem.

\subsubsection{Empirical evidence of the Degeneracy Angle and Optimality Gap}\label{append:distribution_of_gap}

To demonstrate the presence of the degeneracy angle and the optimality gap, we constructed various hypothesis sets to analyze their empirical distributions. Specifically, we generated 100 independent random hypothesis sets, denoted as $\mathcal{H}_1, \dots, \mathcal{H}_{100}$. Each set $\mathcal{H}_k$ consists of five random hypotheses $h_{1}^k, \dots, h_{5}^k$. We constructed each individual hypothesis by drawing $n$ samples with replacement from each category—where $n$ is drawn from $\text{Uniform}[5, 10]$—and calculating the categorical sample mean. Figure \ref{fig:angle} shows the histogram of the experiment result. 
\begin{figure}[H]
    \centering
    \includegraphics[width=0.5\linewidth]{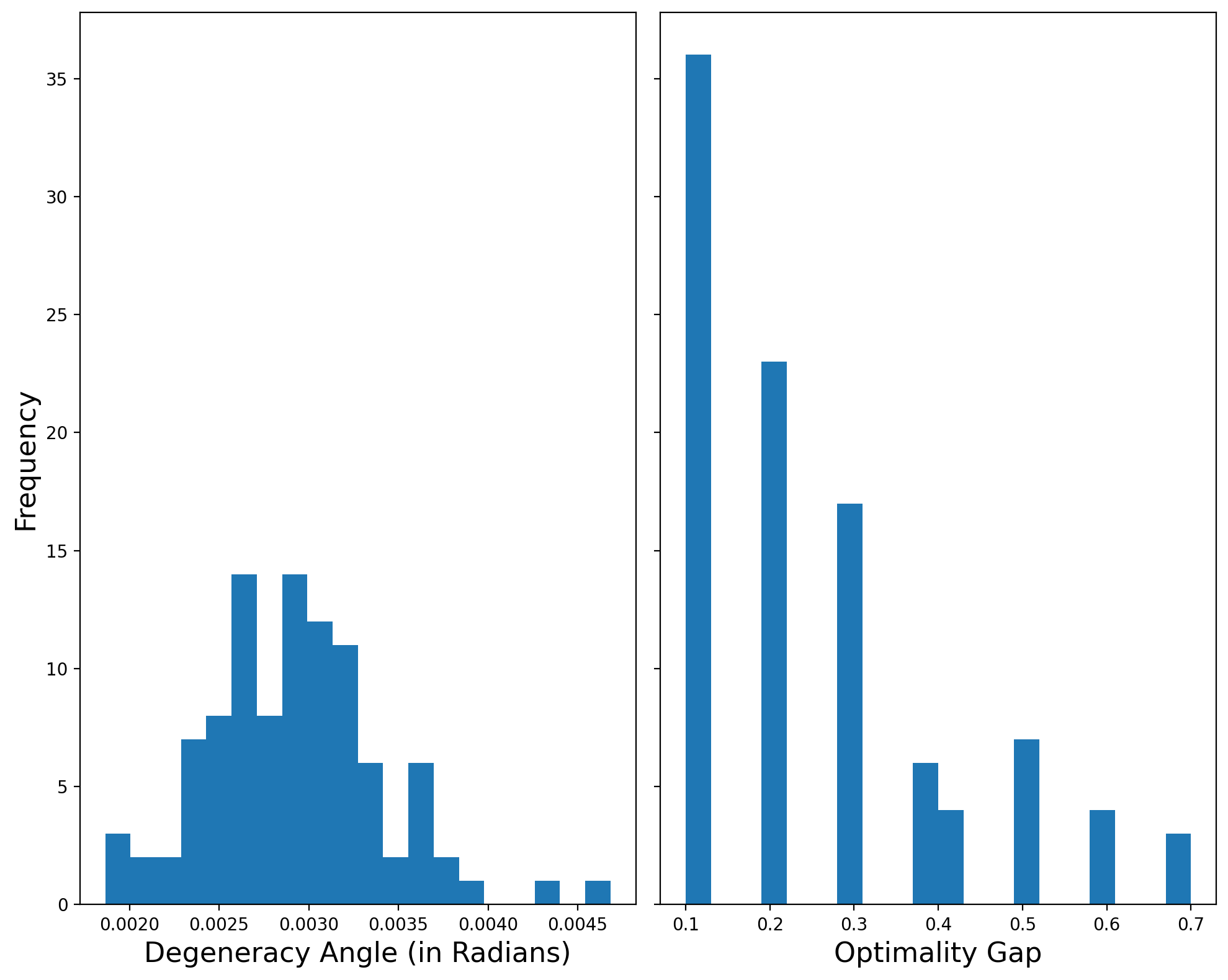}
    \caption{A Illustration of the Degeneracy Angle and Optimality Gap}
    \label{fig:angle}
\end{figure}

Figure \ref{fig:angle} illustrates the existence of a strictly positive degeneracy angle and optimality gap. These metrics are pivotal to the convergence analysis and comparative evaluation of the algorithms discussed previously. This empirical evidence substantiates our earlier theoretical assumptions, verifying that the geometric conditions required for convergence are met in practice. Consequently, the observed behavior aligns closely with the theoretical result.

\section{Proofs}

\subsection{Proof of Theorem \ref{thm:prediction}}

In this appendix, we provide the proof of Theorem \ref{thm:prediction}. The high-level idea consists of two steps. First, we show that when the slackness $r_t$ used to shrink the nested hypothesis class is not too small, the true optimal predictor is safely contained in $\mathcal{H}_t$ (Lemma~\ref{lemma:thm1}). Second, we show that the specific choice of $r_t$ in Theorem~\ref{thm:prediction} satisfies the condition in Lemma~\ref{lemma:thm1}, which completes the proof.

Recall that for all $t\ge 1$,
\begin{align*}
	\bbE[\ell^{\mathtt{rew}}(h; z_t)|\mathcal{F}_{t-1}]=\bbE[\lspo(h(x_t),c_t)] = \rspo(h).
\end{align*}
Consider the above applied to both $h \in \mathcal{H}$ and $h^\ast$ and averaged over $i \in \{1, \ldots, t\}$ to yield:
\begin{equation}\label{equ:surrogate_weighted_averaged}
	\rspo(h) - \rspo(h^\ast) ~=~ \frac{1}{t}\sum_{i = 1}^t \left(\bbE[\ell^{\mathtt{rew}}(h; z_i)|\mathcal{F}_{i-1}] - \bbE[\ell^{\mathtt{rew}}(h^\ast; z_i)|\mathcal{F}_{i-1}]\right).
\end{equation}

For any given $h \in H_{T-1}$, we denote the discrepancy between the conditional expectation and the realized excess re-weighted loss of predictor $h$ at time $t$ by $Z^{\mathtt{t}}_h$, i.e., $Z^{\mathtt{t}}_h := \bbE[\ell^{\mathtt{rew}}(h; z_t) - \ell^{\mathtt{rew}}(h^*; z_t)| \mathcal{F}_{t-1}] -  (\ell^{\mathtt{rew}}(h; z_t) - \ell^{\mathtt{rew}}(h^*; z_t))$. In the subsequent proof, we will show that $Z_h^t$ is a bounded random variable. Since $Z^{\mathtt{t}}_{h_T}$ is bounded and  $\bbE[Z^{\mathtt{t}}_{h_T}|\mathcal{F}_{t-1}]=0$, we have that $\sum_{t=1}^TZ^{\mathtt{t}}_{h_T}$ is a martingale, and we can apply Azuma's lemma to obtain finite sample risk bound.

Thus, for any given $h \in H_{T-1}$, \eqref{equ:surrogate_weighted_averaged} is equivalently written as:
\begin{equation}\label{equ:surrogate_exc}
	\rspo(h) - \rspo(h^\ast) ~=~ \frac{1}{t} \sum_{i = 1}^t 
	Z^{\mathtt{i}}_h + \frac{1}{t}\sum_{i = 1}^t\left(\ell^{\mathtt{rew}}(h; z_i) - \ell^{\mathtt{rew}}(h^*; z_i)\right).
\end{equation}

Before providing the proof of Theorem \ref{thm:prediction}, we first show that the confidence set $H_{T-1}$ contains the true optimal predictor $h^*$ at each iteration if $r_t$ satisfies some conditions in Lemma \ref{lemma:thm1}.
\begin{lemma}\label{lemma:thm1} Given $T \ge 1$, if $r_t$ satisfies that $\sup_{h \in H_{t-1}}\left|\frac{1}{t} \sum_{i = 1}^t Z^{\mathtt{i}}_{h}\right| \le r_t$, for any $t \le T-1$, then we have $h^* \in H_{T-1}$.
	 
\end{lemma}

\begin{proof}{\bfseries Proof of Lemma \ref{lemma:thm1}} Since $H_{T-1} \subseteq H_{T-2} \subseteq ... \subseteq H_0$, we prove Lemma \ref{lemma:thm1} by induction. Obviously, we have that $h^* \in H_0 = \mathcal{H}$. 
	Assume that $h^* \in H_{t}$ for all $t \le T-2$. Next, we will show that $h^* \in H_{T-1}$.

	Since $H_{T-1} = \{h\in H_{T-2}: \hat{\ell}_{}^{T-1}(h) \le \hat{\ell}_{}^{T-1,*} + r_{T-1} \}$, to show  $h^* \in H_{T-1}$, it suffices to show that $\frac{1}{T-1}\sum_{i = 1}^{T-1} \ell^{\mathtt{rew}}(h^*; z_i)\le \frac{1}{T-1}\sum_{i = 1}^{T-1} \ell^{\mathtt{rew}}(h_{T-1}; z_i)+ r_{T-1}$.
	
	Since $\rspo(h_{T-1}) - \rspo(h^\ast)\ge 0$, by \eqref{equ:surrogate_exc}, we have that
	\begin{equation*}
		\rspo(h_{T-1}) - \rspo(h^\ast) =~\frac{1}{T-1} \sum_{i = 1}^{T-1} 
		Z^{\mathtt{i}}_{h_{T-1}} + \frac{1}{T-1}\sum_{i = 1}^{T-1}\left(\ell^{\mathtt{rew}}(h_{T-1}; z_i) - \ell^{\mathtt{rew}}(h^*; z_i)\right) \ge 0.
	\end{equation*}

	Since $h_{T-1} \in H_{T-2}$, by the condition in Lemma \ref{lemma:thm1},  we have $\frac{1}{T-1} \sum_{i = 1}^{T-1} Z^{\mathtt{i}}_{h_{T-1}} \le r_{T-1}$. Therefore, we obtain that 
	\begin{align*}
		\frac{1}{T-1}\sum_{i = 1}^{T-1} \ell^{\mathtt{rew}}(h^*; z_i)\le \frac{1}{T-1}\sum_{i = 1}^{T-1} \ell^{\mathtt{rew}}(h_{T-1}; z_i)+ r_{T-1}. 
	\end{align*}
	Thus, we obtain $h^* \in H_{T-1}$. \hfill \Halmos 
\end{proof}

\begin{proof}{\bfseries Proof of Theorem \ref{thm:prediction}}
	Here we prove a stronger version: with probability $1-\delta$, for all $T\ge 1$, we have
	\[ 
	(A) \text{: For any $h \in H_{T-1}$, we have $R_\ell(h) - R_\ell^\ast \le 2 r_{T} $ and } \sup_{h \in H_{T-1}}\left|\frac{1}{T} \sum_{t = 1}^TZ^{\mathtt{t}}_{h}\right| \le r_T.
	\]
    where $r_T = 2\gamma L \sqrt{\frac{\log(2T|\mathcal{H}|/\delta)}{T}} = \otilde(T^{-1/2})$.

	We prove Claim $(A)$ by the chain rule of probability.
	When $T = 0$, the risk bound in Theorem \ref{thm:prediction} holds by the definition of $r_0  \geq 2 \omega_\ell(\hat{\mathcal{C}},\mathcal{C})$. Next, we assume that Claim $(A)$ holds for all $t \le T-1$ and our goal is to show that Claim {\em (A)} holds for $T$ with probability at least $1-\frac{\delta}{2T^2}$.  In other words, for any $h \in H_{T-1}$, we have $\rspo(h) - \rspo^\ast \le 2 r_{T}$ and $\sup_{h \in H_{T-1}}\left|\frac{1}{T} \sum_{t = 1}^TZ^{\mathtt{t}}_{h}\right| \le r_T$ with probability at least $1-\frac{\delta}{2T^2}$. If it is true, by the chain rule of probability and taking the union bound over $T \in \{1, 2,...,\infty\}$, we will have that Claim {\em (A)} holds for all $T \ge 1$ with probability at least 
	\begin{align*}
		1 - \sum_{T=1}^{\infty}\frac{\delta}{2 T^2} \ge 1 - \frac{\delta \pi^2}{12} \ge 1 - \delta.
	\end{align*}
	Thus, we will obtain Claim $(A)$.

	The rest of the proof of the risk bound in Theorem \ref{thm:prediction} is to show that Claim {\em (A)} holds for $T$ with probability at least $1-\frac{\delta}{2T^2}$.

	Given any $h \in H_{T-1} \subseteq H_{T-2}$, since Claim  {\em (A)} holds for $t \le T-1$, we have that $\rspo(h) - \rspo(h^\ast)\le 2 r_{T-1}$.  
	By the "Lipschitz-like" property of $\lspo$ and the decreasing property of $r_t$, we have that for any $j \le T$,

		\begin{align*}
		|\ell^{\mathtt{rew}}(h; z_j) - \ell^{\mathtt{rew}}(h^*; z_j)|
		& = | \sum_{i=1}^{m} \frac{\mu(x_i)}{w_{j,i}} Q_{j,i}(\lspo(h(x_i),c_i)-\ell{_{SPO}}(h^*(x_i,),c_i))|\\
        & = |\frac{\mu(x_j)}{w_{j,i}}(\ell{_{SPO}}(h(x_j),c_j)-\ell{_{SPO}}(h^*(x_j,),c_j))|\\
        & \leq |\frac{\sum_{k=1}^m p_{j,k}}{p_{j,i}}(\ell{_{SPO}}(h(x_j),c_j)-\ell{_{SPO}}(h^*(x_j,),c_j))|\\
        & \leq \frac{\sum_{k=1}^m p_{j,k}}{p_{j,i}} L\|\frac{h(x_j)}{\|h(x_j)\|}-\frac{h^*(x_j)}{\|h^*(x_j)\|}\| \\
        & \leq L\sum_{k=1}^m p_{t,k} \leq L \gamma.
	\end{align*}

        It implies that the  loss function $\ell^{\mathtt{rew}}(h; z_t) - \ell^{\mathtt{rew}}(h^*; z_t)$ is upper bounded by $L\gamma$, so we can apply Azuma's inequality to the sequence $\sum_{t=1}^TZ^{\mathtt{t}}_{h}$. By taking the average of $Z^{\mathtt{t}}_{h}$, we have that 
	$\left|\frac{1}{T} \sum_{t = 1}^TZ^{\mathtt{t}}_{h}\right| \le \epsilon,$
	with probability at least $1 - 2  e^{-\frac{ \epsilon^2 T}{2 L^2\gamma^2}}$.
	
	By setting the probability $1 - 2  e^{-\frac{ \epsilon^2 T}{2 L^2\gamma^2}} = 1 - \frac{\delta}{2T^2 |\mathcal{H}|^2}$, we obtain that $\epsilon \le 2 \gamma L\sqrt{\frac{ \ln(2T |\mathcal{H}|/\delta)}{T}} \leq r_T$. By applying the union bound over all $h \in H_{T-1} \subseteq \mathcal{H}$, and all $h^* \in \mathcal{H}$, we have that  
	\begin{equation*}
		\sup_{h \in H_{T-1}}\left|\frac{1}{T} \sum_{t = 1}^{T}Z^{\mathtt{t}}_h\right| ~\le~ \epsilon \leq r_{T},
	\end{equation*}
	with probability at least 
	\begin{align*}
		1 -  |\mathcal{H}|^2 \cdot\frac{\delta}{2 T^2|\mathcal{H}|^2} \ge 1- \frac{\delta}{2 T^2}.
	\end{align*}
	Next, conditioning on the occurrence of this event, we will prove that $\rspo(h) - \rspo(h^\ast) \le 2 r_T$. Since Claim $(A)$ holds for all $t \le T-1$, it implies that for any $t \le T-1$, $\sup_{h \in H_{t-1}}\left|\frac{1}{t} \sum_{i = 1}^tZ^{\mathtt{i}}_{h}\right| \le r_t$. Thus, the condition in Lemma \ref{lemma:thm1} holds, and by Lemma \ref{lemma:thm1}, we have that $h^* \in H_{T-1}$.

	Since $h \in H_{T-1}$ and $h^* \in H_{T-1}$, we have that $\ell^{\mathtt{rew}}(h_T; z_t) - \ell^{\mathtt{rew}}(h^*; z_t) \le r_T$ in \eqref{equ:surrogate_exc} and we obtain that 
	$\rspo(h) - \rspo(h^\ast) \le 2 r_T$ if $\left|\frac{1}{T} \sum_{t = 1}^TZ^{\mathtt{t}}_{h}\right| \le r_T$.
	
	Therefore, we obtain that Claim {\em (A)} holds for $T$ with probability at least $1- \frac{\delta}{2 T^2}$.\hfill \Halmos
\end{proof}

\subsection{Other Proofs.}
\begin{proof}{\bfseries Proof of Fact \ref{lem:spo_scale_invariant}}
Fix any $\hat{c}$ and $\alpha>0$. Recall the oracle map
\[
w^*(c)\in\arg\min_{w\in S} w^\top c.
\]
For any $w\in S$, we have
\[
w^\top (\alpha \hat{c})=\alpha\,(w^\top \hat{c}).
\]
Since $\alpha>0$, multiplying the objective by $\alpha$ preserves the set of minimizers, i.e.,
\[
\arg\min_{w\in S} w^\top (\alpha \hat{c})
=
\arg\min_{w\in S} \alpha\,(w^\top \hat{c})
=
\arg\min_{w\in S} w^\top \hat{c}.
\]
Therefore, we can choose the oracle consistently so that $w^*(\alpha\hat{c})=w^*(\hat{c})$.

For the SPO loss, recall
\[
\lspo(\hat{c},c)=c^\top w^*(\hat{c})-c^\top w^*(c).
\]
Using $w^*(\alpha\hat{c})=w^*(\hat{c})$, we obtain
\[
\lspo(\alpha\hat{c},c)
=
c^\top w^*(\alpha\hat{c})-c^\top w^*(c)
=
c^\top w^*(\hat{c})-c^\top w^*(c)
=
\lspo(\hat{c},c),
\]
which proves the claim. \hfill\Halmos
\end{proof}

\begin{proof}{\bfseries Proof of Lemma \ref{lem:SPO_Lipschitz}}
From theorem 8 of \cite{el2022generalization}, we know the following hold for any $\hat{c},w,$ $$\hat{c}^\top(w-w^*(\hat{c})) \geq (\frac{\nu_\mathcal{S}(\hat{c})}{\Delta(\mathcal{S})})\|w-w^*(\hat{c})\|^2.$$

Plug in $\hat{c} = h_1(x), w = w^*(\frac{h_2(x)}{\|h_2(x)\|})$ and $\hat{c} = h_2(x), w = w^*(\frac{h_1(x)}{\|h_1(x)\|})$, we get the following two inequality:
\begin{align*}
    & h_1(x)^\top ( w^*(h_2(x))- w^*( h_1(x)) )\geq (\frac{\nu_\mathcal{S}( h_1(x))}{\Delta(\mathcal{S})})\|w^*(h_2(x))-w^*( h_1(x))\|^2, \\
    &h_2(x)^\top ( w^*(h_1(x))- w^*(h_2(x)) )\geq (\frac{\nu_\mathcal{S}(h_2(x))}{\Delta(\mathcal{S})})\|w^*(h_1(x))-w^*(h_2(x))\|^2.
\end{align*}

Divide $\|h_1(x)\|,\|h_2(x)\| $ on both side and use Assumption \ref{assumption:tau_lower_bound}, we get
\begin{align*}
    & \frac{h_1(x)}{\|h_1(x)\|}^\top ( w^*(h_2(x))- w^*( h_1(x)) )\geq (\frac{ \eta }{\Delta(\mathcal{S})})\|w^*(h_2(x))-w^*( h_1(x))\|^2 ,\\
    &\frac{h_2(x)}{\|h_2(x)\|}^\top ( w^*(h_1(x))- w^*(h_2(x)) )\geq (\frac{\eta}{\Delta(\mathcal{S})})\|w^*(h_1(x))-w^*(h_2(x))\|^2.
\end{align*}

Adding these two inequalities and applying the Cauchy-Schwarz inequality, we get 

\begin{align*}
    \frac{2\eta}{\Delta(\mathcal{S})}\|w^*(h_1(x))-w^*(h_2(x))\|^2  &\leq 
    \left(\frac{h_1(x)}{\|h_1(x)\|} -\frac{h_2(x)}{\|h_2(x)\|} \right)^\top (w^*(h_1(x))-w^*(h_2(x))) \\
    &\leq 
    \left\|\frac{h_1(x)}{\|h_1(x)\|} -\frac{h_2(x)}{\|h_2(x)\|} \right\| \|w^*(h_1(x))-w^*(h_2(x))\|.
\end{align*}

Hence $\|w^*(h_1(x))-w^*(h_2(x))\|$ is upper bounded by $\left\|\frac{h_1(x)}{\|h_1(x)\|} -\frac{h_2(x)}{\|h_2(x)\|} \right\|\frac{\Delta(\mathcal{S})}{2\eta}$. 

Put these results together, we get 
\begin{align*}
    &|\lspo(h_1(x),c)-\lspo(h_2(x),c)|\\
    &=|\lspo(\frac{h_1(x)}{\|h_1(x)\|},c)-\lspo(\frac{h_2(x)}{\|h_2(x)\|},c)|\\
    &=|c^\top \left(w^*(\frac{h_1(x)}{\|h_1(x)\|})-w^*(\frac{h_2(x)}{\|h_2(x)\|})\right)|\\
    &=|c^\top \left(w^*(h_1(x))-w^*(h_2(x))\right)|\\
    &\leq \|c\| \left\|w^*(\frac{h_1(x)}{\|h_1(x)\|})-w^*(\frac{h_2(x)}{\|h_2(x)\|})\right\| \\
    & \leq  \frac{\Delta(\mathcal{S})\rho(\mathcal{C})}{ 2\eta} \left\|\frac{h_1(x)}{\|h_1(x)\|}-\frac{h_2(x)}{\|h_2(x)\|}\right\|, 
\end{align*}
where the first and third equations used Fact \ref{lem:spo_scale_invariant}, the first inequality used the Cauchy-Schwarz inequality, the second inequality applied the upper bound of $\|w^*(h_1(x))-w^*(h_2(x))\|$.
\hfill\Halmos

\end{proof}

\begin{proof}{\bfseries Proof of Theorem \ref{thm:stoppingtime1}}
By Theorem \ref{thm:prediction}, with probability $1-\delta$ we have $R_\ell(h_\tau)-R_\ell(h^*)\leq 2r_\tau$ where $r_\tau = 2\gamma L \sqrt{\frac{\log(2\tau|\mathcal{H}|/\delta)}{\tau}}$. Then $R_\ell(h_\tau)-R_\ell(h^*) \leq \epsilon$ is implied by 
$
2r_\tau = 4\gamma L \sqrt{\frac{\log(2\tau|\mathcal{H}|/\delta)}{\tau}} \leq \epsilon$,
which is equivalent to 
$
\log(\frac{2\tau |\mathcal{H}|}{\delta}) \leq  (\frac{\epsilon}{4\gamma L})^2 \tau
$.
Plug in $\tau = C \frac{1}{\epsilon^2}\log \frac{2|\mathcal{H}|}{\epsilon^2 \delta}$, we have that the RHS can be expressed as 
$$RHS = \frac{C}{(4\gamma L)^2} \log \frac{2|\mathcal{H}|}{\epsilon^2 \delta}.$$

The LHS $\log(\frac{2\tau |\mathcal{H}|}{\delta})$ can be written as
$$
\begin{aligned}
    LHS &= \log \tau + \log \frac{2|\mathcal{H}|}{\delta} \\
    & = \log C + \log(\frac{2|\mathcal{H}|}{\epsilon^2 \delta}) + \log \log (\frac{2|\mathcal{H}|}{\epsilon^2 \delta})\\
    & \leq \log C + 2\log(\frac{2|\mathcal{H}|}{\epsilon^2 \delta}).
\end{aligned}
    $$

Pick $C = 3 (4\gamma L)^2$ we will have $LHS \leq RHS$. Hence 
$\tau =\frac{3(4\gamma L)^2}{\epsilon^2}\log \frac{2|\mathcal{H}|}{\epsilon^2 \delta} = \otilde(\epsilon^-2)$ will safice to bound the risk under $\epsilon$. \hfill \Halmos
\end{proof}

\begin{proof}{\bfseries Proof of Proposition \ref{sublinear_iwal}}
We prove the two statements separately.

\noindent\textbf{Part (a).}
By Theorem~\ref{thm:prediction}, with probability at least $1-\delta$, for all $T\ge 1$,
\[
\rspo(h_T)-\rspo(h^*)
\le
4\gamma L \sqrt{\frac{\log(2T|\mathcal{H}|/\delta)}{T}}.
\]
Since $\log(2T|\mathcal{H}|/\delta)=\log T + \log(2|\mathcal{H}|/\delta)$, the right-hand side is of order
\[
\otilde\!\left(T^{-1/2}\sqrt{\log(1/\delta)}\right),
\]
where $\otilde(\cdot)$ hides logarithmic factors in $T$ and $|\mathcal{H}|$. Moreover, $\rspo^*=\rspo(h^*)$ by definition of $h^*\in\arg\min_{h\in\mathcal{H}}\rspo(h)$, hence
\[
\rspo(h_T)-\rspo^* \le \otilde\!\left(T^{-1/2}\sqrt{\log(1/\delta)}\right),
\]
which establishes (a).

\medskip
\noindent\textbf{Part (b).}
Fix any $\epsilon>0$ and $\delta\in(0,1)$. Let $\tau$ be the stopping time defined in Theorem~\ref{thm:stoppingtime1}:
\[
\tau := \frac{3(4\gamma L)^2}{\epsilon^2}\log \frac{2|\mathcal{H}|}{\epsilon^2 \delta}.
\]
By Theorem~\ref{thm:stoppingtime1}, under the same assumptions as Theorem~\ref{thm:prediction}, we have that when $T\ge \tau$, with probability at least $1-\delta$,
\[
\rspo(h_T)-\rspo(h^*)\le \epsilon.
\]
It remains to verify the claimed asymptotic scaling of $\tau$. Since
\[
\log \frac{2|\mathcal{H}|}{\epsilon^2 \delta}
=
\log\frac{1}{\delta} + \log\frac{2|\mathcal{H}|}{\epsilon^2},
\]
we obtain
\[
\tau
=
\frac{3(4\gamma L)^2}{\epsilon^2}
\left(\log\frac{1}{\delta} + \log\frac{2|\mathcal{H}|}{\epsilon^2}\right)
=
\otilde\!\left(\epsilon^{-2}\log\frac{1}{\delta}\right),
\]
where $\otilde(\cdot)$ hides factors logarithmic in $|\mathcal{H}|$ and $1/\epsilon$. This proves (b).

\medskip
Finally, both (a) and (b) hold on the same $1-\delta$ probability event guaranteed by Theorem~\ref{thm:prediction} and Theorem~\ref{thm:stoppingtime1}, completing the proof.
\end{proof}

\begin{proof}{\textbf{Proof of Theorem \ref{thm:better}}}
We consider the two algorithms in turn.

Regarding the proof of the early stopping of Algorithm \ref{alg:general-iwal}, by
Theorem \ref{thm:prediction} (with
constant \(c=4\gamma L\)), for any confidence parameter
\(\delta\in(0,1]\),
\[
  \Pr\!\Bigl[
      \rspo(h_T)-\rspo^\star
      \le c\sqrt{\frac{\log(2T|\mathcal H|/\delta)}{T}}
     \Bigr]
  \;\ge\;1-\delta .
\]

Choose
\(
  \delta_T = 2T|\mathcal H|\,
            \exp\!\bigl(-\alpha^{2}T/16\gamma^{2}L^{2}\bigr)
      = 2T|\mathcal H|\,e^{-k_1T}.
\)
With this choice, we have
\(c\sqrt{\log(2T|\mathcal H|/\delta_T)/T}=\alpha\). Hence, we have
\[
  \Pr\!\bigl[
        \rspo(h_T)-\rspo^\star >\alpha
      \bigr]
  \;\le\;\delta_T
  = 2T|\mathcal H|\,e^{-k_1T}.
\]

Because of the optimality gap,
\(\rspo(h_T)-\rspo^\star\in\{0\}\cup[\alpha,\infty)\).
Therefore, for some sufficiently large $T$,
$T \ge \frac{3(4\gamma L)^2}{\alpha^2}\log \frac{2|\mathcal{H}|}{\alpha^2 \delta}$, with probability $1-\delta$ we have $\rspo(h_T)-\rspo(h^*)= 0$.

Consequently, after designing $T$ experiments, we have $H_t \subseteq \mathcal{H}^*$.

\begin{figure}[ht]
    \centering
    \includegraphics[width=0.4\linewidth]{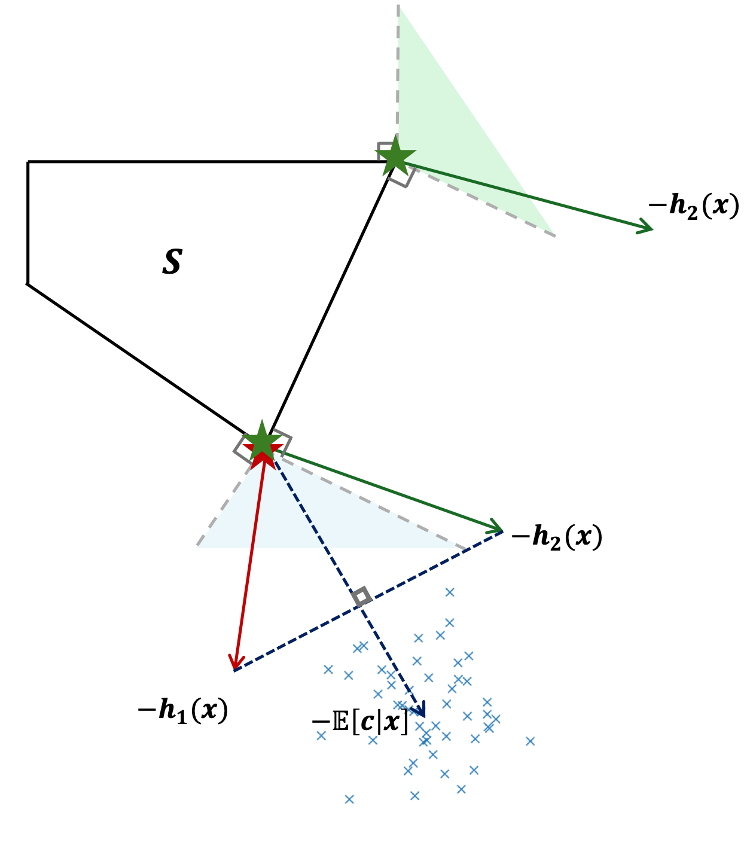}
    \caption{An Example of a Distribution Under Which a Decision-Blind Algorithm Results in a Later Stopping Time}
    \label{fig:counterexample}
\end{figure}

Next, to show that Algorithm \ref{alg:blind} may not converge, we consider a setting where the candidate pool contains a single feature value \(x\), and we assume the ambient space dimension is \(d\ge2\). Consider the specific distribution illustrated in Figure \ref{fig:counterexample}. Suppose we have two prediction candidates, \(h_1(x)\) and \(h_2(x)\in\mathbb{R}^d\), that possess equal norm. Furthermore, assume that the resulting optimal decision variables differ, i.e., \(w^*(h_1(x))\not = w^*(h_2(x))\). Let the conditional expectation \(\bbE[c|x]\) be positioned such that it equally divides the angle \(\angle(h_1(x), h_2(x))\). Additionally, assume that \(\bbE[c|x]\) lies within the same optimal decision cone as \(h_1(x)\), which implies \(w^*(h_1(x)) = w^*(\bbE[c|x])\), making \(h_1\) the optimal prediction model. Suppose the distribution of the cost vector is generated by
\[
c_t \;=\; \bbE[c|x] + \epsilon_t,\qquad t=1,\dots,T,
\]
where \(\{\epsilon_t\}_{t=1}^T\) are i.i.d. with \(\bbE[\epsilon_t]=0\), coordinatewise constant variance \(\sigma_\epsilon^2\in(0,\infty)\). Define $\bar{\Delta} := h_2(x) - h_1(x)$.

Since \(\|h_1(x)\|=\|h_2(x)\|\) and they equally divide the angle, for any mean \(\bbE[c|x]\) we have
\[
\|h_2(x)-\bbE[c|x]\|^2 - \|h_1(x)-\bbE[c|x]\|^2 \;=\; 2\,\bbE[c|x]^\top\bar{\Delta} = 0.
\] 
Consequently, for one draw,
\[
\|h_2(x)-c_t\|^2 - \|h_1(x)-c_t\|^2 = 2(h_2(x)-h_1(x))^\top c_t = 2\,\bar{\Delta}^\top ( \bbE[c|x] + \epsilon_t)
= 2\,\bar{\Delta}^\top\epsilon_t.
\]
Averaging gives
\[
\overline G_T \;:=\; \frac1T\sum_{t=1}^T\big(\|h_2-c_t\|^2 - \|h_1-c_t\|^2\big)
=  2\,\bar{\Delta}^\top \overline\epsilon_T,\qquad
\overline\epsilon_T:=\frac1T\sum_{t=1}^T\epsilon_t.
\]
Define \(W_T:=\bar{\Delta}^\top(\sqrt{T}\,\overline\epsilon_T)\). By the classical central limit theorem,
\[
W_T \ \Rightarrow\ \mathcal N\!\big(0,\ \sigma_\epsilon^2\,\|\bar{\Delta}\|^2\big).
\]
Thus, the event that the blind rule prefers \(h_2\) (or, equivalently, fails to eliminate it) is
\[
\{\overline G_T\le 0\}
\quad\Longleftrightarrow\quad
\Big\{\, \frac{2 W_T}{\sqrt{T}}\le 0\,\Big\}\quad\Longleftrightarrow\quad
\Big\{\,W_T \;\le\; 0\,\Big\}.
\]

By CLT, we have $\Pr(W_T \;\ge\; 0) = 1/2$ when $T$ goes to infinity. Thus, by continuity of the Gaussian cdf, there exists a constant $
p_\infty \;:=\; 1-\Phi\!\Big(\tfrac{1}{2\,\sigma_\epsilon\,}\Big)\ \in (0,1/2]$, (for example, $0.25$), 
and \(T_0<\infty\) such that for all \(T\ge T_0\),
\[
\Pr(\overline G_T\le 0)\ \ge\ \frac{p_\infty}{2}\ =:\delta_0\ >0.
\]
Therefore, with probability at least \(\delta_0\) (uniformly for all large \(T\)) the blind baseline does not eliminate \(h_2\). Since \(\rspo(h_2)>\rspo^\ast\) under \(\mathcal D_T\), we have exhibited a distribution for which the blind algorithm still retains a suboptimal predictor in \(H_T\), which proves item~(2) of Theorem~\ref{thm:better}. \hfill \Halmos
\end{proof}

\begin{proof}{\bfseries Proof of Proposition~\ref{prop:equiv_hypothsis}}
We first verify Assumption~\ref{assumption:tau_lower_bound}. By construction, Algorithm~\ref{alg:aug-hypothesis} augments the hypothesis class so that, for every $h \in \mathcal{H}$ and every $x \in \mathcal{X}$, either
$
\frac{\nu_{\mathcal{S}}(h(x))}{\|h(x)\|} \ge \eta$,

or the prediction $h(x)$ is replaced by a nondegenerate cost vector $c$ corresponding to the cone that induces the same optimizer. By the definition of $\eta$,
it follows that, for all $\tilde{h} \in \tilde{\mathcal{H}}$ and all $x \in \mathcal{X}$, $
\frac{\nu_{\mathcal{S}}(\tilde{h}(x))}{\|\tilde{h}(x)\|} \ge \eta$.
Hence, Assumption~\ref{assumption:tau_lower_bound} holds for the augmented hypothesis class.

We now show that $\mathcal{H}$ and $\tilde{\mathcal{H}}$ are equivalent with respect to $\mathcal{X}$ in the sense of Definition~\ref{def:equiv_hypothesis}.  
First, consider any $x \in \mathcal{X}$ and any $h \in \mathcal{H}$. By the construction of Algorithm~\ref{alg:aug-hypothesis}, there exists a corresponding hypothesis $\tilde{h} \in \tilde{\mathcal{H}}$ such that either $\tilde{h}(x) = h(x)$ or $\tilde{h}(x)$ is replaced by a nondegenerate cost vector lying in the same normal cone. In both cases, the downstream optimizer satisfies $
w^*(h(x)) = w^*(\tilde{h}(x))$,
which implies $
R_{\mathrm{SPO}}(h(x)) = R_{\mathrm{SPO}}(\tilde{h}(x))$.

Conversely, for any $x \in \mathcal{X}$ and any $\tilde{h} \in \tilde{\mathcal{H}}$, by construction there exists an original hypothesis $h \in \mathcal{H}$ such that $\tilde{h}$ coincides with $h$ except possibly at finitely many design points where $\tilde{h}(x)$ is replaced by a cost vector inducing the same optimizer. Therefore, $
R_{\mathrm{SPO}}(\tilde{h}(x)) = R_{\mathrm{SPO}}(h(x))$.
Both conditions in Definition~\ref{def:equiv_hypothesis} are satisfied, which establishes the equivalence of $\mathcal{H}$ and $\tilde{\mathcal{H}}$.\hfill\Halmos
\end{proof}

\begin{proof}{\bfseries Proof of Proposition \ref{prop:gap_example}}
Let $\mathrm{X}_{(j)}$ denote the corresponding design in Condition \ref{assu:x_delta}, for any suboptimal predictor 
$h\in \mathcal{H} \setminus \mathcal{H}^*$ we have
$$
\begin{aligned}
\rspo(h) - \rspo^\star &= \bbE[\lspo(h(x),c)] - \bbE[\lspo(h^*(x),c)]
\\ &=  \bbE [\bbE[\lspo(h(x),c) - \lspo(h^*(x),c)|x]]
\\ &= \sum_{i=1}^n \mu(\mathrm{X}_{(i)})\bbE[\lspo(h(x),c) - \lspo(h^*(x),c)|x = \mathrm{X}_{(i)}]
\\  & \geq \mu (\mathrm{X}_{(j)})\bbE[\lspo(h(x),c) - \lspo(h^*(x),c)|x = \mathrm{X}_{(j)}]
\\ &= \mu (\mathrm{X}_{(j)})\bbE[c^\top \left(w^*(h(x)) - w^*(h^*(x)\right)|x = \mathrm{X}_{(j)}]
\\ &= \mu (\mathrm{X}_{(j)})\bbE[c^\top |x = \mathrm{X}_{(j)}]\left(w^*(h(\X_{(j)})) - w^*(h^*(\X_{(j)})\right),
\end{aligned}
$$
where the first inequality used condition \ref{assu:risk_minimizer}, conditioned on any $x\in \X$ we have $\bbE[\lspo(h(x),c)] \geq \bbE[\lspo(\bbE[c|x],c)] = \bbE[\lspo(h^*(x),c)]$. 

By Condition \ref{assu:risk_gap} and Equation (11) in \cite{el2022generalization}, for $w_1 = w^*(h(\X_{(j)}))$, we have,
\[ \bbE[c|x = \mathrm{X}_{(i)}]^\top\Big(w_1 - w^*(h^*(\mathrm{X}_{(i)}))\Big) = \bbE[c|x = \mathrm{X}_{(i)}]^\top\Big(w_1 - w^*(\bbE[c|(\mathrm{X}_{(i)})])\Big) \ge  {\eta \underline{c}}\underline{\gamma} .\]
Consequently, by the lower bound of $\mu(\cdot)$ in Condition \ref{assu:x_delta}, we have that the suboptimal gap is at least $\underline{\mu}{\eta \underline{c}}\underline{\gamma}$.\hfill \Halmos
\end{proof}

\end{APPENDICES}
\end{document}